\documentclass{article}
\usepackage[margin=1in]{geometry}

\usepackage{microtype}
\usepackage{graphicx}
\usepackage{authblk}
\usepackage{subcaption}
\usepackage{natbib}
\usepackage{booktabs} 
\usepackage{etoc}
\usepackage{titletoc}
\usepackage{parskip}
\usepackage{hyperref}

\usepackage{amsmath}
\usepackage{amssymb}
\usepackage{mathtools}
\usepackage{amsthm}

\usepackage[capitalize,noabbrev]{cleveref}

\theoremstyle{plain}
\newtheorem{theorem}{Theorem}[section]

\newtheorem{lemma}[theorem]{Lemma}
\newtheorem{corollary}[theorem]{Corollary}
\theoremstyle{definition}

\newtheorem{assumption}[theorem]{Assumption}
\theoremstyle{remark}

\usepackage[textsize=tiny]{todonotes}
\usepackage{tcolorbox}
\usepackage{hyperref}
\usepackage{url}
\usepackage{amsthm}
\usepackage{subcaption}
\usepackage{cancel}

\def\piref{\pi_{\mathrm{ref}}}
\def\piout{\pi_{\mathrm{out}}}

\definecolor{mgh}{rgb}{0.1,0.5,0.1}

\newcommand{\pidata}{\pi_{\mathrm{data}}}
\definecolor{lu}{rgb}{0.8,0.1,0.1}

\definecolor{gr}{rgb}{0.4,0.7,0.2}

\definecolor{st}{rgb}{0.1,0.5,0.1}

\definecolor{greenp}{rgb}{0.0, 0.51, 0.5}

\newcommand{\argmin}{\mathop{\rm argmin}}
\newcommand{\argmax}{\mathop{\rm argmax}}

\usepackage[capitalize,noabbrev]{cleveref}
\usepackage{wrapfig}
\newcommand{\abs}[1]{\left |{#1}\right |}
\newcommand{\bc}[1]{\left\{{#1}\right\}}
\newcommand{\br}[1]{\left({#1}\right)}
\newcommand{\bs}[1]{\left[{#1}\right]}

\usepackage{dsfont}

\usepackage{algorithm}
\usepackage{algorithmic}
\usepackage{multirow}
\newcounter{protocol}
\makeatletter

\usepackage{amssymb}
\usepackage{amsfonts}
\usepackage{graphics}
\usepackage{graphicx}
\usepackage{setspace}

\newcommand{\sigmoid}{\sigma}

\newcommand{\innerprod}[2]{\left\langle{#1},{#2}\right\rangle}

\newcommand{\X}{\mathcal{X}}
\newcommand{\A}{\mathcal{A}}

\newcommand{\initial}{\nu_0}

\usepackage[utf8]{inputenc} 
\usepackage[T1]{fontenc}    
\usepackage{url}            
\usepackage{booktabs}       
\usepackage{amsfonts}       
\usepackage{nicefrac}       
\usepackage{microtype}      
\usepackage[table]{xcolor}         

\usepackage{mathtools}
\makeatletter
\newcommand{\mytag}[2]{%
  \text{#1}%
  \@bsphack
  \protected@write\@auxout{}%
         {\string\newlabel{#2}{{#1}{\thepage}}}%
  \@esphack
}
\makeatother

\usepackage{amsthm}
\usepackage{thm-restate}
\title{\textbf{Direct Preference Optimization with Rating Information:\\ Practical Algorithms and Provable Gains}}

\author[1]{Luca Viano}
\author[2]{Ruida Zhou}
\author[3]{Yifan Sun}
\author[2]{Mahdi Namazifar}
\author[1]{Volkan Cevher}
\author[4]{Shoham Sabach}
\author[5]{Mohammad Ghavamzadeh}

\affil[1]{EPFL}
\affil[2]{Amazon AGI}
\affil[3]{University of Illinois at Urbana-Champaign }
\affil[4]{Cornell University}
\affil[5]{Qualcomm AI Research}

\begin{document}

\maketitle
\begin{abstract}
The class of direct preference optimization (DPO) algorithms has emerged as a promising approach for solving the alignment problem in foundation models. These algorithms work with very limited feedback in the form of pairwise preferences and fine-tune models to align with these preferences without explicitly learning a reward model. While the form of feedback used by these algorithms makes the data collection process easy and relatively more accurate, its ambiguity in terms of the quality of responses could have negative implications. 
For example, it is not clear if a decrease (increase) in the likelihood of preferred (dispreferred) responses during the execution of these algorithms could be interpreted as a positive or negative phenomenon. In this paper, we study how to design algorithms that can leverage additional information in the form of \emph{rating gap}, which informs the learner how much the chosen response is better than the rejected one. 
We present new algorithms that can achieve faster statistical rates than DPO in presence of accurate rating gap information. Moreover, we theoretically prove and empirically show that the performance of our algorithms is {\em robust} to inaccuracy in rating gaps. Finally, we demonstrate the solid performance of our methods in comparison to a number of DPO-style algorithms across a wide range of LLMs and evaluation benchmarks. 
%
\end{abstract}


\section{Introduction}
\label{sec:intro} 

Learning from preference data (ranking information) has become a popular paradigm for solving the alignment problem in large language models (LLMs)~\citep{christiano2017deep,rafailov2023direct,zhu2023principled}. Although data collection is easier and the annotators' feedback is less noisy in this setting, compared to the rating style feedback, the amount of information that can be extracted from such data is very limited. 
Indeed, given a prompt, a preferred and a dispreferred response, several scenarios are likely: both responses are of high (low) quality or one is good and the other one is poor. However, the contrastive learning approach used by Direct Preference Optimization (DPO) style algorithms, such as DPO~\citep{rafailov2023direct} and Identity-mapping Preference Optimization (IPO)~\citep{azar2024general}, is only justified for the latter case. In the cases that the two responses have high (low) quality, it makes sense to imitate (forget) both of them. However, it is not possible for a fine-tuning algorithm to identify which of these cases it is facing from the preference style information. The lack of information about the individual or relative quality of responses can create ambiguity for these algorithms and have negative implications on their performance. For example, it has been shown that it can incentivize in-sample probability reduction in DPO-style algorithms and create bias towards policies that favor out-of-distribution responses (e.g.,~\citealt{nemotron,ppo_vs_dpo,smaug,fisch2024robust,xiao2024caldpo,AIPO,sppo,APO}), a phenomenon referred to as {\em likelihood displacement}~\citep{razin2025unintentional}. 

\begin{figure*}[t] 
\centering
\begin{tabular}{ccc}
\subfloat{%
    \includegraphics[width=0.45\linewidth]{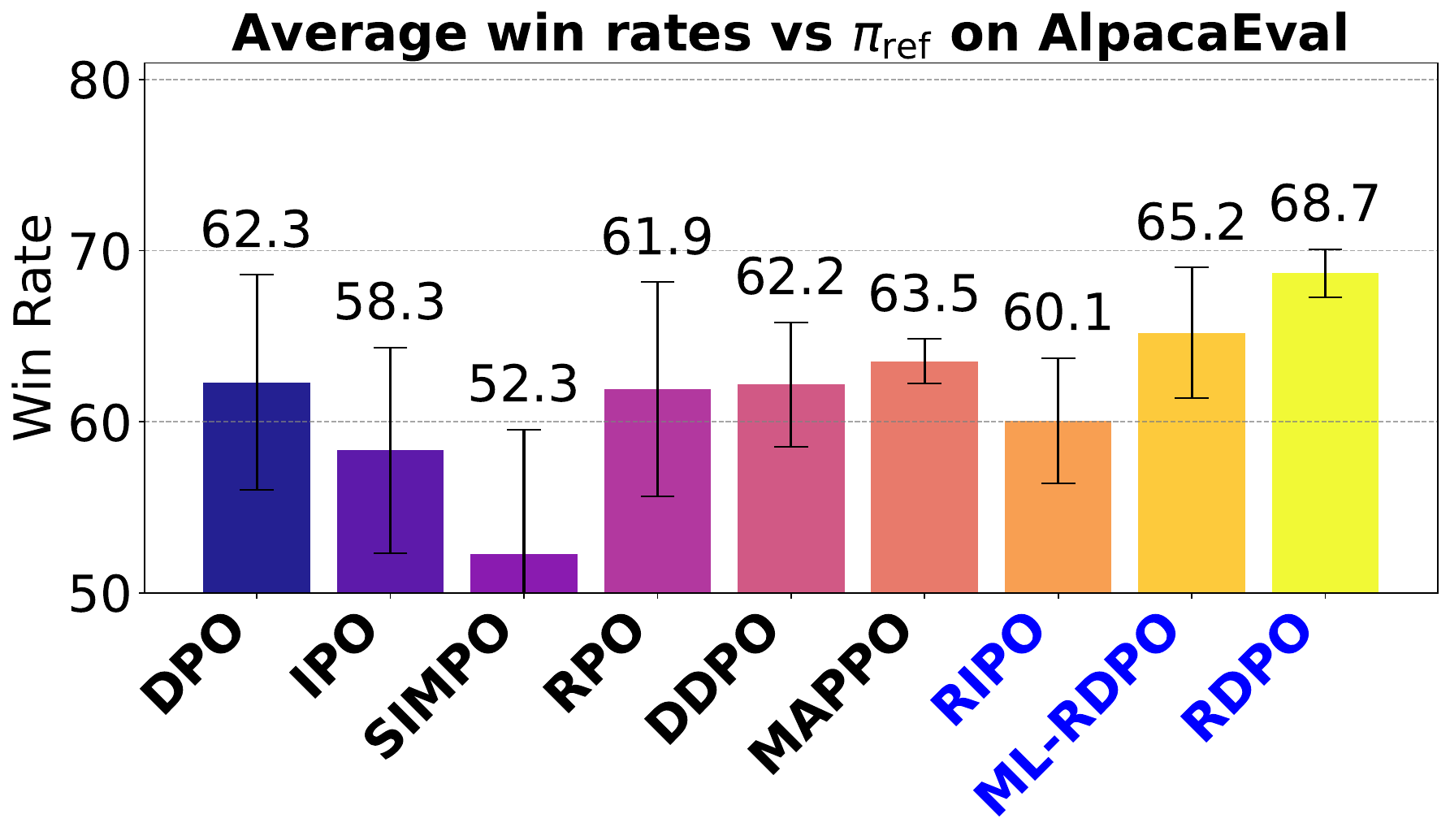}
     } &
\subfloat{%
    \includegraphics[width=0.45\linewidth]{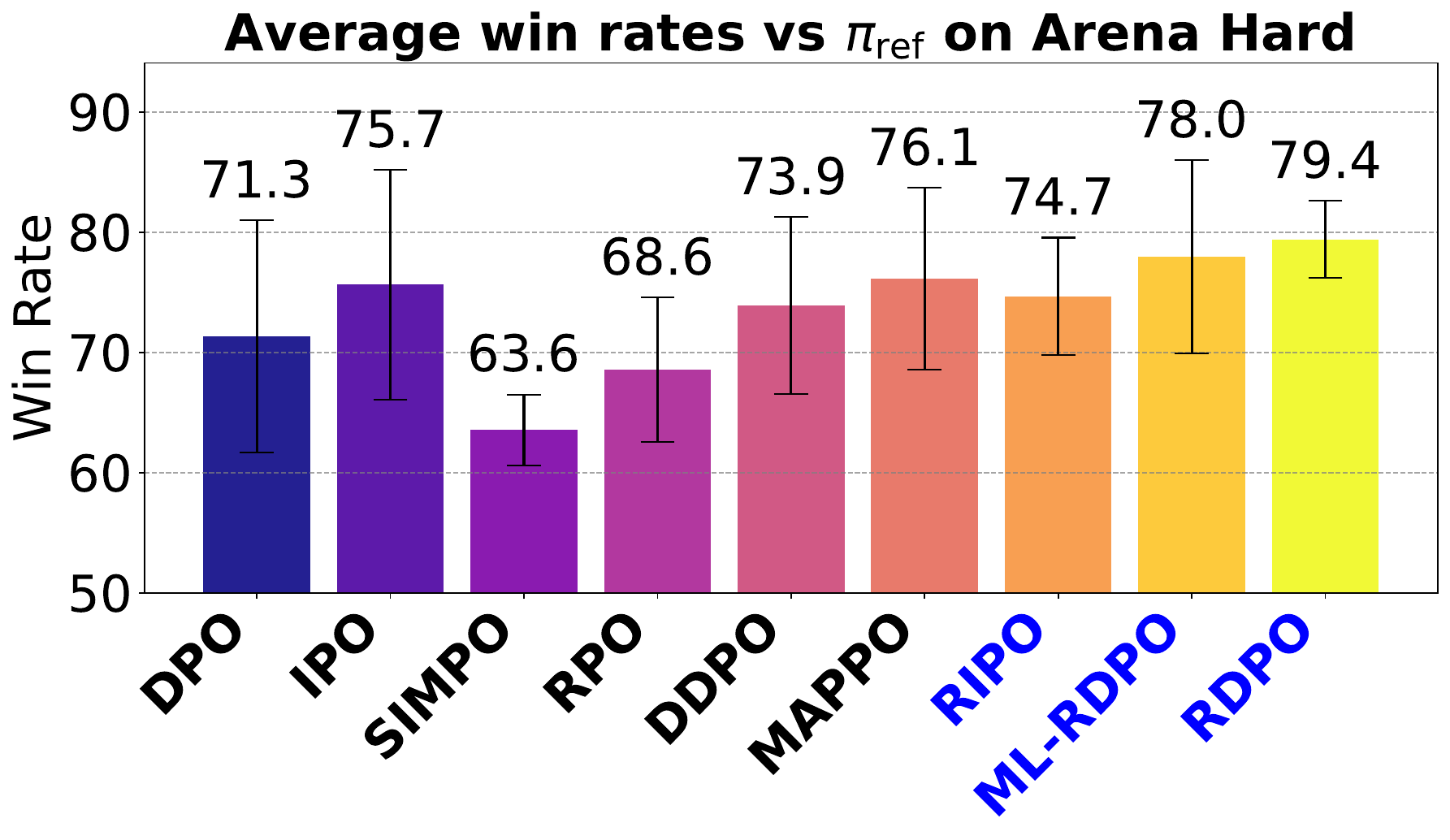}
     } 
\end{tabular}
\caption{Win rates averaged for $\piref \in \bc{\texttt{Llama3.1-8B}, \texttt{Zephyr-7B}, \texttt{Mistral-7B}}$. Our methods are labelled in \textcolor{blue}{blue}. }
\label{fig:intro_exp}
\end{figure*}

In this paper, we study a setting in which, in addition to preference/ranking feedback, the training dataset contains the relative rating of the responses, which we refer to as {\em rating gap}. We propose three algorithms derived from two different approaches that leverage this additional information in an efficient and principled way. We derive our first two algorithms, {\em Rating DPO} (RDPO) and {\em Rating IPO} (RIPO), by changing the RLHF objective to maximize a linear combination of the ranking and rating information. Our third algorithm, {\em Maximum-Likelihood-based Rating DPO} (ML-RDPO), is derived using the maximum-likelihood (ML) principle by making certain assumptions on the joint distribution of the ranking and rating information. In ML-RDPO, we consider a linear combination of the likelihood objectives rather than a linear combination of the two sources of information (ranking and rating), as is done in the derivation of RDPO and RIPO. We provide theoretical analysis for our algorithms showing that they can achieve faster statistical rates than DPO in the presence of accurate rating gap information, and are robust in case this information is noisy. We empirically evaluate our algorithms and compare them to a number of DPO-style methods across a wide range of LLMs and evaluation benchmarks. The results validate our theoretical findings and demonstrate the solid performance of our algorithms.

{\bf Theoretical Contributions.} We prove two desirable properties, \emph{acceleration} and \emph{robustness}, for our algorithms. Indeed, if the ratings are consistent with the latent reward model used by the annotator to rank the responses (under the Bradley-Terry formulation), RDPO and ML-RDPO can accelerate learning compared to DPO. That is, they are able to learn a well-performing policy faster than an algorithm that only leverages ranking information (e.g.,~DPO or IPO). Specifically, leveraging the rating information, our algorithms guarantee an \emph{exponential improvement} over learning from rankings only, by avoiding the exponential dependence in the maximum reward range which affects all known ranking-only algorithms. Moreover, our algorithms are robust in the sense that the above guarantees hold even under a certain degree of corruptions or noise in the ratings.

{\bf Empirical Contributions.} In addition to these nice theoretical properties, RDPO and ML-RDPO perform well in practical alignment problems. \Cref{fig:intro_exp} provides a summary of our experiments (more experiments in Section~\ref{sec:experiments} and in \Cref{app:experiments}), 
showing the average win-rate of the checkpoints fine-tuned by different algorithms against the reference model $\piref$ in AlpacaEval~\citep{alpaca_eval} and ArenaHard~\citep{li2024crowdsourced} benchmarks. The means and standard deviations are computed over $3$ independent experiments where $\piref$ is one of the following models: Zephyr-7B,\footnote{\url{https://huggingface.co/alignment-handbook/zephyr-7b-sft-full}} Llama-3.1-8B,\footnote{\url{https://huggingface.co/allenai/Llama-3.1-Tulu-3-8B-SFT}} and Mistral-7B.\footnote{\url{https://huggingface.co/HuggingFaceH4/mistral-7b-sft-beta}} It can be seen that RDPO and ML-RDPO are the best and second best performing methods, respectively. Moreover, since RDPO has the smallest variance, it can be considered as the most consistently well-performing method across the choices of $\piref$ we experimented with.

In agreement with the theoretical findings, RDPO and ML-RDPO outperform both ranking-only methods, such as DPO, IPO, and SIMPO \citep{meng2024simpo}, and rating-only methods, such as Distilled DPO \citep{fisch2024robust}. Moreover, their performance is superior to RPO \citep{adler2024nemotron} and MAPPO \citep{lan2025mappo}, which are recent alignment methods leveraging both rating and ranking information, as done by RDPO and ML-RDPO. Finally, while RIPO does not perform as well as RDPO in the experiments of \Cref{fig:intro_exp}, it still performs well in some of the experiments we report in Section~\ref{sec:experiments} and in \Cref{app:experiments}. 
Moreover, it serves as an example that our principled derivation of RDPO can be extended beyond the cases where the Bradley-Terry assumption holds. 

{\bf Paper Organization.} In Sections~\ref{sec:prelim} and~\ref{sec:algos}, we first formalize the alignment problem and then derive RDPO, RIPO, and ML-RDPO from first principles. We present the theoretical guarantees of these algorithms in \Cref{sec:analysis}, discuss their closest related work in \Cref{sec:related}, and demonstrate their practical effectiveness in \Cref{sec:experiments}, before concluding the paper in \Cref{sec:conclu}.


\section{Preliminaries}
\label{sec:prelim}

We begin by presenting the main ingredients of the alignment problem in LLMs on which we will build in the subsequent sections. The input to most alignment algorithms is a dataset $\mathcal D$ of tuples $(x,a,a',z)$, where $x$ is a prompt sampled from a distribution $\initial$, $a$ and $a'$ are two responses, and $z$ is a Bernoulli random variable that is one if the response $a$ is preferred to $a'$ conditioned on $x$, i.e., $a \succ a' \mid x$ and is zero, otherwise. Given $z$, we define $a^+ = a$ and $a^-=a'$ if $z=1$ and viceversa $a^+ = a'$ and $a^-=a$ if $z=0$. The binary variable $z$ is assumed to be sampled according to the Bradley-Terry (BT) model\footnote{This can be extended to the more general Plackett-Luce ranking model~\citep{Plackett75ML,Luce12IC} if we have access to several ranked responses.}~\citep{Bradley52RA}, i.e.,~$z\sim\texttt{Bern}(\sigma(r^\star(x,a) - r^\star(x,a')))$, where $\sigma(x)=1/(1+e^{-x})$ is the sigmoid function and $r^\star$ is the latent reward model (RM) of the annotator. We are also given a reference policy $\pi_{\text{ref}}$ (often the SFT checkpoint $\pi_{\text{SFT}}$) which serves as a guardrail. 

In RLHF, we first employ $\mathcal D$ to train a parameterized RM, $r_\phi$, and then use it to solve the following KL-regularized reinforcement learning (RL) objective, denoted by $J_\beta$:
\begin{equation}
\label{eq:RLHF}
\!\max_\theta \; \underbrace{\mathbb E_{\!\!\substack{x\sim\initial\\ a\sim\pi_\theta(\cdot|x)}}\Big[r_\phi(x,a) - \beta \cdot \mathbb{KL}\big(\pi_\theta(\cdot|x)\|\pi_{\text{ref}}(\cdot|x)\big)\Big]}_{:=J_\beta(\pi_\theta;r_\phi)},
\end{equation}
where $\beta>0$ is a hyper-parameter denoting the relative importance of reward maximization against ensuring a low deviation from $\pi_{\text{ref}}$. The RM is learned by minimizing the cross-entropy (CE) loss
\begin{equation}
\label{eq:RM}
\min_\phi \!\! \sum_{(x,a^+,a^-)\in\mathcal D}\!\!\!\!\!\!\! -\log\sigma\big(r_\phi(x,a^+) - r_\phi(x,a^-)\big).
\end{equation}
Fine-tuning $\pi_\theta$ in the RLHF approach is split into two stages: reward learning using the BT model, followed by a policy optimization using~\eqref{eq:RLHF}. More recently, a family of algorithms has emerged that solves the above two problems in a single stage, without explicitly learning a RM. The most notable among this family is DPO~\citep{rafailov2023direct}. The key insight here is that given a RM $r$, the problem in~\eqref{eq:RLHF} admits the following closed-form solution: 
\begin{equation}
\label{eq:Optimal-Policy}
\pi^\star_{\beta}(a|x)=\pi_{\text{ref}}(a|x)\exp(r(x,a)/\beta)/Z(x),    
\end{equation}
where $Z(x)$ is the partition function. 
The next step is to introduce a parameterization of the reward $r$ by first parameterizing the closed-form solution $\pi_{\beta}^\star$ as $\pi_\theta$ in~\eqref{eq:Optimal-Policy}, and then inverting it for $r$ as 
\begin{equation}
\label{eq:Param-RM}
r_\theta(x,a) = \beta \cdot \log\frac{Z(x)\pi_\theta(a|x)}{\pi_{\text{ref}}(a|x)}. 
\end{equation}
The final step is to substitute the parameterized reward from~\eqref{eq:Param-RM} into the CE-loss in~\eqref{eq:RM}, the partition function $Z(x)$ is canceled out, leading to the following loss for DPO:
\begin{equation}
\label{eq:DPO}
\mathcal L_{\text{DPO}}(\theta) = \!\! \sum_{(x,a^+,a^-)\in\mathcal D}\!\!\!\!\!\! -\log\sigma\big(\beta \cdot \Delta_\theta(x,a^+,a^-)\big),
\end{equation}
where $\;\Delta_\theta(x,a^+,a^-) = \log\frac{\pi_\theta(a^+|x)}{\pi_{\text{ref}}(a^+|x)} - \log\frac{\pi_\theta(a^-|x)}{\pi_{\text{ref}}(a^-|x)}$.

Another popular member of the family of DPO-style algorithms is IPO~\citep{azar2024general}. IPO uses the following square loss to fine-tune the LLM $\pi_\theta$:
\begin{equation}
\label{eq:IPO}
\mathcal L_{\text{IPO}}(\theta) = \!\! \sum_{(x,a^+,a^-)\in\mathcal D}\!\!\!\!\!\! \big(\Delta_\theta(x,a^+,a^-) - \frac{1}{2\beta}\big)^2.
\end{equation}
We conclude the section by introducing some notation that will be useful in the rest of the paper.
\paragraph{Notation.} We denote by $\mathcal X$ and $\mathcal A$ the set of prompts given to and the set of responses generated by a LLM.  Moreover, we denote by $\Pi$ the class of policies that can be induced by a choice of the parameters $\theta$, i.e.,~$\Pi = \bc{\pi \mid \exists~ \theta~ \text{s.t.}~ \pi(a|x) = \pi_\theta(a|x)~\forall x,a\in\X\times\A}$, and by $\abs{\Pi}$ the size of this class.\footnote{For infinite classes, the result can be generalized using covering numbers.} Finally, we denote the expected reward or performance of a policy $\pi$ under reward $r$ as  $
\innerprod{\pi}{r} := \sum_{x\in\X} \initial(x) \sum_{a \in \A} \pi(a|x) r(x,a)$, and the optimal policy under the latent reward model $r^\star$ as $\pi^\star = \argmax_{\pi\in \Pi} J_\beta(\pi; r^\star)$.

\section{Algorithms for Alignment with Ranking and Rating Information}
\label{sec:algos}

We now consider the problem of fine-tuning a LLM where the original dataset is augmented with the {\em rating gaps} of the responses, i.e., 
\begin{equation}
\label{eq:dataset}
\mathcal{D} = \bc{(x_i, a_i, a'_i, z_i, \hat{r}(x_i, a^+_i) - \hat{r}(x_i, a^-_i))}^N_{i=1}.    
\end{equation}
In~\eqref{eq:dataset}, $\hat{r}$ is the rating, i.e.,~an approximation of the latent RM of the annotator $r^\star$. Thus, the rating gap $\Delta_{\hat{r}} (x_i,a^+_i, a^-_i) := \hat{r}(x_i, a^+_i) - \hat{r}(x_i, a^-_i)$ serves as an estimate of how much better is the selected action $a^+_i$ compared to the rejected one $a^-_i$. Having introduced the dataset, we are now ready to present our algorithms in the next sections.


\subsection{Ratings Direct Preference Optimization (RDPO)}
\label{subsec:RDPO}


For deriving RDPO, we consider the RLHF objective in~\eqref{eq:RLHF} where the reward of interest is a linear combination of the original reward model (RM) $r$ and the rating estimate $\hat r$, i.e.,
\begin{align}
\label{eq:RDPO_RLHF}
\max_\theta \; \mathbb E_{\substack{x\sim\initial\\a\sim\pi_\theta(\cdot|x)}}&\bigg[r(x,a) + \frac{\beta \cdot \hat{r}(x,a)}{\beta_1} \nonumber \\&- \beta \cdot \mathbb{KL}\big(\pi_\theta(\cdot|x)\|\pi_{\text{ref}}(\cdot|x)\big)\bigg].
\end{align}
The parameter $\beta_1$ weighs the contribution of the rating information $\hat r$ in~\eqref{eq:RDPO_RLHF}, which is inversely related to our confidence in the accuracy of the rating information: the smaller $\beta_1$ is, the more confident we are about the accuracy of $\hat r$. 
The problem in~\eqref{eq:RDPO_RLHF} admits the following closed-form solution:
\begin{equation}
\label{eq:OptimalRDPO}
\!\pi_{\beta}^\star(a|x)=\pi_{\text{ref}}(a|x)\exp\Big(\frac{r(x,a)}{\beta} + \frac{\hat{r}(x,a)}{\beta_1}\Big)/Z(x).
\end{equation}
Inverting the above equation and approximating  $\pi^\star_\beta$ via $\pi_\theta$, we can define the following implicit RM parametrized for any policy parameter $\theta$ as $r_\theta(x,a) = \beta \cdot \log\frac{Z(x)\pi_\theta(a|x)}{\pi_{\text{ref}}(a|x)} - \frac{\beta}{\beta_1} \cdot \hat{r}(x,a)$. Finally, plugging $r_\theta$ into~\eqref{eq:RM}, we obtain the following loss for RDPO: 
%
\begin{equation}
\mathcal{L}_{\mathrm{RDPO}}(\theta) = \sum^N_{i=1} - \log \sigma \Big(
\beta \Delta^i_\theta - \frac{\beta}{\beta_1}  \Delta^i_{\hat{r}}  
\Big), 
\label{eq:RDPO}
\end{equation}
where for compactness we used $\Delta^i_\theta := \Delta_{\theta}(x_i,a^+_i, a^-_i)$ and $\Delta^i_{\hat{r}} := \Delta_{\hat{r}}(x_i,a^+_i, a^-_i)$. The exact definitions of $\Delta_{\theta}(\cdot)$ and $\Delta_{\hat{r}}(\cdot)$ can be found right after Eqs.~\ref{eq:DPO} and~\ref{eq:dataset}, respectively. Following similar steps, we can derive several variants of RDPO as explained in the rest of this section.



{\bf RDPO variants.} Replacing the BT preference model with more general alternatives, as in IPO~\cite{azar2024general}, we can derive the Ratings IPO (RIPO) loss in which the log sigmoid function in~\eqref{eq:RDPO} is substituted with a parabola centered at $1/2$, i.e.,
\begin{equation}
    \mathcal{L}_{\mathrm{RIPO}}(\theta) := \sum^N_{i=1}\Big(\beta\Delta^i_{\theta} - \frac{\beta}{\beta_1} \Delta^i_{\hat{r}}- \nicefrac{1}{2}\Big)^2. \label{eq:RIPO}
\end{equation} 
It is easy to see that the RIPO loss is a shifted version of the loss used by Distilled-DPO~\citep{fisch2024robust}, which, of course, was derived from a different perspective. Other RDPO variant can be derived by replacing the KL divergence in~\eqref{eq:RDPO_RLHF} with other $f$-divergences such as $\chi^2$. 


We conclude this section by noticing that RDPO does not make any statistical assumption about the rating gap information. However, it requires access to the rating gaps $\Delta_{\hat{r}}(x,a^+,a^-)$ for all $(x,a^+,a^-) \in \mathcal{D}$ (not necessarily their actual ratings). 
In the next section, we propose ML-RDPO to bypass this requirement at the cost of assuming that $\Delta_{\hat{r}}(x,a^+,a^-)$ is a Gaussian random variable.  
\subsection{Maximum-Likelihood-based RDPO (ML-RDPO)}
\label{sec:MLRDPOder}

To derive ML-RDPO, we consider a \emph{linear combination of the log likelihood objectives} rather than a linear combination of the two reward models. To this end, for any reward model $r$, let $P^{\mathrm{rat,rank}}_r$ be the joint probability distribution of the ranking and rating information. The idea behind the ML-RDPO algorithm is to output the policy $\piout \propto \piref(a|x) \exp\br{\tilde{r}(x,a)/\beta}$ in which $\tilde{r}$ is an approximation of $r^\star$ computed by maximizing the joint log-likelihood: 
\begin{equation}
\label{eq:ML-RDPO-0}
 \tilde{r} = \argmax_{r\in\mathcal{R}} \; \log P^{\mathrm{rat,rank}}_r( \mathcal{D}),
\end{equation}
where $\mathcal{R}$ is parameterized by the policy weights as $\mathcal{R}=\bc{ r ~| ~\exists ~\theta~ \text{s.t.}~ \beta \Delta_\theta = \Delta_r}$. The equality in the definition of $\mathcal{R}$ holds elementwise. 
In order to obtain $\piout$ without computing $\tilde{r}$, we assume conditional independence and Gaussian distributed rating gaps, to simplify the joint log-likelihood as shown in \Cref{lemma:ml-rdpo_derivation}.

\begin{lemma}\label{lemma:ml-rdpo_derivation}
Let us assume that the rating gaps $\Delta^i_{\hat{r}} | x_i, a_i, a'_i$ are normally distributed with variance $\mathbb{V}$, and $z_i$, i.e. the preference bits, and $\Delta^i_{\hat{r}}$ are conditionally independent given $x_i,a_i,a'_i$. Then, it holds that $$\log P^{\mathrm{rat,rank}}_r( \mathcal{D}) = \sum^{N}_{i=1} \log \sigma (\Delta^i_r) - (2\mathbb{V})^{-1}\br{\Delta^i_{\hat{r}} - \Delta^i_r}^2,$$ where $\Delta^i_r := r(x_i,a^+_i) - r(x_i,a^-_i) $.
\end{lemma}

Now, given the definition of the induced reward class $\mathcal R$, we can perform the change of variable from $r$ to $\beta \log (\pi_\theta/\piref)$ in the optimization problem~\eqref{eq:ML-RDPO-0} to obtain
%
\begin{equation}
    \mathcal{L}_{\mathrm{ML-RDPO}} (\theta) := \sum^{N}_{i=1}\underbrace{ - \log \sigma (\beta\Delta^i_{\theta})}_{\text{Ranking term}} + \underbrace{\frac{\br{\Delta^i_{\hat{r}} - \beta\Delta^i_{\theta}}^2}{2\mathbb{V}}}_{\text{Rating term}}
\label{eq:ML-RDPO}. 
\end{equation}

Then,~ML-RDPO outputs the policy with parameters that minimize $\mathcal{L}_{\mathrm{ML-RDPO}}$.
Note that $\mathcal{L}_{\mathrm{ML-RDPO}}$ is the sum of two terms, one that only depends on the ranking information and one that is solely computed from ratings. This allows to easily apply ML-RDPO even when the ranking and rating information are not provided for the same prompt-response tuples (unlike the algorithms derived in Section~\ref{subsec:RDPO}). With a careful look at the ML-RDPO loss in~\eqref{eq:ML-RDPO}, we notice that the ranking and rating terms coincide with the DPO loss and the Distilled-DPO loss weighted by $\mathbb{V}^{-1}$, respectively. Thus, the more confidence we have in the rating information, the smaller the value of $\mathbb{V}$ should be. At an intuitive level, $\mathbb{V}$ in ML-RDPO acts similarly to $\beta_1$ in RDPO.

{\bf Tuning hyperparameters $\beta_1$ and $\mathbb{V}$.} As discussed above, the value of $\beta_1$ and $\mathbb{V}$ depend on the accuracy of the rating information, which might be unknown. However, we managed to tune them rather easily in our experiments. We ablated $5$ different values of each of these hyperparameters on one model, namely Zephyr-7B, and used the best choice for the rest of the models of similar size (Llama-3.1-8B and Mistral-7B). We will have a detailed discussion of the hyperparameters selection in \Cref{app:zephyr_ablation}.


\section{Theoretical results}
\label{sec:analysis}

Our main theoretical contribution is to prove that augmenting the DPO dataset with rating gaps $\Delta_{\hat{r}} (x_i,a^+_i, a^-_i)$ allows us to achieve better statistical guarantees. However, the benefit depends heavily on the quality of the rating estimates $\hat{r}$. In order to formalize this concept, we need the following assumption about the dataset generation.
\begin{assumption}\textbf{(i.i.d.~dataset generation)}
Each pair of candidate actions is sampled i.i.d.~from the policy $\pidata$, i.e.,~$a_i,a'_i \sim \pidata(\cdot|x_i)$. 
\end{assumption}
At this point, we introduce the approximation error $\mathrm{Err}(\hat{r})$, defined as
\begin{equation*}
\mathrm{Err}(\hat{r}) := \mathbb{E}_{\!\!\!\!\!\substack{x \sim \initial\\a,a'\sim\pidata(\cdot|X)}}\big[\big(\Delta_{r^\star}(x,a,a') - \Delta_{\hat{r}}(x,a,a') \big)^2\big].
\end{equation*}
Intuitively, $\mathrm{Err}(\hat{r})$ quantifies the rating estimate quality, i.e. how much the reward differences $\Delta_{r^\star}$ and $\Delta_{\hat{r}}$ induced by rewards $r^\star$ and $\hat{r}$ are close to each other under the support of the reference policy $\piref$. Note that requiring small $\mathrm{Err}(\hat{r})$ is much weaker than requiring $\hat{r}(x,a)$ to be close to $r^\star(x,a)$ for all $x,a \in \X\times\A$. In practice, it is reasonable to assume that $\mathrm{Err}(\hat{r})$ is small for datasets like \texttt{ultrafeedback\_binarized}\footnote{\url{https://huggingface.co/datasets/HuggingFaceH4/ultrafeedback_binarized}} in which $\hat{r}$ is obtained by querying a (rather) reliable evaluator, such as Chat GPT-4~\citep{cui2023ultrafeedback}. We will make use of $\mathrm{Err}(\hat{r})$ to characterize the statistical rate attained by RDPO in~\Cref{subsec:RDPO-Theory}. Before delving into the proofs, let us state an important assumption that we leverage in all our technical results. 

\begin{assumption}
\label{ass:policy_realizability_main}  
Let $0 \leq R_{\min}\leq r^\star(x,a) \leq R_{\max}$ for all $x,a\in \X\times\A$, and $\pi^\star_\beta$ be defined as in~\eqref{eq:Optimal-Policy} for any $\piref$ and reward $r:\X\times\A\rightarrow[R_{\min},R_{\max}]$. Then, we assume that we have access to a policy class $\Pi$ such that $\pi^\star_\beta \in \Pi$. 
\end{assumption}

This assumption indicates that the policy class is expressive enough to realize the closed-form solution of~\eqref{eq:RLHF} corresponding to any bounded reward function. Note that this assumption is standard in our setting---it is even needed for deriving the DPO loss from the RLHF problem in~\eqref{eq:RLHF}.



\subsection{Theoretical Guarantees for RDPO} 
\label{subsec:RDPO-Theory}

After stating our main assumptions, we are now ready to present our theoretical results starting from RDPO. In order to report our bounds more clearly, let us define $\mathrm{Err_{DPO}}(N,\delta) := \frac{ c R^2_{\max}e^{4 R_{\max}} \log(\abs{\Pi}/\delta)}{N} $ for some $c>0$. The definition comes from the fact that DPO finds a policy $\pi_{\mathrm{DPO}}$ enjoying the following guarantees with probability at least $1-\delta$ (see e.g.,~\citealt[Theorem 3.1]{huang2024correcting}),
\begin{equation}
J_\beta(\pi^\star;r^\star) - J_\beta(\piout;r^\star) \leq \mathcal{O}\br{\sqrt{ C^{\max} \cdot \mathrm{Err_{DPO}}(N,\delta) }},
\label{eq:dpo_rate}
\end{equation}
where for some $\pi$, we define the concentrability coefficient as $C^\pi = \sum_{x \in \X} \initial(x)\sum_{a\in \A} \pi(a|x) \abs{\frac{\pi(a|x)}{\pidata(a|x)}}$. Moreover, for the optimal policy $\pi^\star$, we use the notation $C^\star := C^{\pi^\star}$, and denote the maximum concentrability by $C^{\max} := \max_{\pi\in\Pi} C^{\pi}$. 
We now state our main result for RDPO. For technical reasons, we consider minimizing the RDPO loss in~\eqref{eq:RDPO} over the constrained policy class $\Pi'\subset\Pi$ such that $\pi \in \Pi'$ implies that $\beta \Delta_{\log \pi /\piref} - \frac{\beta}{\beta_1}\Delta_{\hat{r}} $ is in the interval $ [-R_{\max}, R_{\max}] $.
\begin{restatable}{theorem}{RDPO}
\label{thm:ratings_dpo}
For any $\beta\in (0,\infty)$ , let $\piout$ be the policy with parameters minimizing $\mathcal{L}_{\mathrm{RDPO}}$ over the policy class $\Pi'$ with hyperparameter $\beta_1 = \beta \cdot \nicefrac{\mathrm{Err}(\hat{r})}{\mathrm{Err_{DPO}}(N,\delta)}$ 
Then, it holds that with probability at least $1-\delta$, 
 the suboptimality gap,i.e. $J_{\frac{\beta\beta_1}{\beta+\beta_1}}(\pi^\star;r^\star) - J_{\frac{\beta\beta_1}{\beta+\beta_1}}(\piout;r^\star)$, is upper-bounded by 
\begin{equation}\mathcal{O}\br{\sqrt{C^{\max} \cdot \min\bc{ \mathrm{Err_{DPO}}(N,\delta),\mathrm{Err}(\hat{r})}}}. \label{eq:RDPO_rate}
\end{equation}
\end{restatable}
%
The proof can be found in \Cref{app:proof_rdpo_main}.
Moreover \Cref{thm:ratings_dpo_app} in \Cref{app:Cstar} shows a stronger theoretical bound, where $C^{\max}$ is replaced by the smaller quantity $C^\star$ that can be achieved via an additional $\chi^2$ regularization and the assumption that $\pidata$ coincides with $\piref$. However, we noticed that the $\chi^2$ regularization can make the algorithm numerically unstable in practice. As a final comment, we notice that the restriction over $\Pi'$ conveys additional stability to the algorithm and has a positive practical effect (see \Cref{app:exp_constrain}). 

We now discuss the important theoretical properties that the rate in~\eqref{eq:RDPO_rate} features.

\paragraph{Acceleration.} When the ratings $\hat r$ are accurate, it is expected that $\mathrm{Err}(\hat{r}) \leq \mathrm{Err_{DPO}}(N,\delta) $. Then, the bound in \Cref{thm:ratings_dpo} predicts a rate of order $ \mathcal{O}\br{\sqrt{ C^{\max}  \cdot \mathrm{Err}(\hat{r})}}$, which is faster than the DPO rate in~\eqref{eq:dpo_rate}.

\paragraph{Robustness.} Even if the ratings $\hat r$ are inaccurate, i.e.,~$\mathrm{Err}(\hat{r}) \geq \mathrm{Err_{DPO}}(N,\delta)$, the guarantees for RDPO do not collapse. Indeed, \Cref{thm:ratings_dpo} predicts a rate equal to $\mathcal{O}\br{ \sqrt{C^{\max}\cdot \mathrm{Err_{DPO}}(N,\delta) }}$, which is the same bound achieved by DPO-like algorithms.
 

\paragraph{Principled choice of $\beta$ and $\beta_1$ in practice.} 
\Cref{thm:ratings_dpo} suggests to set the rating trust $1/\beta_1$ to the inverse of the regularization parameter (i.e.,~$1/\beta$) times the errors ratio $\nicefrac{\mathrm{Err_{DPO}}(N,\delta)}{\mathrm{Err}(\hat{r})}$.  
This finding gives us the following practical guideline: 
\begin{tcolorbox}[
  colback=blue!10!white,
  colframe=blue!80!black,
  width=\linewidth
]
\textbf{\textcolor{blue}{Practical Take Away:}} The more we trust the rating over the ranking, the larger $\beta/\beta_1$ should be. As an example, $\beta$ can be set to the default value in DPO, e.g., $\beta =0.1$, and $\beta_1$ can be chosen consequently to obtain the desired ratio $\beta/\beta_1$.
\end{tcolorbox}


\subsection{Theoretical Guarantees for ML-RDPO}

We now prove our theoretical results for ML-RDPO. 
Similar to the derivation of ML-RDPO, we assume that $z_i$ and $\Delta^i_{\hat{r}}$ are conditionally independent given $x_i, a_i, a'_i$, and rating gaps are Gaussian $\Delta^i_{\hat{r}} \sim \mathcal{N}( \Delta^i_{r^\star}, \mathbb{V})$ for all $i \in [N]$. For ML-RDPO, the theoretical analysis also requires restricting the optimization domain to the parameter spaces that induce a policy class $\tilde{\Pi}\subset\Pi$ with the property that $\beta \Delta_{\log \pi/\piref}\in[-R_{\max}, R_{\max}]$ for all $\pi \in \tilde{\Pi}$.  

\begin{restatable}{theorem}{MLRDPO} 
\label{thm:DPO+distilled_main}
Let $\piout$ be the policy with parameters minimizing $\mathcal{L}_{\mathrm{ML-RDPO}}$ over the policy class $\tilde{\Pi}$. 
Then for any value of $\beta\in(0,\infty)$, with probability at least $1-\delta$, the suboptimality $J_{\beta}(\pi^\star;r^\star) - J_\beta(\piout;r^\star) $ is upper-bounded by 
\begin{equation*}
\tilde{\mathcal{O}}\br{\sqrt{\frac{C^{\max} \min\bc{ e^{R_{\max}}R^2_{\max},R_{\max}^2 + \mathbb{V}} \log (\abs{\Pi}/\delta)}{N} }}.
\end{equation*}
\end{restatable}

The proof can be found in \Cref{app:proofMLRDPOmain}. Moreover, as we show in \Cref{thm:DPO+distilled_app} in \Cref{app:CstarMLRDPO}, with an additional $\chi^2$ divergence term, we can obtain a tighter bound where $C^{\max}$ is replaced with a smaller quantity $C^\star$. We now highlight some interesting features of the ML-RDPO bound in \Cref{thm:DPO+distilled_main}. 

\paragraph{Acceleration.} As for RDPO, ML-RDPO enjoys both the acceleration and robustness properties. Regarding acceleration, it can be seen that when the variance $\mathbb{V}$ is low, i.e.,~$\mathbb{V} \leq \mathcal{O}\br{e^{R_{\max}}R^2_{\max}}$, \Cref{thm:DPO+distilled_main} shows a better rate for ML-RDPO than the one for DPO given by~\eqref{eq:dpo_rate}. In particular, when the variance is small, we have the following exponential improvement:
\begin{tcolorbox}[
  colback=blue!10!white,
  colframe=blue!80!black,
  width=\linewidth
]
\textbf{\textcolor{blue}{Exponential Improvement:}} ML-RDPO incurs only a polynomial dependence on $R_{\max}$ and $\mathbb{V}$, while the sample complexity of all known algorithms that only use ranking information, such as DPO, scales with $e^{R_{\max}}$.
\end{tcolorbox}

In alignment using only ranking, the $e^{R_{\max}}$ dependence can be moved to lower order terms \cite{chen2025avoiding}, requiring online generation of new responses and it is conjectured to be unavoidable in the offline setting \cite{das2024active}.

\paragraph{Robustness.} We note that the proof technique in \Cref{thm:DPO+distilled_main}
can be used to prove a new bound for Distilled-DPO~\citep{fisch2024robust} of order $\mathcal{O}\br{\sqrt{C^{\max} (\mathbb{V} + R_{\max}^2)\log(\abs{\Pi}/\delta)/N}}$. This allows us to conclude that ML-RDPO with appropriately tuned $\beta$ and $\mathbb{V}$ is no worse than the best between DPO and ML-RDPO. 

\paragraph{Comparison with RDPO.}
As we previously mentioned, $\mathcal{L}_{\mathrm{ML-RDPO}}$ is written as the sum of a rating term and a ranking term. This fact makes ML-RDPO easily generalizable to the setting where the rating feedback is observed only for certain responses in the dataset. On the contrary, RDPO cannot be directly applied in this setting because $\mathcal{L}_{\mathrm{RDPO}}$ does not decompose into rating and ranking terms. We elaborate on this in \Cref{app:hetereogenous_data} and show the efficiency of ML-RDPO with partial rating information in practice in \Cref{fig:ml-rdpo-missing}. Of course, a downside of ML-RDPO is the Gaussian assumption on the rating, which is avoided by  RDPO.


\section{Related Work}
\label{sec:related}

Before empirically evaluating our proposed algorithms, we provide a brief overview of the related work. Some of the algorithms we discuss in this section will serve as baselines in our experiments. A more comprehensive literature review is provided in~\Cref{app:related}.

Among the algorithms that only work with ranking information, we use DPO~\citep{rafailov2023direct}, which leverages the BT model, and IPO~\citep{azar2024general}, which is applied to general preference models, as baselines. Moreover, we compare with SIMPO~\citep{meng2024simpo}, which modifies DPO by removing the regularization towards $\piref$ and dividing the log ratios by the length of the responses in the dataset. Among the ratings-only alignment algorithms, we compare with Distilled DPO (DDPO)~\citep{fisch2024robust}, which uses the squared-loss to learn a policy whose difference of log ratios is close to the difference of observed ratings. Interestingly, DDPO is quite similar to our RIPO algorithm, albeit they derived it from different perspectives.

We also consider recently proposed algorithms that utilize both rating and ranking information. We use as a baseline RPO~\citep{adler2024nemotron}, which minimizes $\mathbb{KL}(\mathrm{Bern}(\sigma(\Delta^i_\theta)), \mathrm{Bern}(\sigma(\Delta^i_{\hat{r}}) ))$. 
Unfortunately, RPO does not enjoy the performance guarantees attained by our RDPO and ML-RDPO algorithms, and more importantly, it is difficult to tune due to the non-smoothness of its loss function.
We also compare with another rating-augmented alignment algorithm dubbed MAPPO~\citep{lan2025mappo} that minimizes 
the DPO loss up to the fact that the $\beta$ for the rejeced response is multiplied by the rating gap.



\begin{figure*}[t] 
\centering
\begin{tabular}{cc} 
\subfloat{
        \includegraphics[width=0.45\textwidth]{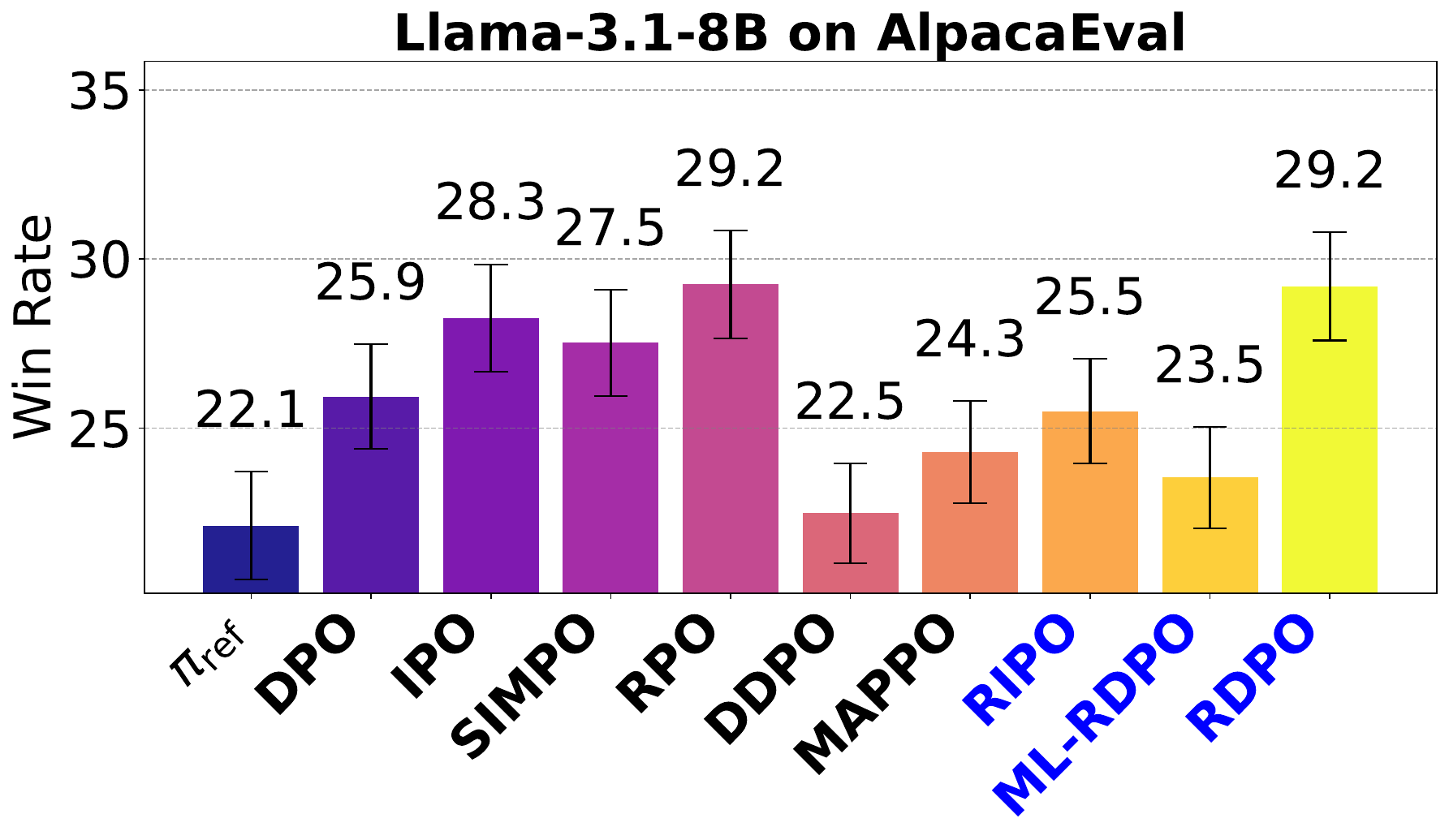}
    }
    &
    \subfloat{
        \includegraphics[width=0.45\textwidth]{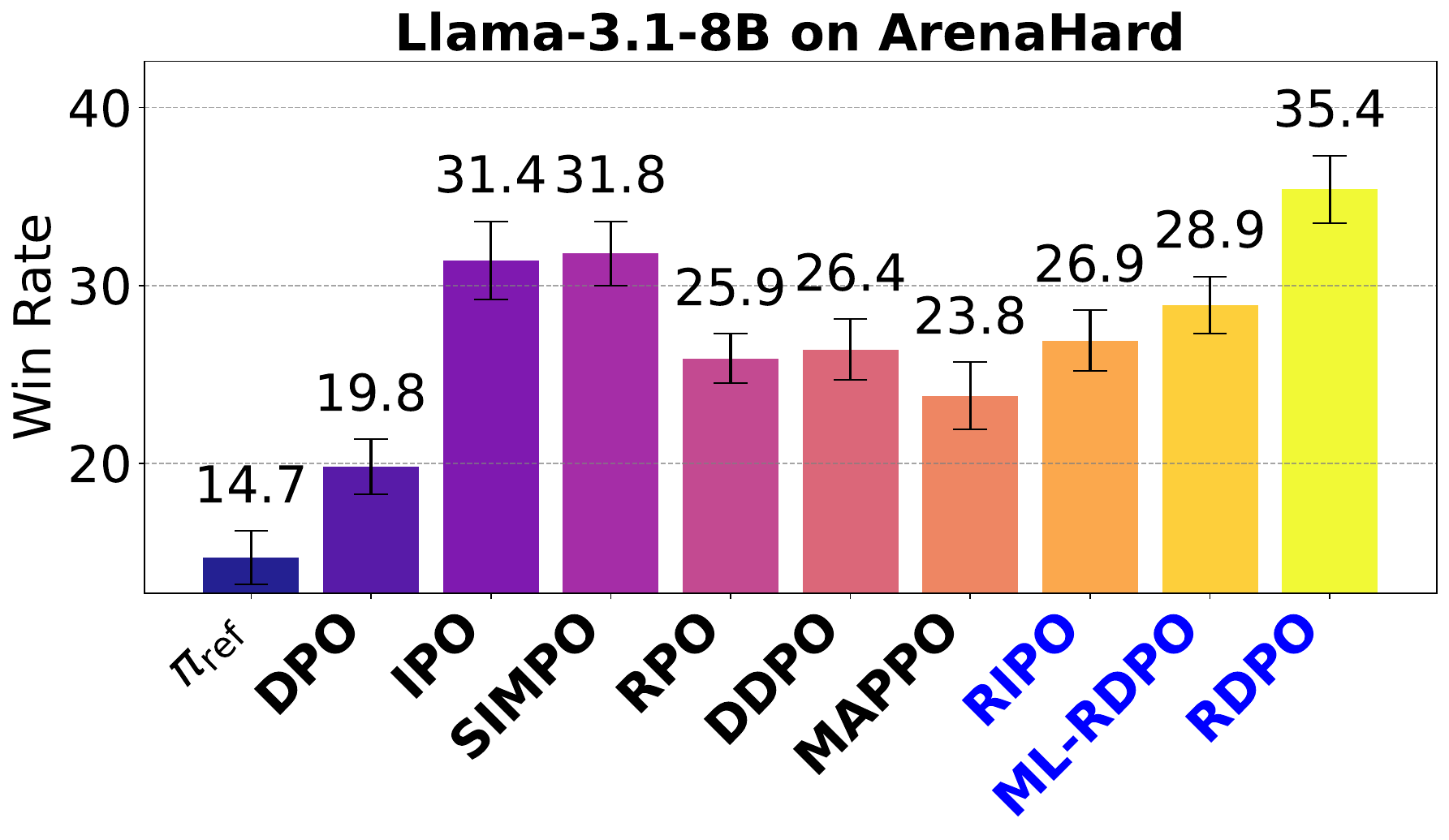}
    }
\\
 \subfloat{
        \includegraphics[width=0.45\textwidth]{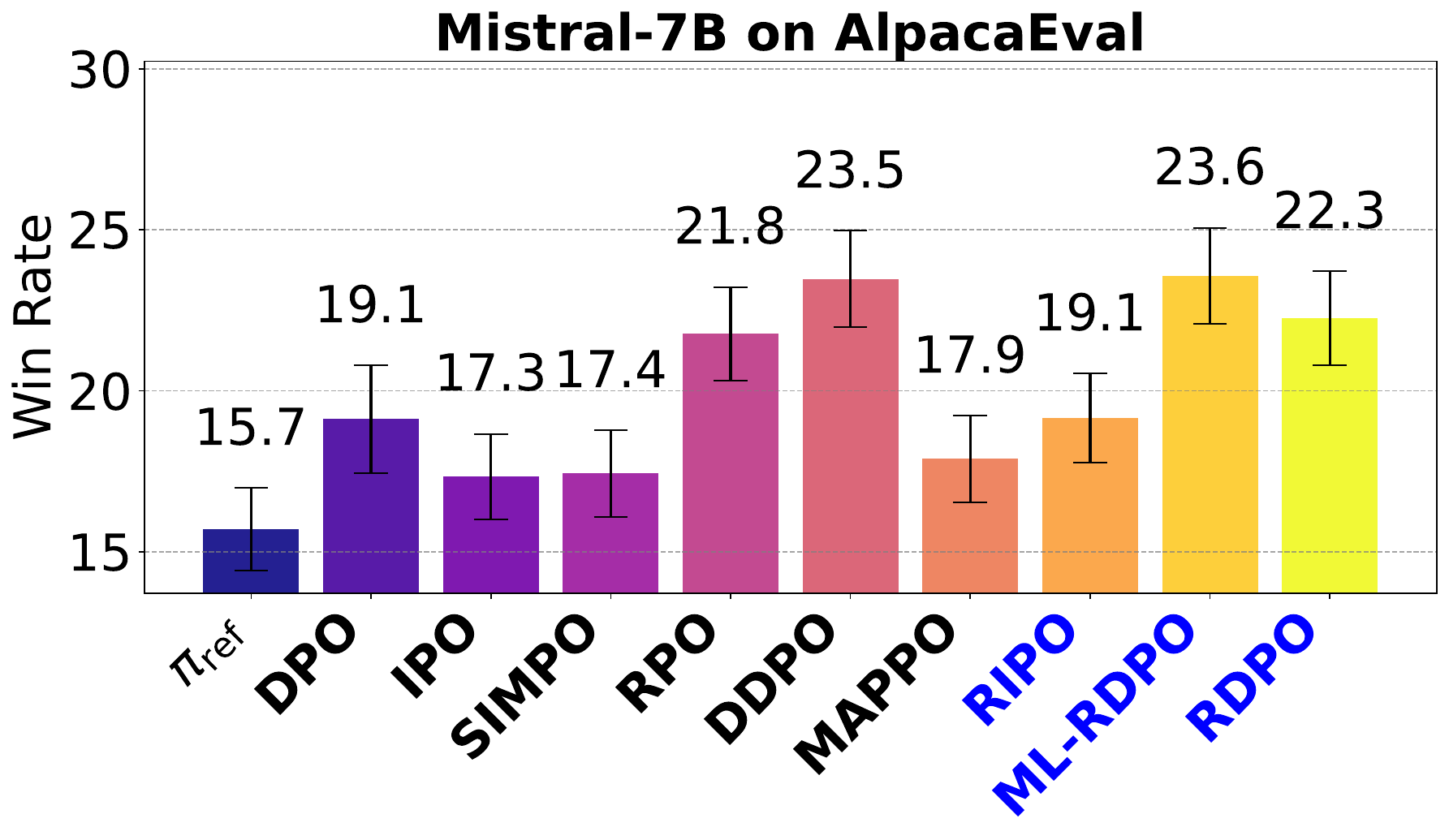}
}
    &
    \subfloat{
        \includegraphics[width=0.45\textwidth]{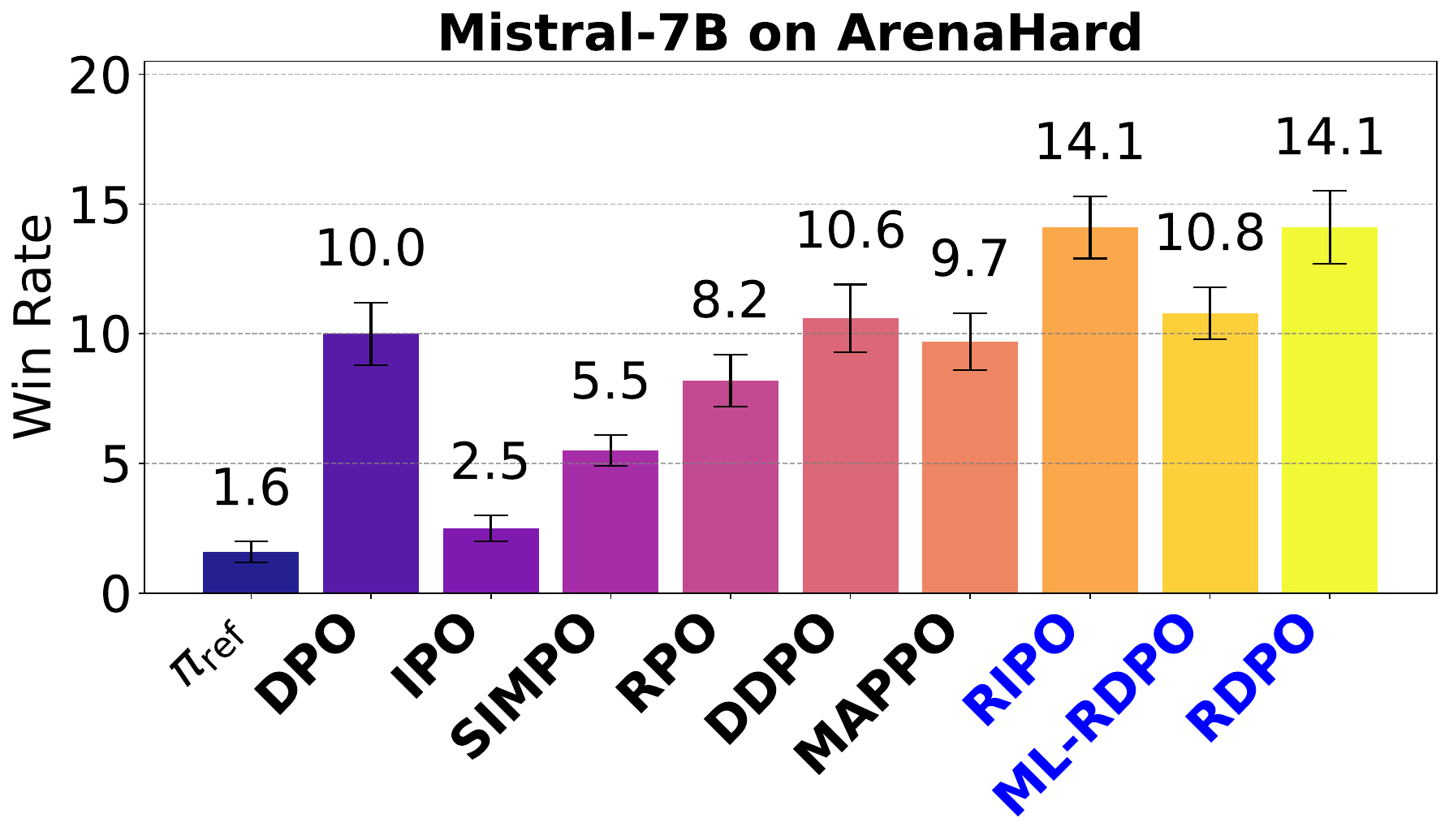}
    } \\
    \subfloat{
        \includegraphics[width=0.45\textwidth]{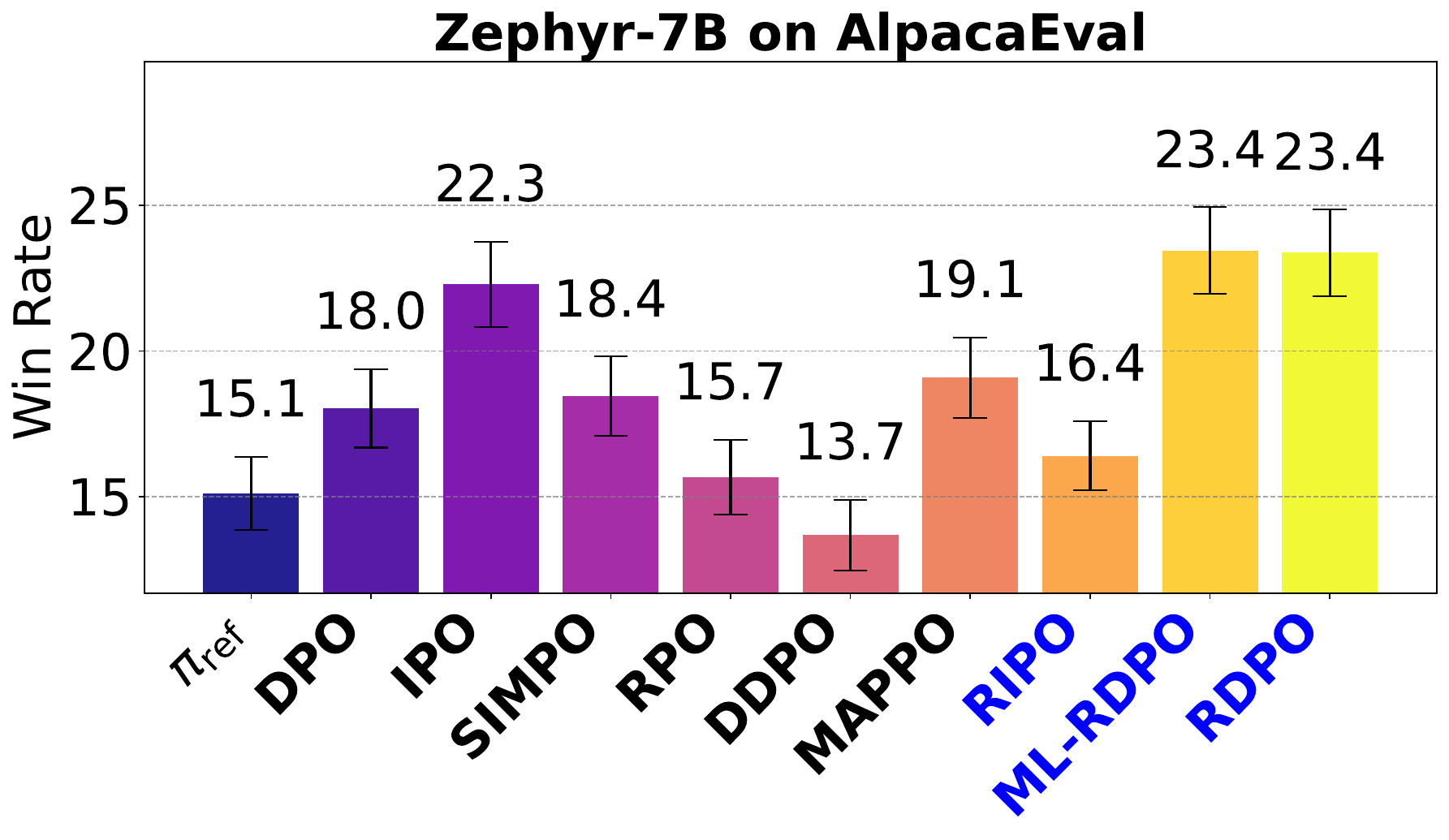}
    }
    &
    \subfloat{
        \includegraphics[width=0.45\textwidth]{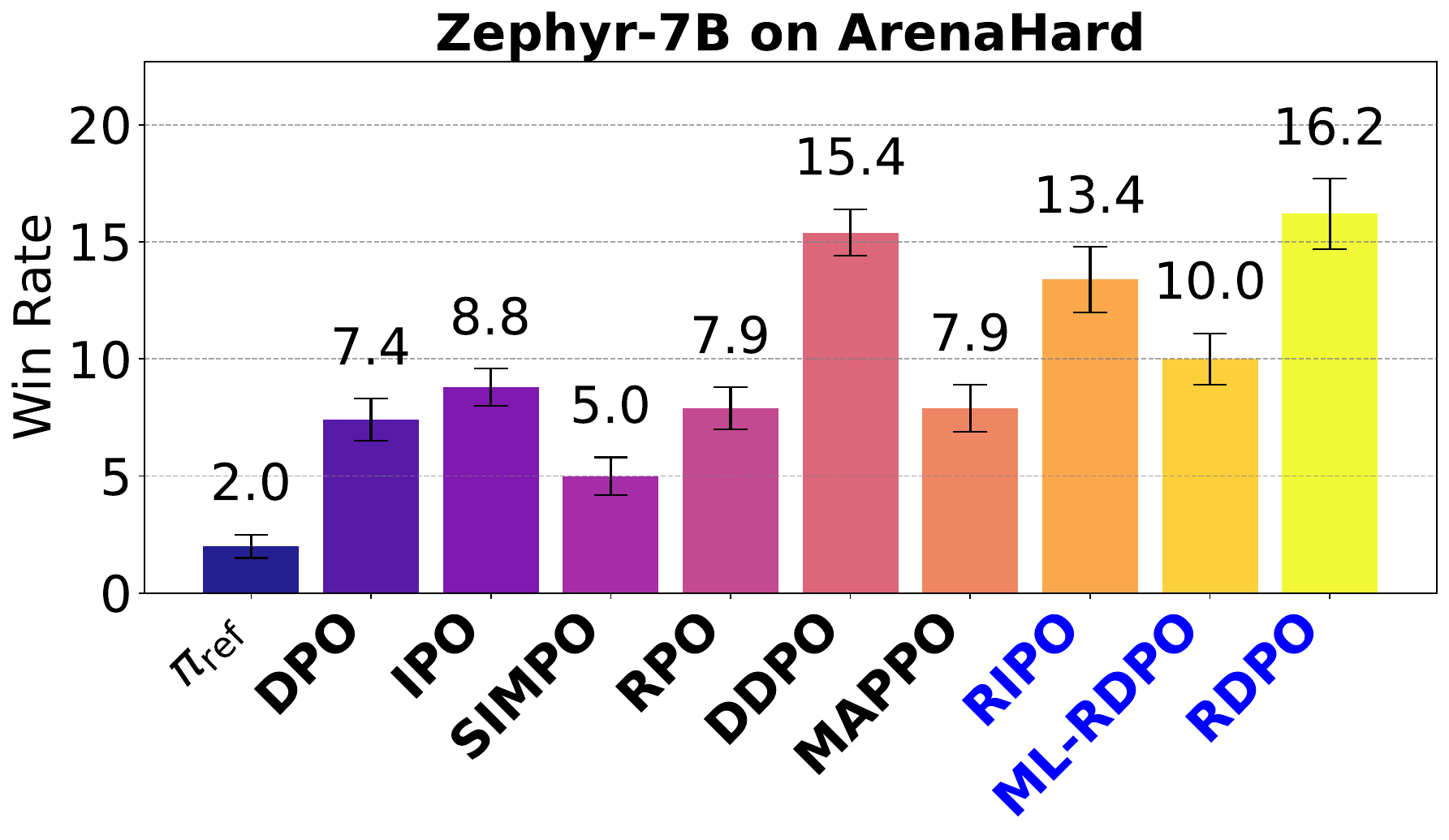}
    }
\end{tabular}
\caption{Win rates against GPT4, judged by Claude-Sonnet-3.5 v2, in AlpacaEval and ArenaHard.}
\label{fig:gpt4scores}
\end{figure*}

\section{Experiments}
\label{sec:experiments}

In this section, we empirically evaluate our proposed algorithms, RDPO, RIPO, and ML-RDPO, and demonstrate their solid performance in comparison to a number of DPO-style methods across a wide range of LLMs and evaluation benchmarks. The experiments verify our theoretical findings, especially in terms of robustness and faster convergence rate (acceleration). 


\paragraph{Experiments to assess acceleration.}
Here we aim at verifying the favorable convergence properties of our algorithms shown in Theorems~\ref{thm:ratings_dpo} and~\ref{thm:DPO+distilled_main}. We consider having access to high-quality rating information, $\hat{r}$, in the \texttt{ultrafeedback\_binarized} dataset, and compare our methods to a diverse set of baselines: DPO, IPO, and SIMPO (only use ranking information); DDPO~\citep{fisch2024robust} (only uses rating information); and finally RPO~\citep{adler2024nemotron} and MAPPO~\citep{lan2025mappo} (use both sources of information).


For most algorithms, we set the DPO hyperparameter as $\beta=0.1$ and learning rate to $10^{-6}$. For RDPO and ML-RDPO, we tune $\beta_1$ and $\mathbb{V}$ in Zephyr-7B, as shown in \Cref{tab:ablation7B} in Appendix~\ref{app:zephyr_ablation}, and keep their best values, $\beta_1 = 1/10$ and $\mathbb{V} = 1/100$, for the experiments with the other two models: Llama-3.1-8B and Mistral-7B. To obtain good results with DDPO, we reduced the learning rate to $10^{-7}$, while for IPO we used length normalization, larger $\beta$, i.e.,~$\beta=0.8$, and a smaller learning rate $10^{-7}$.

In \Cref{fig:gpt4scores}, we start the fine-tuning from each of the specified base models (Llama, Mistral, Zephyr), and then evaluate the fine-tuned checkpoints in terms of their win-rates against GPT-4, judged by Claude-Sonnet-3.5 v2, over {\em AlpacaEval} and {\em ArenaHard} benchmarks. RDPO performs consistently well, being the best performing method in $3$ out of the $6$ experiments, and the second or third best in the others. ML-RDPO is less consistent, but it is the best performing method for Mistral-7B and Zephyr-7B in AlpacaEval. RPO is the best for Llama-3.1-8B in AlpacaEval, but the performance is sub-optimal in the other cases. 
The other methods that are only based on either rating or ranking information are less efficient in all experiments. 

\paragraph{Experiments to assess robustness.}
Rating information is harder to collect and can be noisier than its ranking counterpart. Therefore, we evaluate the robustness of RDPO by corrupting the rating information in two different forms. In the first experiment, we swap the score of the chosen and discarded responses for a certain portion ($0\%$, $10\%$, $30\%$) of the training data. As can be seen in \Cref{fig:corruption_via_swaps}, when we use a small value $\beta_1=1/40$ (high trust in rating information), the performance of RDPO worsens quickly as the number of swaps is increased. On the other hand, a larger $\beta_1 = 1/10$ makes the performance of RDPO robust and (almost) unaffected by this corruption. We can draw a similar conclusion from our second experiment shown in \Cref{fig:corruption_via_noise}, where we add zero-mean Gaussian noise with increasing variance to the rating information (UltraFeedback scores). Overall, these experiments confirm our theoretical observation that no matter how ``bad'' the rating information is, i.e.,~for any value of $\mathrm{Err}(\hat{r})$, there exists a choice of $\beta_1$ such that RDPO is guaranteed to perform at least as well as DPO.
\begin{figure}
    \centering
\begin{tabular}{cc} 
\subfloat[Corruption via swaps, $\piref = \texttt{Llama-3.1-8B}$. \label{fig:corruption_via_swaps}]{
        \includegraphics[width=0.45\textwidth]{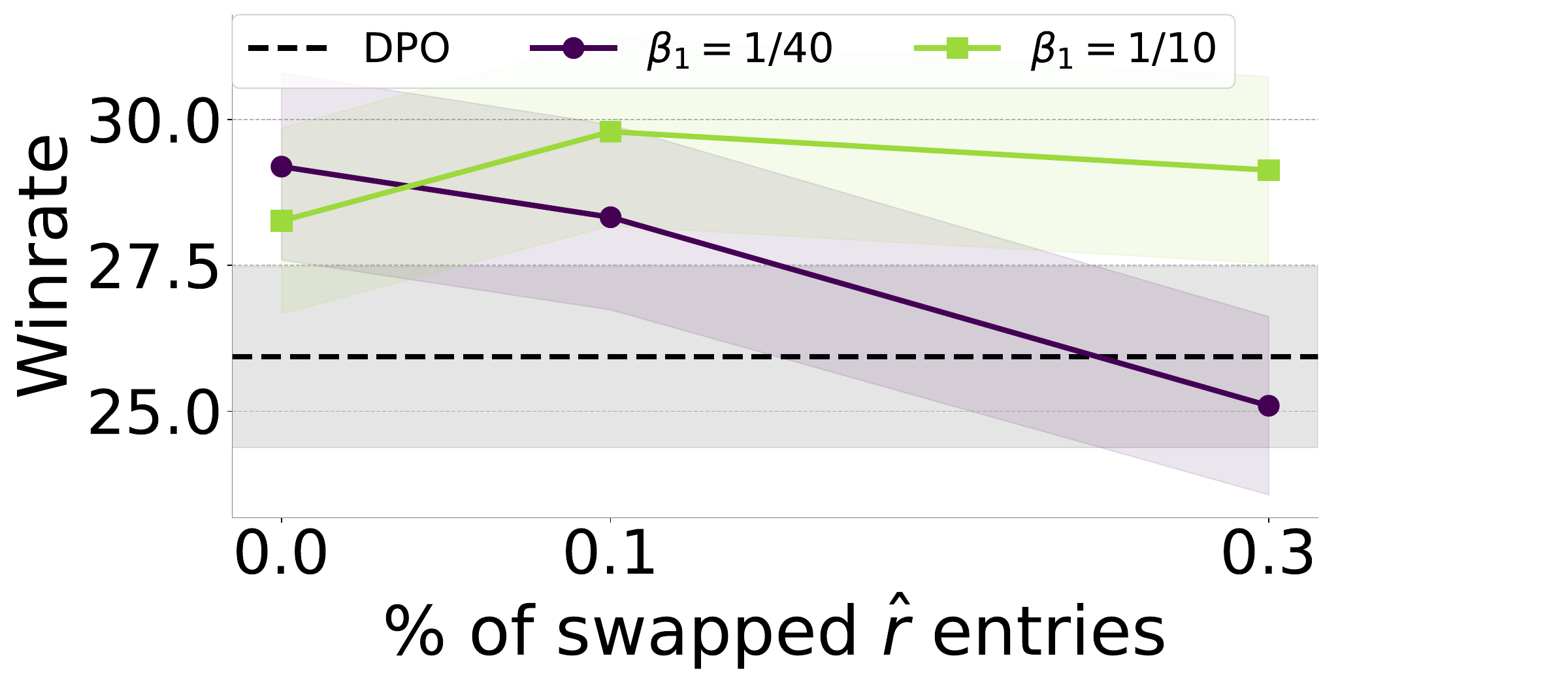}
    }
    &
    \subfloat[Corruption via noise, $\piref = \texttt{Llama-3.1-8B}$.\label{fig:corruption_via_noise}]{
        \includegraphics[width=0.45\textwidth]{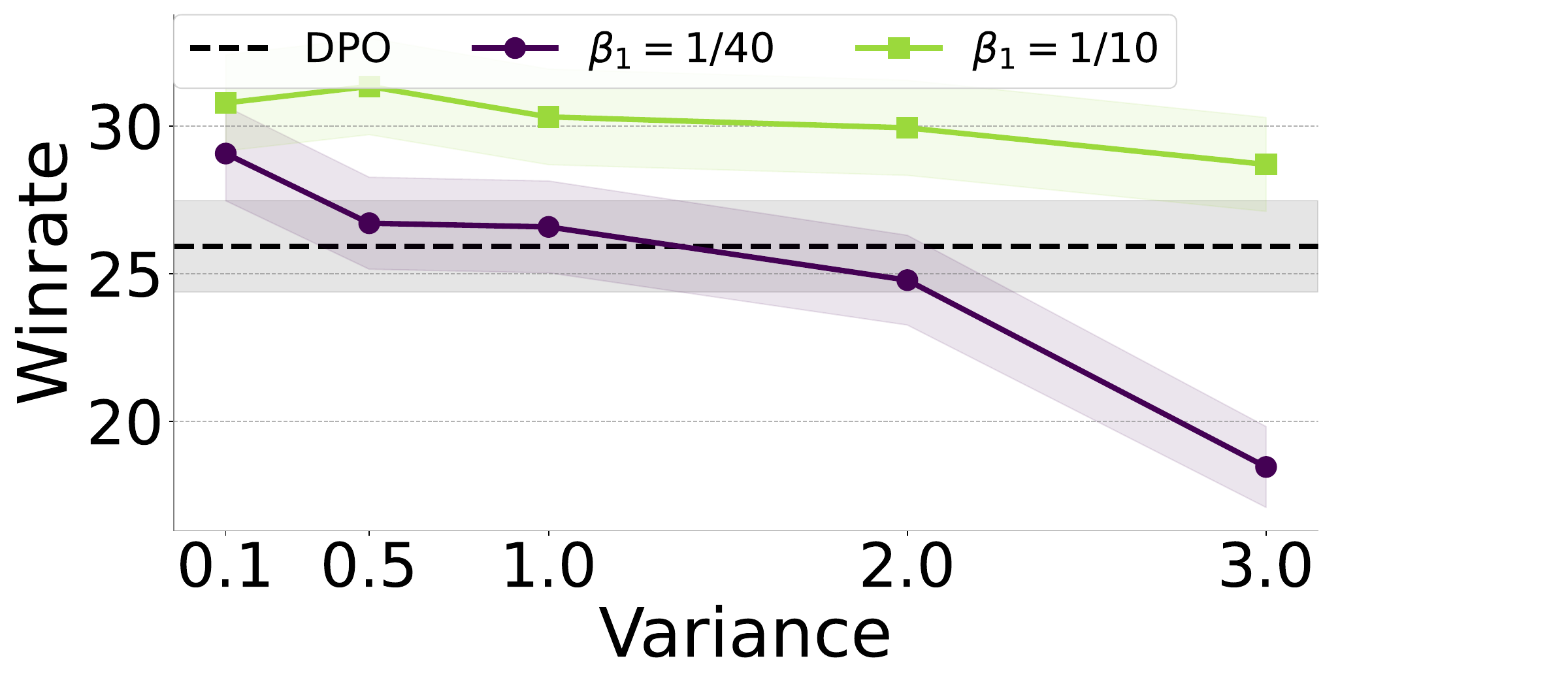}
    }
    \end{tabular}
    \caption{Robustness to inaccurate ratings experiments}
\end{figure}

\paragraph{Experiment with missing ratings.}
Given that collecting rating information is more complicated than ranking (preference) data, it is reasonable to expect that in a practical scenario, some data points come only with ranking labels and without the {\em rating gap} information.

\begin{wrapfigure}{r}{0.5\textwidth}
    \centering
   \begin{tabular}{c}
       \subfloat{%
\includegraphics[width=.95\linewidth]{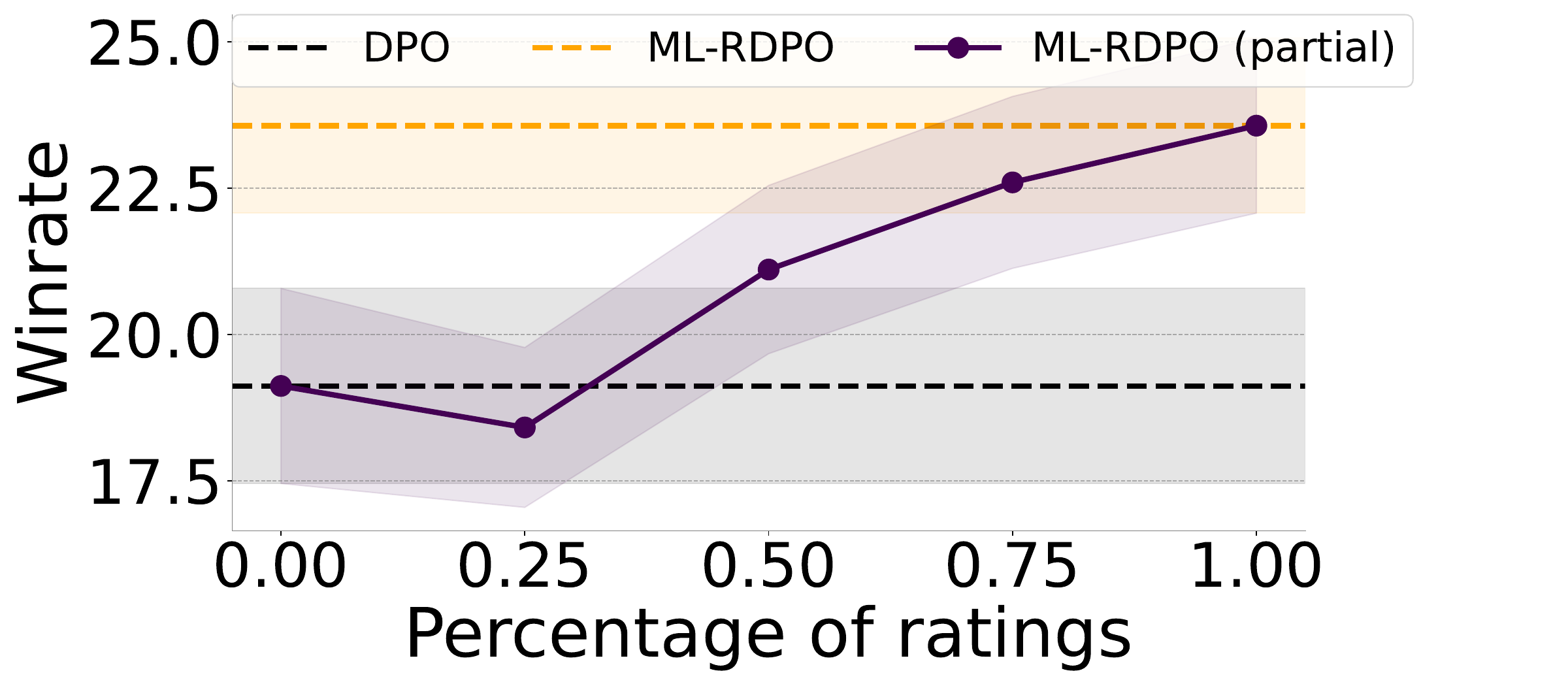}
        }
 \end{tabular}
  \caption{\label{fig:ml-rdpo-missing} ML-RDPO on \texttt{ultrafeedback} with partial ratings observation, $\piref = \texttt{Mistral-7B}$.}
\end{wrapfigure}
We set up an experiment in \Cref{fig:ml-rdpo-missing} to show the capability of ML-RDPO in handling training data with partially scored labels. In particular, we show that even when only $50\%$ of the training data contains rating information, ML-RDPO is preferable to DPO. Recall that, unlike RDPO, ML-RDPO does not need to have access to the rating gap for all the data-points in its training set. We provide the derivation of the ML-RDPO algorithm where some of the rating information is missing in \Cref{app:hetereogenous_data}, and the theoretical guarantees for this setting in \Cref{app:guarantees_hetero}.

\paragraph{Additional experiments.}
We report additional experiments in \Cref{app:experiments}. In particular, in \Cref{sec:battles}, we report the pairwise win rates between any pair of considered algorithms at 7/8B models. This serves as additional verification of the good performance of RDPO and ML-RDPO. In \Cref{sec:noisy}, we repeat the corruption experiments on 
other choices of $\piref$, which confirms the trend observed in \Cref{fig:corruption_via_noise}, especially using Mistral-7B as $\piref$. In \Cref{sec:SmolLM2}, we show that the gains from rating information are even more visible at smaller scales, using the SmolLM2 models~\citep{allal2025smollm2smolgoesbig} of sizes 135M, 360M, and 1.7B, as $\piref$. In particular, RDPO achieves win rates over $90\%$ in all these cases. Finally, in \Cref{app:exp_constrain}, we show that minimizing a modified RDPO loss, which encourages the restriction to the smaller policy class $\Pi'$  needed for the proof of \Cref{thm:ratings_dpo}, leads to improved results for Llama-3.1-8B and Mistral-7B as $\piref$. In particular, we add a piecewise linear penalty to encourage the implicit reward difference, $\beta \Delta^i_\theta - \beta \Delta^i_{\hat{r}}/\beta_1$, to lie in the interval $[-R_{\max}, R_{\max}]$.


\section{Conclusions}
\label{sec:conclu}

We studied how DPO-style algorithms can efficiently leverage additional information in the form of {\em rating gap}, which informs the learner how much the preferred response is better than the rejected one. We presented three algorithms, RDPO, RIPO, and ML-RDPO, that we derived using two different principled approaches. Our theoretical and empirical results show the better performance of our algorithms compared to methods that only use ranking information, when the rating information is accurate. We also showed that our algorithms are robust to noise in the rating information, and can still perform well in this situation if their relevant hyper-parameters, $\beta_1$ and $\mathbb{V}$, are tuned properly.  

A quick look at the losses of our proposed algorithms, we notice that they only leverage the rating gap, $\Delta^i_{\hat{r}}$, and not the individual ratings $\hat{r}(x_i, a^+_i)$ and $\hat{r}(x_i, a^-_i)$. This is positive because these algorithms can work with a weaker form of feedback (rating gap vs.~individual rating), but as a consequence, they may not be able to use all the available information when the individual ratings are available.  
We leave developing principled algorithms, possibly inspired by our methods, that can make full use of the individual ratings as an interesting open question for future work.  

\newpage
\bibliographystyle{plainnat}
\bibliography{references}
\newpage
\appendix

\section*{Appendix}
\startcontents[appendices]
\printcontents[appendices]{}{1}{}
\section{Extended Related Works discussion}
\label{app:related}
This section outlines important related works that have been omitted from the main text due to space limitations.
\paragraph{Learning from preferences and ratings}
The use of rating gaps has received little attention at the current stage, mainly because of the concern that the data collection becomes more complicated and inevitably noisier. However, there are some works that go against this belief and have shown practical benefit in using rating information. For example, \cite{adler2024nemotron} introduced RPO and used it for the post training of a large (340B) language model after an initial post-training phase that was carried out from preferences only via DPO. In their case, the ratings allowed to use a large dataset with 300K examples with a rather permissive quality filtering on the chosen response. Their conclusion is therefore that bad quality chosen ( over even worse ones ) responses do not negatively affect the performance since they are accompanied by a low rate. Albeit RPO is effective in practice, as confirmed by our implementation, it does not enjoy the performance guarantees attained by RDPO and ML-RDPO, and it is difficult to tune due to the non-smoothness of its loss function.

RPO has also been generalized in \cite{sun2025reward} to show that it can be seen as a general framework capturing various existing algorithms. 
Remaining on the algorithmic side, ratings information has been leveraged in deriving MAPPO \cite{lan2025mappo}, which modifies the DPO objective using the rating gap to modify the regularization parameter $\beta$ in front of the log ratio of the loser response.
As another example, \cite{chen2024noise} provides a reward-based generalization of DPO dubbed NCA (Noisy Contrastive Alignment), which ensures that the probability of the winning response does not decrease, while this can (and in fact it does, very often) happen in DPO.  However, it is questionable if such behavior is always desirable. For example, our intuition is that this should not happen if the rating of the winner's response is low. Moreover, notice that the same effect can be achieved by using an SFT term within an Online DPO scheme \cite{pang2024iterative}. In this work, the authors assume rating feedback in order to assign the roles of winner and loser responses in the preference datasets generated iteratively by Online DPO without the need to resort to an LLM as a judge. 
Finally, rating information has been used within a data augmentation technique \cite{zhang2024reward} to double the size of the preference dataset available for DPO training.

Moreover, access to ratings has been used in the context of alignment from multi preferences, i.e. when multiple preferred and dispreferred responses are available in \cite{gupta2025wmpo,guptaampo} which introduced MPO and W-MPO.  In passing, we notice that when only one preferred response and one rejected one is available MPO and W-MPO reduce to DPO. Interestingly, if the definition of $\Delta W_{\mathrm{abs}}$ is replaced with $\Delta W(y) = S(y) - S_{\mathrm{mean}}(x)$ and only two responses are available then W-MPO reduces to our Ratings DPO.
In a similar vein, REFA \cite{gupta2024refa} generalizes SIMPO \cite{meng2024simpo} to the case of multiple preferred and discared responses. \cite{gupta2024refa}  also introduced a version of REFA, dubbed W-REFA, which allows the use of rating gaps.

\paragraph{Learning from ratings only}
Several methods have focused on datasets containing only the rating gap information and developed methods to compute the soft optimal policy despite the computational hardness of the partition function.
Examples are Distilled DPO \cite{fisch2024robust} that studied in offline setting, the squared loss minimization suggested in \cite{gao2024rebel,gao2024regressing} in the context of online RLHF for single or multi-turn, respectively.
A similar approach is derived in \cite{cai2023ulma}.
\paragraph{Other variants of DPO with additional information} 
\cite{song2024hybrid} studied which kind of coverage conditions are necessary for DPO, IPO and in general purely offline alignment methods. Moreover, they relax the notion of coverage needed, resorting to online interaction with the environment but without invoking the preference oracle. This allows to conclude that even without collecting new preferences, online access to the environment is beneficial for DPO.
Similarly, \cite{chang2024dataset} studied preference optimization from the RLHF perspective when new trajectories ( from any desired initial prompt) can be collected.

Furthermore, \cite{liu2024provably} tackled the likelihood displacement problem and mitigated the issue by adding to the algorithm an imitation learning component that increases the likelihood of generating the responses observed in the dataset. \cite{cen2024value} introduced a similar regularization mechanism by changing the maximum likelihood estimation problem for the reward function. In particular, they look for the reward that maximizes the likelihood of the observed preferences and jointly makes the cumulative reward achievable by the best policy as high as possible. 
\paragraph{Provably learning from preferences}
Our work establishes some theoretical guarantees under single policy concentrability only. It is therefore related to the line of literature started by \cite{zhan2023provable} that studied the problem of offline alignment from a statistical learning viewpoint. The approaches in \cite{zhan2023provable, zhu2023principled} achieve indeed a single policy concentrability guarantee style. However, their algorithm is not implementable because it involves estimating confidence sets for the reward function and the transition dynamics in the multi-turn setting. Moreover, once the confidence sets are computed, they require solving a robust MDP problem \cite{iyengar2005robust,wiesemann2013robust} using the confidence sets as uncertainty sets. Although some efficient policy gradient methods \cite{li2023policy, kumar2025dual,viano2021robust,viano2022robust} and value-based methods \cite{zouitine2024solving,clavier2024near} have been developed, this approach is considered not scalable for LLMs.
Alternatively, \cite{li2023reinforcement} implements pessimism, implementing pessimistic value iteration via elementwise bonuses similar to what is done in the optimistic setting with UCBVI \cite{azar2017minimax}. While more elegant theoretically, this approach does not solve the scalability issue. A  more practical pessimistic approach is obtained in \cite{huang2024correcting} via $\chi^2$ regularization, which also serves as a source of inspiration for our analysis. 

\paragraph{Learning from preferences imposing a margin}
Our method is also related to all variants of DPO that imposes that the implicit reward of the chosen completion is higher that the discarded completion one by at least a fixed margin. Examples are IPO that imposes the margin to be \emph{exactly} $1/2$ via the squared loss, in the training of Llama2 \cite{touvron2023llama} they use three possible discrete values for the margin in the log sigmoid loss. Notice that the log sigmoid does not penalize the overestimation of the margin so it encourages the implicit reward difference to be at least as high as the margin but not to match the exact value. SimPO \cite{meng2024simpo} uses a fixed margin as in IPO but uses the log sigmoid loss instead of squares and neglects the effect of the reference policy in the definition of the implicit rewards.
RDPO presented in this document can be seen as a generalization of the above approaches using a margin which is prompt completions dependent and takes values in a continuous set.
\section{Generalizations of RDPO and ML-RDPO }
In this appendix, we present two important extensions on RDPO and ML-RDPO. First, we present a generalization that allows to change the KL regularizer to any $f$-divergence. In particular, using as regularizer the sum of KL and $\chi^2$ divergence allows to prove statistical guarantees under single policy concentrability.
Secondly, we show that we can handle disjoint datasets of rating and ranking. That is, we can apply our algorithm where the data are organized in two dataset: a classic ranking dataset $\mathcal{D}_{\mathrm{rank}}=\bc{x_i, a^+_i, a^-_i}^{N_{\mathrm{rank}}}_{i=1}$ and a rating only dataset $\mathcal{D}_{\mathrm{rat}}=\bc{x'_i, \tilde{a}_i, \bar{a}_i, \hat{r}(x'_i, \tilde{a}_i) - \hat{r}(x'_i, \bar{a}_i) }^{N_{\mathrm{rat}}}_{i=1}$.
We show that in this setting RDPO requires the intermediate step of estimating the unobserved ratings. That is, the rating gaps over the ranking dataset which are $\bc{ \hat{r}(x_i, a^+_i) -   \hat{r}(x_i,a^-_i) }$. 
Surprisingly, ML-RDPO avoids such step completely making it a more attractive alternative to handle this setting. We elaborate on this in \Cref{app:hetereogenous_data}. 
\subsection{Different choices for the regularizer in Eq.~\ref{eq:RLHF}}
\label{app:chi2PO}
A natural generalization of RDPO can be derived replacing the KL regularization in Eq.~\ref{eq:RDPO_RLHF} with other possible statistical divergences.
An important case, is considering regularization given by the sum of KL and $\chi^2$ divergence which gives the following RLHF-like problem
\begin{equation}
\argmax_{\pi\in\Pi} \innerprod{\pi}{\frac{\hat{r}_{BT}(x,a)}{\beta} + \frac{\hat{r}(x,a)}{\beta_1} } - D_{\mathrm{KL}}(\pi,\piref) -  D_{\chi^2}(\pi,\piref) \label{eq:RLHF_with_chi2}
\end{equation}
where $\hat{r}_{BT} \in \argmax_{r \in \mathcal{R}_\Pi} \sum^N_{i=1} \log \sigmoid(r(x_i,a_i^+) - r(x_i,a_i^-))$ is the Bradley-Terry reward estimator and the divergences (averaged w.r.t. the prompt distribution $\initial$), i.e.  $D_{\mathrm{KL}}(\pi,\piref)$, $D_{\chi^2}(\pi,\piref)$, are defined as follows $$D_{\mathrm{KL}}(\pi,\pi') = \sum_{x\in\X} \initial(x)\sum_{a\in \A} \pi(a|x) \log \frac{\pi(a|x)}{\pi'(a|x)}$$ and of the $\chi^2$-squared divergence defined as $$D_{\chi^2}(\pi,\pi') = \frac{1}{2}\br{\sum_{x\in\X} \initial(x)\sum_{a\in \A} \pi(a|x) \frac{\pi(a|x)}{\pi'(a|x)} - 1}.$$ It is easy to see that since $2x - 1\geq \log(x)$ holds for any $x > 0$, it also holds that for a fixed policy pair $\pi,\pi'$, we have that $D_{\chi^2}(\pi,\pi) \geq D_{\mathrm{KL}}(\pi,\pi')$. However, the KL term is still important to ensure that Eq.~\ref{eq:RLHF_with_chi2} admits the following closed form solution
\[
\piout(a|x) \propto  \piref(a|x)\phi_{\mathrm{KL}+\chi^2}^{-1}\br{ \frac{\hat{r}_{BT}(x,a)}{\beta} + \frac{\hat{r}(x,a)}{\beta_1} -\zeta(x)}.
\]
where we used $\phi_{\mathrm{KL}+\chi^2}(x) : = \log(x) + x$. The next Lemma shows that, when the reward class is parametrized by the policy parameters, the computation of $\piout$ can be carried out without an explicit estimation of $\hat{r}_{BT}$.
\begin{algorithm}[t]
\caption{ RDPO with arbitrary regularizer\label{alg:ratingsDPO}}
\begin{algorithmic}[1]
\STATE \textbf{Input:} Datasets $\mathcal{D} = \bc{(x_i, a_i, a'_i, z_i, \Delta^i_{\hat{r}})}$, $\beta'$, $\alpha$, Policy class $\Pi'$, convex increasing function $\phi:\mathbb{R}\rightarrow\mathbb{R}$.
\STATE Set $\beta = \frac{\beta'}{1-\alpha}$, $\beta_1 = \frac{\beta'}{\alpha}$.
\STATE Let $\Delta_{\phi \pi/\piref}(x,a,b) :=  \phi \br{\frac{\pi(a|x)}{\piref(a|x)}} -  \phi \br{\frac{\pi(b|x)}{\piref(b|x)}}$
\STATE \textbf{Return:} $\pi_{\mathrm{out}} \in \argmin_{\pi \in \Pi'} \sum^{N}_{i=1} - \log  \sigma \bigg ( \beta \Delta_{\phi \pi/\piref}(x_i, a_i, a'_i)- \frac{\beta}{\beta_1} \Delta^i_{\hat{r}}\bigg) $.
\end{algorithmic}
\end{algorithm}
\begin{lemma}\textbf{Generalization of RDPO}.\label{lem:gen_rdpo}
Let us consider an increasing invertible function $\phi:\mathbb{R}^+ \rightarrow \mathbb{R}$ and, given an arbitrary policy class $\Pi$, let us define the reward class $\mathcal{R}_{\Pi}$   defined as 
\begin{equation}
\mathcal{R}_\Pi = \bc{ r: ~~\text{s.t.}~~ \exists \pi \in \Pi, \zeta:\X\rightarrow\mathbb{R} ~~~~ r(x,a) = \beta \phi\br{\frac{\pi(a|x)}{\piref(a|x)}} - \frac{\beta}{\beta_1}\hat{r}(x,a) + \beta \zeta(x) }, \label{eq:general_R_pi}  
\end{equation}
and the Bradley-Terry reward estimate $\hat{r}_{BT}$ computed as
\[
\hat{r}_{BT} = \argmax_{r \in \mathcal{R}_\Pi} \sum^N_{i=1} \log \sigma( r(x_i,a^+_i) - r(x_i, a^-_i)).
\]
Moreover, let us consider $\piout$ defined as
\[
\piout(a|x) \propto \piref(a|x)\phi^{-1}\br{ \frac{\hat{r}_{BT}(x,a)}{\beta} + \frac{\hat{r}(x,a)}{\beta_1} -\zeta(x)}.
\]
Then, $\piout$ is in the solution set of the following optimization problem, which corresponds to a generalization of $\mathcal{L}_{RDPO}$. In formulas, we have
 \begin{equation}
 \piout \in \argmin_{\pi\in \Pi}  \sum^N_{i=1} - \log \sigma \br{\beta\Delta^i_{\phi \pi / \piref} - \frac{\beta}{\beta_1} \Delta^i_{\hat{r}}},
 \end{equation}
 where $\Delta_{\phi \pi/\piref}^i = \Delta_{\phi \pi/\piref}(x_i,a_i,a'_i) =  \phi \br{\frac{\pi(a_i|x_i)}{\piref(a_i|x_i)}} -  \phi \br{\frac{\pi(a'_i|x_i)}{\piref(a'_i|x_i)}}$.
\end{lemma}
For $\phi(x) = \log(x)$, we recover the original DPO loss.
For $\phi(x) = \phi_{\mathrm{KL} + \chi^2}(x) := x + \log(x)$ we obtain an RDPO version with an additional $\chi^2$ regularization term, which has been shown useful in \cite{huang2024correcting} to recover better theoretical results. Moreover, this generalization allows for rating gaps informed versions of $f$-DPO \cite{wang2023beyond}.
\begin{proof}
Let us consider the log likelihood maximization of the ranking under the Bradley-Terry model ( see Eq.~\ref{eq:RM}) and let us choose as reward class $\mathcal{R}$ the following class 
\[
\mathcal{R}_\Pi = \bc{ r: ~~\text{s.t.}~~ \exists \pi \in \Pi, \zeta:\X\rightarrow\mathbb{R} ~~~~ r(x,a) = \beta \phi\br{\frac{\pi(a|x)}{\piref(a|x)}} - \frac{\beta}{\beta_1}\hat{r}(x,a) + \beta\zeta(x) }.\]
We can perform a change of variable to obtain the following equality,
\begin{align*}
& \min_{r \in \mathcal{R}_{\Pi}} \sum^N_{i=1} - \log \sigma \br{
r(x_i, a^+_i) - r(x_i, a^-_i)
} \\&= \min_{\pi \in \Pi} - \sum^N_{i=1} \log \sigma \bigg(
\beta \phi\br{\frac{\pi(a^+_i|x_i)}{\piref(a^+_i|x_i)}} - \beta \phi\br{\frac{\pi(a^-_i|x_i)}{\piref(a^-_i|x_i)}} - \frac{\beta}{\beta_1}\br{\hat{r}(x_i, a^+_i) - \hat{r}(x_i, a^-_i)} \\
& \quad \quad \quad \quad \quad \quad - \beta\zeta(x) + \beta\zeta(x)
\bigg) \\
&= \min_{\pi \in \Pi} \sum^N_{i=1} - \log \sigma \br{
\beta \phi\br{\frac{\pi(a^+_i|x_i)}{\piref(a^+_i|x_i)}} - \beta\phi\br{\frac{\pi(a^-_i|x_i)}{\piref(a^-_i|x_i)}} - \frac{\beta}{\beta_1}\br{\hat{r}(x_i, a^+_i) - \hat{r}(x_i, a^-_i)}
} \\
&= \min_{\pi \in \Pi} \sum^N_{i=1} - \log \sigma \br{
\beta \Delta^i_{\phi \pi /\piref} - \frac{\beta}{\beta_1} \Delta^i_{\hat{r}}
},
\end{align*}
which is the loss we report in \Cref{alg:ratingsDPO}.
\end{proof}

Moreover, we have the following generalization of ML-RDPO.

\paragraph{Generalization of ML-RDPO}
Using the reward class $\mathcal{R}_\Pi$ for a general $\phi$ given in \cref{eq:general_R_pi} as constraint for the following optimization problem 
$\tilde{r} = \argmax_{r\in\mathcal{R}_\Pi} \log P^{\mathrm{rat,rank}}_r( \mathcal{D})$ 
and using the conditional independence assumption and Gaussianity of the ratings gap as in the ML-RDPO derivation, as done in the proof of \Cref{lemma:ml-rdpo_derivation}, we arrive at the following minimization problem to compute the output policy,
\begin{align}
\piout \in \argmin_{\pi\in\Pi} \sum^{N}_{i=1} - \log \sigma (\Delta^i_{\phi \pi/\piref}) + (2\mathbb{V})^{-1}\br{\Delta^i_{\hat{r}} - \Delta^i_{\phi \pi/\piref}}^2. \label{eq:ML_RDPO_phi}
\end{align}
Again, the only difference is that $\Delta^i_{\phi \pi/\piref}$ replaces the log gaps $\Delta^i_{\log \pi/\piref}$.
\subsection{Extension to heterogeneous datasets and missing ratings}
\label{app:hetereogenous_data}
In this section, we explain how RDPO and ML-RDPO can be extended to the case where the rating and ranking datasets are only partly overlapping or even disjoint. As previously mentioned, we denote the ranking dataset as $\mathcal{D}_{\mathrm{rank}}=\bc{x_i, a^+_i, a^-_i}^{N_{\mathrm{rank}}}_{i=1}$ and the rating only dataset $\mathcal{D}_{\mathrm{rat}}=\bc{x'_i, \tilde{a}_i, \bar{a}_i, \hat{r}(x'_i, \tilde{a}_i) - \hat{r}(x'_i, \bar{a}_i) }^{N_{\mathrm{rat}}}_{i=1}$.
An important special case is represented by the setting where the ranking is available for all data pairs in the data set, but the rating is present only on a portion of it. Such a situation is quite likely to happen since rating gaps are more costly to obtain. 
\paragraph{RDPO for heterogeneous data}
In order to apply RDPO to this setting we need to maintain a rating estimator $\hat{r}_{\mathrm{rat}}$ computed via least square regression using the available observations $\bc{\hat{r}(x'_i, \tilde{a}_i) - \hat{r}(x'_i, \bar{a}_i) }^{N_{\mathrm{rat}}}_{i=1}$ as regression targets.
In particular, as a first step, we compute,
\[
\hat{r}_{\mathrm{rat}} = \argmin_{r \in \mathcal{R}} \sum^{N_{\mathrm{rat}}}_{i=1} (r(x'_i, \tilde{a}_i) - r(x'_i, \bar{a}_i) - \hat{r}(x'_i, \tilde{a}_i) + \hat{r}(x'_i, \bar{a}_i) )^2
\]
At this point, we can use RDPO to evaluate the reward model on the ranking dataset. This gives \Cref{alg:RDPO_heterogenous}.
\begin{algorithm}[t]
\caption{RDPO for heterogeneous data \label{alg:RDPO_heterogenous}}
\begin{algorithmic}[1]
\STATE \textbf{Inputs:} \begin{itemize} \item $\phi(x) = \gamma x + \log(x)$, $\beta$, $\alpha$. 
\item Policy class $\Pi'$ such that $\pi\in\Pi'$ implies $\abs{\beta\Delta_{\phi \pi/\piref} - \beta/\beta_1 \cdot \Delta_{\hat{r}_{\mathrm{rat}}}}\leq R_{\max}-R_{\min}$, \item Datasets:
\begin{itemize}
    \item $\mathcal{D}_{\mathrm{rank}}=\bc{x_i, a^+_i, a^-_i}^{N_{\mathrm{rank}}}_{i=1}$
    \item $\mathcal{D}_{\mathrm{rat}}=\bc{x'_i, \tilde{a}_i, \bar{a}_i, \hat{r}(x'_i, \tilde{a}_i) - \hat{r}(x'_i, \bar{a}_i) }^{N_{\mathrm{rat}}}_{i=1}$
\end{itemize}
\end{itemize}
\STATE Set $\beta' = \beta (1-\alpha)$ and $\beta_1 = \frac{\beta'}{\alpha}$.
\STATE Let us denote \begin{itemize}
    \item $\Delta_{\phi \pi / \piref}(x,a,b) = \beta \phi \br{\frac{\pi(a|x)}{\piref(a|x)}} - \beta \phi \br{\frac{\pi(b|x)}{\piref(b|x)}}$
    \item $\Delta^i_{\phi \pi / \piref} = \Delta_{\phi \pi / \piref}(x_i, a^+_i, a^-_i)$
\end{itemize}
\STATE Compute the reward estimator based on ratings
\[
\hat{r}_{\mathrm{rat}} = \argmin_{r \in \mathcal{R}} \sum^{N_{\mathrm{rat}}}_{i=1} (r(x'_i, \tilde{a}_i) - r(x'_i, \bar{a}_i) - \hat{r}(x'_i, \tilde{a}_i) + \hat{r}(x'_i, \bar{a}_i) )^2
\]
\STATE Set $\Delta^i_{\hat{r}_{\mathrm{rat}}} = \hat{r}_{\mathrm{rat}}(x_i,a^+_i) - \hat{r}_{\mathrm{rat}}(x_i,a^-_i)$ for all $i\in[N_{\mathrm{rank}}]$
\STATE \textbf{Return:} $\pi_{\mathrm{out}}$ such that $\piout \in \argmax_{\pi\in\Pi'} \sum^{N_{\mathrm{rank}}}_{i=1} \log \sigma \bigg ( \beta \Delta^i_{\phi \pi / \piref}- \frac{\beta}{\beta_1} \Delta^i_{\hat{r}_{\mathrm{rat}}}\bigg) $.
\end{algorithmic}
\end{algorithm}
\paragraph{ML-RDPO for heterogeneous data}
ML-RDPO is particularly convenient in the heterogeneous case because it avoids the least square regression problem for computing $\hat{r}_{\mathrm{rat}}$.
For this derivation, we assume that the datasets $\mathcal{D}_{\mathrm{rank}}$
and $\mathcal{D}_{\mathrm{rat}}$ are independent. Therefore $P^\mathrm{rat,rank}_r( \mathcal{D}_{\mathrm{rat}},  \mathcal{D}_{\mathrm{rank}}) = P^{\mathrm{rat}}_r( \mathcal{D}_{\mathrm{rat}})  P^{\mathrm{rank}}_r(\mathcal{D}_{\mathrm{rank}}) $. Under this condition, we can rewrite the joint maximum likelihood problem as follows for any reward class $\mathcal{R}$
\begin{align*}
\tilde{r} &= \argmax_{r\in\mathcal{R}} \log \br{P^{\mathrm{rank}}_r(\mathcal{D}_{\mathrm{rank}}) P^{\mathrm{rat}}_r( \mathcal{D}_{\mathrm{rat}})  }  \\
&= \argmax_{r\in\mathcal{R}} \log \br{P^{\mathrm{rank}}_r\br{\bc{x_i, a_i^{+}, a_i^{-}}^{N_{\mathrm{rank}}}_{i=1}} P^{\mathrm{rat}}_r\br{\bc{x'_i,\tilde{a}_i, \bar{a}_i, \Delta^i_{\hat{r}}}^{N_{\mathrm{rat}}}_{i=1}}  } \\
&= \argmax_{r\in\mathcal{R}}  \sum^{N_{\mathrm{rank}}}_{i=1} \log P^{\mathrm{rank}}_r\br{ z_i | x_i, a_i, a'_i} + \sum^{N_{\mathrm{rat}}}_{i=1} \log P^{\mathrm{rat}}_r\br{ \Delta^i_{\hat{r}} | x'_i,  \tilde{a}_i, \bar{a}_i} 
\end{align*}
Now, plugging the Bradley-Terry model with reward $r \in \mathcal{R}$ as probability law $P^{\mathrm{rank}}_r$ and the Gaussian density of variance $\mathbb{V}$ and mean $\Delta^i_{r}:= r(x_i', \tilde{a}_i) - r(x_i', \bar{a}_i) $ for  $P^{\mathrm{rat}}_r$, we obtain the following loss
\begin{equation*}
\argmin_{r\in\mathcal{R}} 
\sum^{N_{\mathrm{rank}}}_{i=1} - \log \sigma( r(x_i,a_i^+) - r(x_i,a_i^-))  + \sum^{N_{\mathrm{rat}}}_{i=1} \br{ \Delta^i_{\hat{r}} -r(x'_i,  \tilde{a}_i) + r(x'_i, \bar{a}_i)}^2. 
\end{equation*}
Then, choosing the reward class 
\[
\mathcal{R} = \bc{ r ~~|~~ \exists \pi \in \tilde{\Pi} ~~\text{s.t.}~~r(x,a) = \beta \phi( \pi(a|x) /\piref(a|x)) },
\]
we can perform the change of variable from $r$ to $\beta \phi( \pi /\piref)$
\begin{equation}
\argmin_{\pi\in \tilde{\Pi}} 
\sum^{N_{\mathrm{rank}}}_{i=1} - \log \sigma( \Delta_{\phi \pi/\piref}(x_i,a_i^+,a_i^-))  + \sum^{N_{\mathrm{rat}}}_{i=1} \br{ \Delta^i_{\hat{r}} - \Delta_{\phi \pi/\piref}(x'_i,\tilde{a}_i,\bar{a}_i)}^2 \label{eq:ML_RDPO_het}
\end{equation}
which is the ML-RDPO loss but with the sums taken on different set of data. 
\section{Omitted Proofs}
\label{app:proofs}
This section contains the technical proofs of all the results included in the main text.
\subsection{Proof of \Cref{lemma:ml-rdpo_derivation}}
\begin{proof}
The key idea of the proof is to use the conditional independence and Gaussian assumption to rewrite the joint likelihood in  a simpler, computationally attractive fashion.
\begin{align*}
\tilde{r} &= \argmax_{r\in\mathcal{R}} \sum^N_{i=1}\log \mathbb{P}(x_i, a_i, a'_i, \Delta^i_{\hat{r}}, z_i) \\ &= \argmax_{r\in\mathcal{R}} \sum^N_{i=1}\log P^\mathrm{rat,rank}_r(\Delta^i_{\hat{r}}, z_i | x_i, a_i, a'_i) \mathbb{P}(x_i, a_i, a'_i) \\
 &= \argmax_{r\in\mathcal{R}} \sum^N_{i=1}\log P^\mathrm{rat,rank}_r(\Delta^i_{\hat{r}}, z_i | x_i, a_i, a'_i) ~~~~~ \text{(Using conditional independence)}\\
&= \argmax_{r\in\mathcal{R}} \sum^{N}_{i=1} \log P^{\mathrm{rank}}_r\br{ z_i | x_i, a_i, a'_i} + \sum^{N}_{i=1}\log P^{\mathrm{rat}}_r\br{ \Delta^i_{\hat{r}} | x_i, a_i, a'_i} 
\end{align*}
At this point, we model $P^{\mathrm{rat}}_r(\Delta^i_{\hat{r}} | x_i, a_i, a'_i)$ as a gaussian random variable of variance $\mathbb{V}$ and mean $ r(x_i, a_i) - r( x_i,a'_i)$, That is, we set $P^{\mathrm{rat}}_r(\Delta^i_{\hat{r}}|x,a_i,a'_i) = \frac{1}{\sqrt{2 \pi \mathbb{V}}}\exp\br{- \frac{(\Delta^i_{\hat{r}} - \Delta^i_r)^2}{2\mathbb{V}}}$.
Then, using the assumption that $P^{\mathrm{rank}}_r$ follows the Bradley-Terry model. All in all, we arrive at the following maximization problem
\[
\tilde{r} = \argmax_{r\in\mathcal{R}} \sum^{N}_{i=1} \log \sigma ( r(x_i, a^+_i) - r(x_i, a^-_i)) - \frac{\sum^{N}_{i=1} \br{r(x_i, a_i) - r(x_i, a'_i) -\Delta^i_{\hat{r}}}^2}{2\mathbb{V}}.
\]
Moreover, noticing that no matter if $a_i = a^+_i$ or $a'_i = a^+_i$ the squared loss term does not change by symmetric around $0$, we obtain that
\[
\tilde{r} = \argmax_{r\in\mathcal{R}} \sum^{N}_{i=1} \log \sigma ( r(x_i, a^+_i) - r(x_i, a^-_i)) - \frac{\sum^{N}_{i=1} \br{r(x_i, a^+_i) - r(x_i, a^-_i) -\Delta^i_{\hat{r}}}^2}{2\mathbb{V}}.
\]
Replacing the definition $\Delta^i_r = r(x_i, a^+_i) - r(x_i, a^-_i)$  concludes the proof.
\end{proof}
\subsection{Proof of \Cref{thm:ratings_dpo} (RDPO guarantees)}
\label{app:proof_rdpo_main}
Here, we prove our main result concerning RDPO.
\RDPO*
\begin{proof}
    Let us introduce the notation $\alpha = \br{1 + \mathrm{Err}(\hat{r}) / \mathrm{Err_{DPO}}(N,\delta)}^{-1}$.
    Notice that under this notation, we have that the choice for $\beta_1$ can be rewritten as 
    $\beta_1 = \frac{1-\alpha}{\alpha}\beta$.
    Moreover, let us define $\beta' =(1-\alpha) \beta$ and the implicit reward
    \[
r_{\mathrm{out}}(x,a) = \frac{\beta'}{1-\alpha}\log\br{\frac{\piout(a|x)}{\piref(a|x)}} - \frac{\alpha}{1-\alpha} \hat{r}(x,a).
\]
Then, by \Cref{lemma:dpo_rlhf_duality}, we have that $\piout$ is the solution of the following RLHF like problem
\begin{align}
\pi_{\mathrm{out}} = \argmax_{\pi\in\Pi}\innerprod{\pi}{(1-\alpha)r_{\mathrm{out}} + \alpha \hat{r}} -\beta' D_{\mathrm{KL}}(\pi,\piref)
\label{eq:optimality}
\end{align}
Then, let us define $\bar{r} = (1-\alpha) r_{\mathrm{out}} + \alpha \hat{r}$ . Then, by the fact that $\piout$ is a maximizer of the problem in~\eqref{eq:optimality}, we obtain that 
\[
\innerprod{\pi^\star}{\bar{r}} - \beta' D_{\mathrm{KL}}(\pi^\star,\pi_{\mathrm{ref}}) 
\leq \innerprod{\pi_{\mathrm{out}}}{\bar{r}} - \beta' D_{\mathrm{KL}}(\pi_{\mathrm{out}},\pi_{\mathrm{ref}})
\]
Therefore, we notice that
\begin{align*}
    \innerprod{\pi^\star - \piout}{r^\star} &= \innerprod{\pi^\star - \piout}{\bar{r}} + \innerprod{\pi^\star - \piout}{r^\star-\bar{r} } \\
    &\leq \beta' D_{\mathrm{KL}}(\pi^\star,\pi_{\mathrm{ref}}) - \beta' D_{\mathrm{KL}}(\pi_{\mathrm{out}},\pi_{\mathrm{ref}}) + \innerprod{\pi^\star - \piout}{r^\star-\bar{r} }.
\end{align*}
Rearranging implies that 
\begin{align*}
J_{\beta'}(\pi^\star; r^\star) &- J_{\beta'}(\piout;r^\star) \leq \innerprod{\pi^\star - \piout}{r^\star-\bar{r} } \\
&= \sum_{x\in\X}\initial(x)\sum_{a\in \A} \sum_{b\in \A} (\pi^\star(a|x) - \piout(a|x)) \pidata(b|x) \br{r^\star(x,a)-\bar{r}(x,a) - r^\star(x,b) + \bar{r}(x,b) }
\end{align*}
Then, notice that for any policy $\pi$ we can perform the following change of measure where $f(x,a)$ stands for $\sum_{b\in\A}\pidata(b|x)\br{r^\star(x,a)-\bar{r}(x,a) - r^\star(x,b) + \bar{r}(x,b) }$
\begin{align*}
\sum_{x\in \X} \initial(x) \sum_{a\in \A} \pi(a|x) f(x,a) &=
\sum_{x\in \X} \initial(x) \sum_{a\in \A} \frac{\pi(a|x)}{\sqrt{\pi_{\mathrm{data}}(a|x)} } \sqrt{\pi_{\mathrm{data}}(a|x)}  f(x,a) \\
&\leq \sqrt{\sum_{x\in \X} \initial(x) \sum_{a\in \A} \frac{\pi^2(a|x)}{\pi_{\mathrm{data}}(a|x)}} \sqrt{\sum_{x\in \X} \initial(x) \sum_{a\in \A} \pi_{\mathrm{data}}(a|x)  f^2(x,a) } \\
&\leq \sqrt{C^\pi \mathrm{Err}(\bar{r})} 
\\
&\leq \sqrt{C^{\max} \mathrm{Err}(\bar{r})} ,
\end{align*}
where we recall that
\[
\mathrm{Err}(\bar{r}) := \mathbb{E}_{\!\!\!\!\!\substack{x \sim \initial\\a,a'\sim\pidata(\cdot|X)}}\big[\big(\Delta_{r^\star}(x,a,a') - \Delta_{\bar{r}}(x,a,a') \big)^2\big].
\]
Therefore, we can upper bound $\innerprod{\pi^\star - \piout}{r^\star-\bar{r}} \leq 2\sqrt{C^{\max} \mathrm{Err}(\bar{r})}$ and therefore,
\[
J_{\beta'}(\pi^\star; r^\star) - J_{\beta'}(\piout; r^\star) \leq 2\sqrt{C^{\max} \mathrm{Err}(\bar{r})}
\]
At this point, invoking the auxiliary error decomposition proven in \Cref{lemma:error_decomp}, we have that
    \[
\mathrm{Err}(\bar{r}) \leq \frac{2\mathrm{Err_{DPO}} (N,\delta) \mathrm{Err} (\hat{r})}{\mathrm{Err_{DPO}} (N,\delta) + \mathrm{Err} (\hat{r})}.
    \]
    The proof is then concluded noticing that
    \[
    \beta' = (1-\alpha) \beta = \frac{\beta_1}{\beta + \beta_1} \beta 
    \]
    where we use that the definition of $\alpha$ implies that $1-\alpha =  \frac{\beta_1}{\beta + \beta_1} $. All in all this implies that for any policy $\pi$, it holds that $J_{\beta'}(\pi; r^\star) = J_{\frac{\beta\beta_1}{\beta + \beta_1}}(\pi; r^\star)$.
\end{proof}
The next step towards the proof of \Cref{thm:ratings_dpo} is to decompose the generalization error of the mixed reward $\bar{r}$ into a contribution depending only on the ratings gap-based estimator and a second contribution depending only on the estimator learned via ranking.
\begin{lemma}{\textbf{Error Decomposition}} \label{lemma:error_decomp}
Let the policy class $\Pi'$ and the errors $\mathrm{Err_{DPO}}(N,\delta)$ and $\mathrm{Err}(\hat{r})$ be defined as in \Cref{thm:ratings_dpo}, then choosing $\alpha = \br{1 + \mathrm{Err}(\hat{r}) / \mathrm{Err_{DPO}}(N,\delta)}^{-1}$ it holds that
\[
\mathrm{Err}(\bar{r}) \leq 2\br{\mathrm{Err_{DPO}}^{-1}(N,\delta) + \mathrm{Err}^{-1}(\hat{r})  }^{-1} \leq 2 \min \bc{\mathrm{Err_{DPO}}(N,\delta), \mathrm{Err}(\hat{r})  }
\]

\end{lemma}
\begin{proof}
    Let us start replacing the definition of $\bar{r}$ in the definition of $\mathrm{Err}(\bar{r})$. We obtain,
    \begin{align*}
\mathrm{Err}(\bar{r}) &= 
\sum_{x\in \X} \initial(x) \sum_{a\in \A} \pidata(a|x) \pidata(b|x)\br{\Delta_{r^\star}(x,a,b) - \beta' \Delta_{\phi \piout/\piref}(x,a,b)}^2\\
&= 
\sum_{x\in \X} \initial(x) \sum_{a\in \A} \pidata(a|x) \pidata(b|x)\br{\Delta_{r^\star}(x,a,b) -  \Delta_{(1-\alpha)r_{\mathrm{out}} + \alpha \hat{r}}(x,a,b)}^2,\\
    \end{align*}
    where $\phi(x)$ is defined as in \Cref{alg:ratingsDPO}, i.e. $\phi(x) = \log(x)$ when no $\chi^2$ regularization is used and $\phi(x) = \log(x) + x$ otherwise.
At this point, we use the following decompositions 
\begin{align*}
\Delta_{(1-\alpha)r_{\mathrm{out}} + \alpha \hat{r}}(x,a,b) &= (1-\alpha)(r_{\mathrm{out}}(x,a) - r_{\mathrm{out}}(x,b)) + \alpha
(\hat{r}(x,a) - \hat{r}(x,b)) \\
&= (1-\alpha)\Delta_{r_{\mathrm{out}}}(x,a,b) + \alpha \Delta_{\hat{r}}(x,a,b)
\end{align*}
Equipped of this fact and rewriting $\Delta_{r^\star}(x,a,b) = (1-\alpha) \Delta_{r^\star}(x,a,b) + \alpha \Delta_{r^\star}(x,a,b)$, we have,
\begin{align*}
\mathrm{Err}(\bar{r}) &= \sum_{x\in \X} \initial(x) \sum_{a\in \A} \pidata(a|x) \pidata(b|x)\bigg((1-\alpha)(\Delta_{r^\star}(x,a,b) - \Delta_{r_{\mathrm{out}}}(x,a,b)) \\&\phantom{=}+ \alpha (\Delta_{r^\star}(x,a,b) - \Delta_{\hat{r}}(x,a,b))\bigg)^2\\
&\leq 2 (1-\alpha)^2 \sum_{x\in \X} \initial(x) \sum_{a\in \A} \pidata(a|x) \pidata(b|x) (\Delta_{r^\star}(x,a,b) - \Delta_{r_{\mathrm{out}}}(x,a,b))^2 \\
&\phantom{=}+2 \alpha^2 \sum_{x\in \X} \initial(x) \sum_{a\in \A} \pidata(a|x) \pidata(b|x) (\Delta_{r^\star} (x,a,b) - \Delta_{\hat{r}}(x,a,b))^2
\end{align*}
where the inequality follows from Young's inequality.
The second term is by definition equal to $\mathrm{Err}(\hat{r})$. 
For the first term, we notice that we can rewrite the optimization problem in \eqref{eq:RM} as a log likelihood maximization problem over reward variables 
\[
r_{\mathrm{out}} = \argmax_{r\in\mathcal{R}_{\Pi'}} \sum^{N}_{i=1}\log \sigma \br{r(x_i, a_i^{+}) - r(x_i, a_i^{-})}
\]
for the reward class $\mathcal{R}_{\Pi'} = \bc{r | \exists \pi\in\Pi' ~~~\text{such that}~~~ r(x,a) = \frac{1}{1-\alpha}\br{\beta' \phi\br{\frac{\pi(a|x)}{\pi_{\mathrm{ref}}(a|x)}} - \alpha \hat{r}(x,a)}}$. 
Here, the reward class depends on $\hat{r}$, which could be dependent on the dataset $\mathcal{D}$.
This invalidates the classic MLE results (see e.g. \citet[Proposition 1]{foster2023foundations}) if the dataset of rankings is correlated with the dataset of ratings. Notice that this is indeed the case if ratings and rankings come from the same source as for example in \texttt{ultrafeedback-binarized} dataset.
However, we know that $\hat{r} \in \mathcal{R}$ by definition. Therefore, let us handle dependent data with a covering number. In particular, let us fix $\tilde{r} \in \mathcal{R}$ such that the policy induced reward class $\mathcal{R}_{\tilde{r}}$
defined as
\begin{equation}
\mathcal{R}_{\tilde{r}} = \bc{r | \exists \pi\in\Pi'_{\tilde{r}} ~~~\text{such that}~~~ r(x,a) = \frac{1}{1-\alpha}\br{\beta' \phi\br{\frac{\pi(a|x)}{\pi_{\mathrm{ref}}(a|x)}} - \alpha \tilde{r}(x,a)}}. \label{eq:new_reward_class}
\end{equation}
with $\Pi'_{\tilde{r}}$ being the analog of $\Pi'$ where the fixed reward $\tilde{r}$ replaces $\hat{r}$. That is,

\begin{align*} \Pi'_{\tilde{r}} = \bigg\{\pi \in \Pi : \text{s.t.} ~~\forall x,a,b   ~~ \abs{\beta' \Delta_{\phi \pi/\piref}(x,a,b) -\alpha \Delta_{\tilde{r}}(x,a,b)} \leq (1-\alpha)  ( R_{\max} - R_{\min}) \bigg\},
\end{align*} 
where $\Delta_{\tilde{r}}(x,a,b) = \tilde{r}(x,a) - \tilde{r}(x,b)$.
Clearly, $\mathcal{R}_{\tilde{r}}$ is independent of $\mathcal{D}$ since $\tilde{r}$ is independent of $\mathcal{D}$ \footnote{in fact, it is a fixed quantity, not a random variable.}.
Next, we show that this reward class realizes the true reward $r^\star$ up to a state-dependent shift $\zeta: \X \rightarrow \mathbb{R}$. To this end let us define the policy $\pi_{\tilde{r}}$ as follows
\[
  \pi_{\tilde{r}}(a|x) = \pi_{\mathrm{ref}}(a|x) \phi^{-1}(\beta'^{-1}((1-\alpha)r^\star(x,a) + \alpha \tilde{r}(x,a) + \zeta(x)))
\]
where $\zeta(x)$ is chosen to ensure that $\sum_{a\in\A} \pi_{\tilde{r}}(a|x) = 1$ for all $x\in \X$. 
Now we show that $\pi_{\tilde{r}} \in \Pi'_{\tilde{r}}$.
Rearranging, we have that,
\[
\beta' \phi\br{\frac{\pi_{\tilde{r}}(a|x)}{\pi_{\mathrm{ref}}(a|x)}} = (1-\alpha)r^\star(x,a) + \alpha \tilde{r}(x,a) + \zeta(x)
\]
Therefore, taking the difference between the log policy ratios evaluated at two different actions $a,b$ for a fixed prompt $x$, we have that
\[
\beta' \phi\br{\frac{\pi_{\tilde{r}}(a|x)}{\pi_{\mathrm{ref}}(a|x)}} - \beta' \phi\br{\frac{\pi_{\tilde{r}}(b|x)}{\pi_{\mathrm{ref}}(b|x)}} = (1-\alpha)(r^\star(x,a)-r^\star(x,b)) + \alpha (\tilde{r}(x,a) -\tilde{r}(x,b)).
\]
It follows that 
\begin{align*}
(1-\alpha)(R_{\min} - R_{\max}) + \alpha (\tilde{r}(x,a) - \tilde{r}(x,b) ) &\leq \beta' \phi\br{\frac{\pi_{\tilde{r}}(a|x)}{\pi_{\mathrm{ref}}(a|x)}} - \beta' \phi\br{\frac{\pi_{\tilde{r}}(b|x)}{\pi_{\mathrm{ref}}(b|x)}} \\&\leq (1-\alpha)(R_{\max} - R_{\min}) + \alpha(\tilde{r}(x,a) - \tilde{r}(x,b)).
\end{align*}
Let us consider a reward class $\mathcal{R}$ such that there exists a reward function $\bar{r}^\star \in \mathcal{R}$ such thar $\Delta_{\bar{r}^\star} = \Delta_{r^\star}$. Moreover, we consider that the class $\mathcal{R}$ is coincise in the sense that there are no two reward functions $r_1, r_2 \in \mathcal{R}$ such that $\Delta_{r_1} = \Delta_{r_2}$. This ensures that the cardinality of the reward class is upper bounded by the cardinality of the policy $\Pi$ as long as it satisfies \Cref{ass:policy_realizability_main}. 
Moreover, notice that for any $\tilde{r}\in\mathcal{R}$, we have that $\pi_{\tilde{r}} \in \Pi'_{\tilde{r}}$. At this point, notice that
for $\pi(a|x) = \pi_{\tilde{r}}(a|x)$ in \eqref{eq:new_reward_class}, 
we have that the true reward plus a state (prompt) dependent shift belongs to the induced reward class. In formulas, we have that there exists $\zeta: \X \rightarrow \mathbb{R}$ such that $ r_{\zeta} \in \mathcal{R}_{\tilde{r}} $ where we defined $r_{\zeta}(x,a) = r^\star(x,a) + \zeta(x)/(1-\alpha)$ for all $x,a$.
We can conclude that the true reward function is realizable (up to a state dependent shift) for all possible values of $\tilde{r}$. 
Introducing the probability model, $P_r(a ~\text{is preferred to}~b |x,a,b) = \sigma (r(x,a) - r(x,b))$, we can consider that
\[
r_{\mathrm{out}} = \argmax_{r \in \mathcal{R}_{\tilde{r}}} \sum^{N_{\mathrm{rank}}}_{i=1} \log P_r(a_i^+ ~\text{is preferred to} ~a_i^{-} |x_i, a_i^-, a_i^+).
\]
Notice that since $r_{\zeta} \in \mathcal{R}_{\tilde{r}}$ we have that
\begin{align*}
P_{r_{\zeta}}(a ~\text{is preferred to}~b |x,a,b) &= \sigma \br{r^\star(x,a) + \zeta(x)/(1-\alpha)- r^\star(x,b) - \zeta(x)/(1-\alpha)} 
\\&= \sigma \br{r^\star(x,a)- r^\star(x,b)}
\\&= P_{r^{\star}}(a ~\text{is preferred to}~b |x,a,b)
\end{align*}
Therefore, the true preference generator $P_{r^\star}$ is realized by the preference models class induced by the reward class $\mathcal{R}_{\tilde{r}}$ even if we do not necessarily have that $r^\star \in \mathcal{R}_{\tilde{r}}$. This fact should not surprise the reader since the Bradley-Terry model is invariant to shifts depending on the state only.

Moreover, the definition of $\Pi'$ guarantees that $$
\pi \in \Pi' \implies  \abs{ \beta' \Delta_{\phi \pi / \piref}(x,a,b) - \alpha \Delta_{\hat{r}}(x,a,b)} \leq (1-\alpha)(R_{\max} - R_{\min}) ~~~~\forall x,a,b.$$
Therefore, since $\piout\in\Pi'$,
the above fact implies that $\Delta_{r_{\mathrm{out}}}\in[-R_{\max},R_{\max}]$.
At this point, thanks to \Cref{lemma:hellinger_upper_bound}, a covering number over the class $\mathcal{R}$ and with \Cref{lemma:conditionalMLE}  invoked with failure probability $\delta/\abs{\mathcal{R}}$, we obtain that with probability at least $1-\delta$ it holds that 
\begin{align*}
(1-\alpha)^2 &\sum_{x\in \X} \initial(x) \sum_{a\in \A} \pidata(a|x) \pidata(b|x)\br{\Delta_{r^\star}(x,a,b) - \Delta_{r_{\mathrm{out}}}(x,a,b)}^2
\\&=(1-\alpha)^2 \sum_{x\in \X} \initial(x) \sum_{a\in \A} \pidata(a|x) \pidata(b|x)\br{r^\star(x,a) - r^\star(x,b) - r_{\mathrm{out}}(a|x)+ r_{\mathrm{out}}(x,b)}^2 \\ & \leq 
32 R^2_{\max} (1-\alpha)^2 e^{4 R_{\max}} \frac{\log (\abs{\Pi'}\abs{\mathcal{R}}/\delta) }{N_{\mathrm{rank}}}\\
&\leq 32 R^2_{\max} (1-\alpha)^2 e^{4 R_{\max}} \frac{\log (\abs{\Pi}^2/\delta) }{N_{\mathrm{rank}}} \\
&= \mathcal{O}\br{R^2_{\max} (1-\alpha)^2 e^{4 R_{\max}} \frac{\log (\abs{\Pi}/\delta) }{N_{\mathrm{rank}}}} 
= (1-\alpha)^2\mathrm{Err_{DPO}}(N,\delta),
\end{align*}
where the last inequality uses that $\Pi' \subset \Pi$ and that by the assumption about the expressivity of the policy class (i.e. \Cref{ass:policy_realizability_main}) and by the fact that we considered the class $\mathcal{R}$ to be coincise, we have that $\abs{\mathcal{R}} = \abs{\Pi}$. Furthermore, the last equality holds by definition of $\mathrm{Err_{DPO}}(N,\delta)$ and by the fact that $\Pi' \subset \Pi$.
Putting all together, we have that
\begin{align*}
\mathrm{Err}(\bar{r}) &\leq 2 \bigg( (1-\alpha)^2 \mathrm{Err_{DPO}} (N,\delta) + \alpha^2 \mathrm{Err} (\hat{r})\bigg) \\
&= \frac{2\mathrm{Err_{DPO}} (N,\delta) \mathrm{Err} (\hat{r})}{\mathrm{Err_{DPO}} (N,\delta) + \mathrm{Err} (\hat{r})}
\end{align*}
where the first equality follows from the choice of $\alpha$. 
Finally, using the relation $(1/\mathfrak{a} + 1/\mathfrak{b})^{-1} \leq \min(\mathfrak{a},\mathfrak{b})$ that holds between any two positive scalars $\mathfrak{a},\mathfrak{b} \in \mathbb{R}^+$, we that 
\begin{equation}
\br{ \mathrm{Err_{DPO}}^{-1}(N,\delta)   + {\mathrm{Err}^{-1}(\hat{r})}}^{-1} \leq \min\bc{ \mathrm{Err_{DPO}}(N,\delta)  ,\mathrm{Err}(\hat{r})} 
\end{equation}
and therefore
\begin{align*}
\mathrm{Err}(\bar{r}) \leq 2\min\bc{ \mathrm{Err_{DPO}}(N,\delta)  ,\mathrm{Err}(\hat{r})}.
\end{align*}
\end{proof}

\subsection{RDPO with $\chi^2$ regularization: single policy concentrability guarantees}
\label{app:Cstar}
Next, we prove guarantees for the more general form in \Cref{alg:ratingsDPO} based on single policy concentrability.
We will make use of the following quantities 
\[
C^\pi = \sum_{x\in\X} \initial(x)\sum_{a\in \A} \pi(a|x) \frac{\pi(a|x)}{\piref(a|x)}  ~~~~~ C^\star = \sum_{x\in\X} \initial(x)\sum_{a\in \A} \pi^\star(a|x) \frac{\pi^\star(a|x)}{\piref(a|x)},
\]
and of the following relation:
\[
\frac{1}{2}(C^\pi - 1) = D_{\chi^2}(\pi,\piref) \geq D_{\mathrm{KL}}(\pi,\piref)
\]
for all policies $\pi\in\Pi_{\mathrm{all}}$. Therefore, we also have that \[
\frac{1}{2}(C^\star - 1) = D_{\chi^2}(\pi^\star,\piref) \geq D_{\mathrm{KL}}(\pi^\star,\piref).
\]
 Moreover, a crucial quantity for the analysis is the on policy estimation error for the implicit reward estimation error $\bar{r}(x,a) = \beta'\phi\br{\frac{\piout(a|x)}{\pidata(a|x)}}$ defined as
 \[
\mathrm{Err}_{\pi_{\mathrm{ref}}}(\bar{r}) =
\sum_{x\in \X} \initial(x) \sum_{a\in \A} \pi_{\mathrm{ref}}(a|x) \pi_{\mathrm{ref}}(b|x)\br{\Delta_{r^\star}(x,a,b) - \beta' \Delta_{\phi \piout/\piref}(x,a,b)}^2,
\]
where we recall that  $\Delta_{\phi \pi / \piref}(x,a,a') :=  \phi\br{\frac{\pi(a|x)}{\piref(a|x)}} - \phi\br{ \frac{\pi(a'|x)}{\piref(a'|x)}}$.
Moreover, we highlighted the fact that for this result we need $\pidata=\piref$ by explicitely adding $\piref$ in the subscript of the reward estimation error defined above. 
\begin{theorem} 
\label{thm:ratings_dpo_app}
 Let us consider the setting where $\pidata = \piref$, i.e. the dataset of responses is collected with the initial model.
  Moreover, for any $\beta \in (0, \infty)$ let $\piout$ be the output of \Cref{alg:ratingsDPO}, ran with $\phi(x) =  \gamma x + \log(x)$ for $\gamma = \sqrt{\frac{\mathrm{Err}_{\max}}{\beta^2 (1-\alpha)^2 C^\star}} $ where
\[
\mathrm{Err}_{\max} = 2\min\bc{ \mathrm{Err_{DPO}}(N,\delta)  ,\mathrm{Err}_{\piref}(\hat{r})}
\]
and  $$ \alpha = \br{1 + \mathrm{Err}_{\piref}(\hat{r}) / \mathrm{Err_{DPO}}(N,\delta)}^{-1}$$
Then, it holds that
    \[
J_{\frac{\beta\beta_1}{\beta+\beta_1}}(\pi^\star;r^\star) - J_{\frac{\beta\beta_1}{\beta+\beta_1}}(\piout;r^\star)  \leq  \mathcal{O}\br{\sqrt{C^\star \min\br{ \mathrm{Err_{DPO}}(N,\delta),\mathrm{Err}_{\piref}(\hat{r})}}}, \quad \text{w.p.}\quad  1-\delta.
\] 
\end{theorem}

In order to prove the statements, we consider the function $\phi(x) = \gamma x + \log(x)$.
 The proof of \Cref{thm:ratings_dpo_app} is divided in the next two lemmas: \Cref{lemma:first_bound} converts the suboptimality of the extracted policy $\pi_{\mathrm{out}}$ into the excess risk of the reward estimator corresponding to $\piout$, i.e. $r_{\mathrm{out}}$

 Then, \Cref{lemma:error_decomp} ( invoked recalling that $\pidata=\piref$ in this setting)  bounds $\mathrm{Err}_{\pi_{\mathrm{ref}}}(\bar{r})$ in high probability. Once these core facts are proven, the final result follows by plugging in the upper bound from \Cref{lemma:error_decomp} into \Cref{lemma:first_bound}.
\begin{lemma} 
\label{lemma:first_bound}
    Under the setting of \Cref{thm:ratings_dpo_app}, it holds that
    \[
J_{\frac{\beta\beta_1}{\beta+\beta_1}}(\pi^\star;r^\star) - J_{\frac{\beta\beta_1}{\beta+\beta_1}}(\piout;r^\star) 
 \leq  \mathcal{O}\br{\sqrt{\mathrm{Err}_{\max}C^\star }}.
\]

\end{lemma}
When comparing \Cref{thm:ratings_dpo,thm:ratings_dpo_app}, an important remark is that the KL regularization is strictly suboptimal compared to $\chi^2$ regularization in the worst case, since we have that 
$C^{\max} > C^\star$ unless we are in the spurious case of $\Pi' = \bc{\pi_{\mathrm{ref}}}$. However, when the policy class $\Pi'$ does not allow large deviations from $\pi_{\mathrm{ref}}$ the difference between $D_{\chi^2}$ and $D_{\mathrm{KL}}$ is well controlled as shown in the \Cref{fig:kl_vs_chi2}.
Therefore, even with $\gamma=0$ we expect good results in practice if the algorithm output remains close to $\pi_{\mathrm{ref}}$.
\begin{figure}
    \centering
    \includegraphics[width=0.5\linewidth]{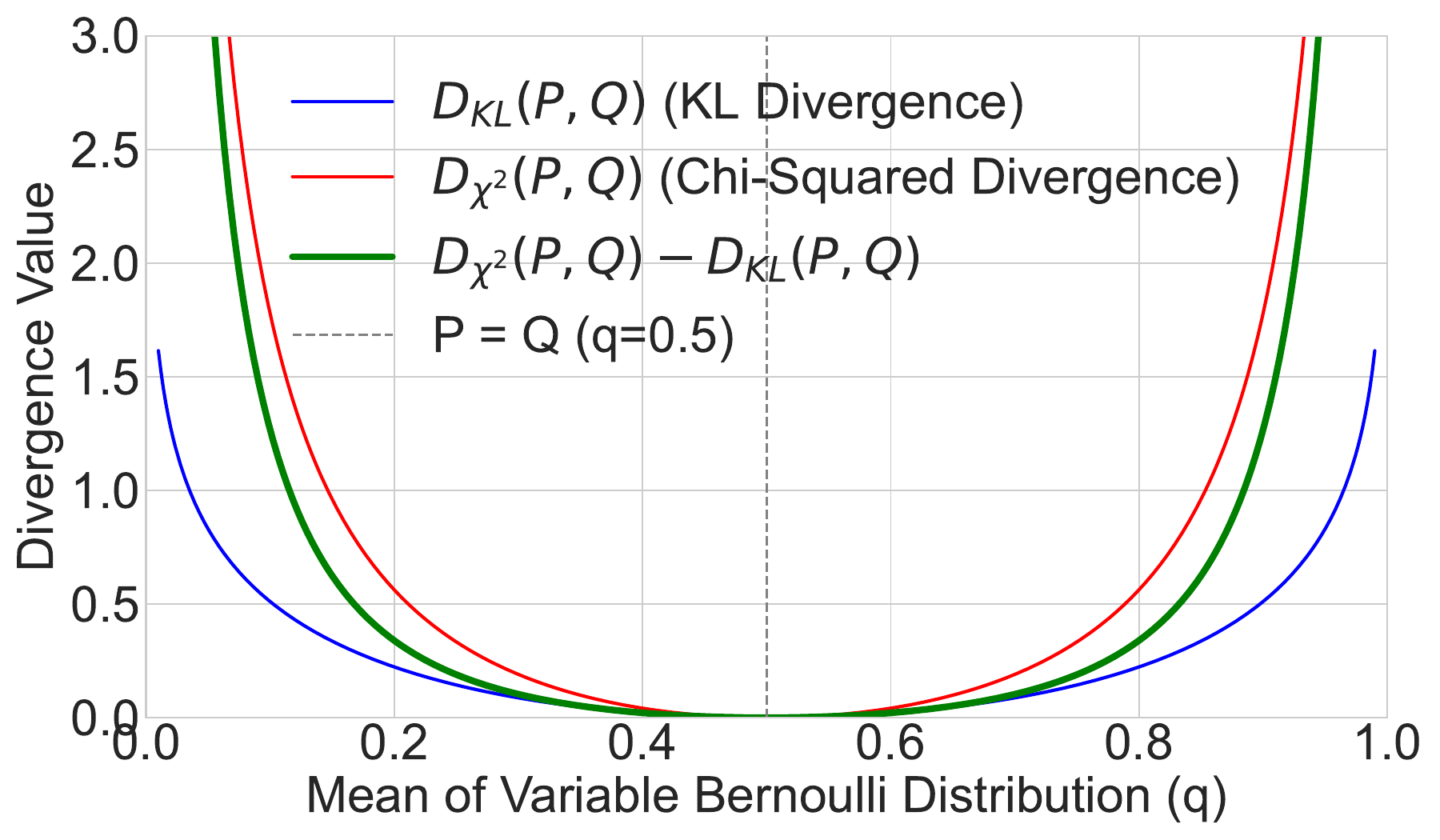}
    \caption{Comparison between $D_{\mathrm{KL}}$ and $D_{\chi^2}$ and their differences for Bernoulli random variables. }
    \label{fig:kl_vs_chi2}
\end{figure}
Next, we state the proof of \Cref{lemma:first_bound}.
\begin{proof}\emph{of \Cref{lemma:first_bound} }
Let us consider the notation
\[
r_{\mathrm{out}}(x,a) = \frac{\beta'}{1-\alpha}\phi\br{\frac{\piout(a|x)}{\piref(a|x)}} - \frac{\alpha}{1-\alpha} \hat{r}(x,a).
\]
Rearranging the above definition, we obtain the following elementwise equality
\[
(1-\alpha)r_{\mathrm{out}} + \alpha \hat{r} = \beta' \phi\br{\frac{\piout}{\piref}}
\]
By \Cref{lemma:dpo_rlhf_duality}, we have that $\piout$ is the solution of the following RLHF like problem
\begin{align}
\pi_{\mathrm{out}} = \argmax_{\pi\in\Pi}\innerprod{\pi}{(1-\alpha)r_{\mathrm{out}} + \alpha \hat{r}} - \beta' \gamma D_{\chi^2}(\pi,\piref) -\beta' D_{\mathrm{KL}}(\pi,\piref)
\end{align}
At this point, let $\bar{r} = (1-\alpha) r_{\mathrm{out}} + \alpha \hat{r}$ and $D_F(\pi,\piref) := \gamma D_{\chi^2}(\pi,\piref) + D_{\mathrm{KL}}(\pi,\piref)$. It holds that
\[
\innerprod{\pi^\star}{\bar{r}} - \beta' D_F(\pi^\star,\pi_{\mathrm{ref}}) 
\leq \innerprod{\pi_{\mathrm{out}}}{\bar{r}} - \beta' D_F(\pi_{\mathrm{out}},\pi_{\mathrm{ref}})
\]
We can now use the above fact to bound the suboptimality against the optimal policy $\pi^\star$
\begin{align}
\innerprod{\pi^\star}{r^\star} & - \innerprod{\pi_{\mathrm{out}}}{r^\star} = \nonumber\\
&\leq \innerprod{\pi^\star}{r^\star} - \innerprod{\pi^\star}{\bar{r}} + \beta' D_F(\pi^\star,\pi_{\mathrm{ref}}) 
+ \innerprod{\pi_{\mathrm{out}}}{\bar{r}} - \beta' D_F(\pi_{\mathrm{out}},\pi_{\mathrm{ref}})
- \innerprod{\pi_{\mathrm{out}}}{r^\star} 
\nonumber\\&= \innerprod{\pi^\star}{r^\star} - \innerprod{\pi^\star}{\bar{r}} + \frac{\beta' \gamma}{2} C^\star + \beta' D_{\mathrm{KL}}(\pi^\star,\piref) \nonumber\\&\phantom{=}
+ \innerprod{\pi_{\mathrm{out}}}{\bar{r}} - \frac{\beta' \gamma}{2}C^{\piout} - \beta' D_{\mathrm{KL}}(\piout,\piref)
- \innerprod{\pi_{\mathrm{out}}}{r^\star}. 
\label{eq:decomposition}
\end{align}
In the final equality, we used the following equivalencies
\begin{align*}
D_F(\pi^\star,\piref) - D_F(\piout,\piref) &=  \gamma D_{\chi^2}(\pi^\star,\piref) + 
D_{\mathrm{KL}}(\pi^\star,\piref) - \gamma D_{\chi^2}(\piout,\piref)\\
&\phantom{=} -
D_{\mathrm{KL}}(\piout,\piref) \\
&= \gamma \br{\frac{C^\star}{2} - 1} + 
D_{\mathrm{KL}}(\pi^\star,\piref) - \gamma \br{\frac{C^{\piout}}{2} - 1}\\
&\phantom{=}-
D_{\mathrm{KL}}(\pi,\piref) \\
&= \gamma\frac{C^\star}{2} + 
D_{\mathrm{KL}}(\pi^\star,\piref) - \gamma \frac{C^{\piout}}{2} -
D_{\mathrm{KL}}(\piout,\piref)
\end{align*}
At this point, we bound the reward estimation error under $\pi^\star$ as follows
\begin{align}
\innerprod{\pi^\star}{r^\star - \bar{r}} &= \sum_{x\in \X} \initial(x) \sum_{a\in \A} \pi^\star(a|x)(r^\star(x,a) - \bar{r}(x,a)) \nonumber \\ &= \sum_{x\in \X} \initial(x) \sum_{a\in \A} \pi^\star(a|x) \pi_{\mathrm{ref}}(b|x)(r^\star(x,a) - r^\star(x,b) - \bar{r}(x,a) + \bar{r}(x,b) ) \nonumber \\&\phantom{=}- \underbrace{\sum_{x\in \X}\initial(x)\sum_{b\in\A} \pi_{\mathrm{ref}} (b|x) (\bar{r}(x,b) - r(x,b))}_{=\mathrm{Shift}} \label{eq:intermediate}
\end{align}
At this point, we can perform the following change of measure trick. For notation compactness, let $f(x,a) = \sum_{b\in\A}\pi_{\mathrm{ref}}(b|x)(r^\star(x,a) - r^\star(x,b) - \bar{r}(x,a) + \bar{r}(x,b))$ and let $\pi$ be an arbitrary policy
\begin{align*}
\sum_{x\in \X} \initial(x) \sum_{a\in \A} \pi(a|x) f(x,a) &=
\sum_{x\in \X} \initial(x) \sum_{a\in \A} \frac{\pi(a|x)}{\sqrt{\pi_{\mathrm{ref}}(a|x)} } \sqrt{\pi_{\mathrm{ref}}(a|x)}  f(x,a) \\
&\leq \sqrt{\sum_{x\in \X} \initial(x) \sum_{a\in \A} \frac{\pi^2(a|x)}{\pi_{\mathrm{ref}}(a|x)}} \sqrt{\sum_{x\in \X} \initial(x) \sum_{a\in \A} \pi_{\mathrm{ref}}(a|x)  f^2(x,a) } \\
&\leq \sqrt{C^\pi} \sqrt{\mathrm{Err}_{\pi_{\mathrm{ref}}}(\bar{r}) },
\end{align*}
where the last inequality holds because, thanks to Jensen's inequality, we have that
\begin{align*}
\sum_{a\in \A} \pi_{\mathrm{ref}}(a|x)  f^2(x,a) &\leq \sum_{a\in\A} \sum_{b\in\A}\pi_{\mathrm{ref}}(b|x)\pi_{\mathrm{ref}}(a|x)(r^\star(x,a) - r^\star(x,b) - \bar{r}(x,a) + \bar{r}(x,b))^2\\
&= \mathrm{Err}_{\pi_{\mathrm{ref}}}(\bar{r}) .
\end{align*}
Endowed with this result, we continue by upper bounding \eqref{eq:intermediate} as follows
\begin{align*}
\innerprod{\pi^\star}{r^\star - \bar{r}} &\leq \sqrt{ C^\star \sum_{x\in \X} \initial(x) \sum_{a\in \A} \pi_{\mathrm{ref}}(a|x) \pi_{\mathrm{ref}}(b|x)(r^\star(x,a) - r^\star(x,b) - \bar{r}(x,a) + \bar{r}(x,b) )^2  } \\&\phantom{=}+\mathrm{Shift} \\
&\leq  \sqrt{ C^\star \mathrm{Err}_{\pi_{\mathrm{ref}}}(\bar{r}) }  + \mathrm{Shift} \\
&\leq  \frac{\beta'\gamma C^\star }{2}  + \frac{ \mathrm{Err}_{\pi_{\mathrm{ref}}}(\bar{r}) }{2\beta'\gamma }  + \mathrm{Shift}
\end{align*}
where in the last line we recognized the generalization error $\mathrm{Err}_{\pi_{\mathrm{ref}}}(\bar{r})$ for the hypothesis $\bar{r}$ on the response distribution generated by $\pi_{\mathrm{ref}}$. In the above derivation, we added and subtracted the term $r^\star(s,b)$ because the Bradley-Terry preference distribution is invariant under a constant additional shift of the reward function; therefore, we can only aim at learning the difference between the reward assigned to two distinct actions after seeing the same prompt. We can not learn the individual values $r^\star(s,a)$ and $r^\star(s,b)$.

Following the same steps, we can also bound the expected reward difference under $\pi_{\mathrm{out}}$.

\begin{align*}
\innerprod{\pi_{\mathrm{out}}}{\bar{r} - r^\star} &\leq \sqrt{ C^{\piout} \mathrm{Err}_{\pi_{\mathrm{ref}}}(\bar{r}) }  - \mathrm{Shift} \\
&\leq  \frac{\beta'\gamma C^{\piout}}{2}  + \frac{ \mathrm{Err}_{\pi_{\mathrm{ref}}}(\bar{r}) }{2\beta'\gamma } -  \mathrm{Shift}.
\end{align*}

At this point, plugging into 
\eqref{eq:decomposition}, we obtain 
\begin{align*}
\innerprod{\pi^\star}{r^\star} - \innerprod{\pi_{\mathrm{out}}}{r^\star} &\leq \innerprod{\pi^\star}{r-\bar{r}} + \frac{\beta' \gamma}{2} C^\star + \beta' D_{\mathrm{KL}}(\pi^\star,\piref)
+ \innerprod{\pi_{\mathrm{out}}}{\bar{r} - r} \\ &\phantom{=}- \frac{\beta'\gamma}{2} C^{\piout} - \beta' D_{\mathrm{KL}}(\piout,\piref)   \\
&\leq  \frac{\beta' \gamma C^\star}{2}  + \frac{ \mathrm{Err}_{\pi_{\mathrm{ref}}}(\bar{r}) }{2\beta'\gamma } 
\\&\phantom{=} + \frac{\beta'\gamma C^\star}{2} + \beta' D_{\mathrm{KL}}(\pi^\star,\piref)
\\&\phantom{=} +  \frac{\beta'\gamma C^{\piout}}{2}  + \frac{ \mathrm{Err}_{\pi_{\mathrm{ref}}}(\bar{r}) }{2\beta'\gamma } \\
&\phantom{=}- \frac{\beta' \gamma}{2} C^{\piout} - \beta' D_{\mathrm{KL}}(\piout,\piref)
\end{align*}
Therefore, bringing the term $\beta' D_{\mathrm{KL}}(\pi^\star,\piref)- \beta' D_{\mathrm{KL}}(\piout,\piref)$ on the left hand side and noticing on the right hand side that $\frac{\beta'\gamma C^{\piout}}{2} - \frac{\beta'\gamma C^{\piout}}{2} = 0$, we obtain that
\begin{align*}
J_{\beta'}(\pi^\star; r^\star) - J_{\beta'}(\piout; r^\star)
\leq  \beta' \gamma C^\star + \frac{ \mathrm{Err}_{\pi_{\mathrm{ref}}}(\bar{r}) }{\beta' \gamma }
\end{align*}
Therefore, if $\gamma =\sqrt{\frac{\mathrm{Err}_{\max} }{(\beta')^2 C^\star}}$ for any $\mathrm{Err}_{\max}$ such that $\mathrm{Err}_{\max} \geq \mathrm{Err}_{\pi_{\mathrm{ref}}}(\bar{r})$, we have that
\[
J_{\beta'}(\pi^\star; r^\star) - J_{\beta'}(\piout; r^\star) \leq \mathcal{O}\br{\sqrt{C^\star \mathrm{Err}_{\max}}}.
\]
\end{proof}

\subsection{Proof of \Cref{thm:DPO+distilled_main} (ML-RDPO guarantees)}\label{app:proofMLRDPOmain}
We restate here the main guarantees for ML-RDPO which we will prove in the following.
\MLRDPO*
\begin{proof}
Let $$r_{\mathrm{out}}(x,a) = \beta \log(\piout(a|x)/\piref(a|x)) + \zeta(x) $$ for some unknown state dependent function $\zeta$ and $\piout \in \argmin_{\pi \in \tilde{\Pi}} \mathcal{L}_{\mathrm{ML-RDPO}}$.
Then, by \Cref{lemma:dpo_rlhf_duality} invoked for $\gamma=0$ it holds that
\[
\piout \in \argmax_{\pi\in\Pi} \innerprod{r_{\mathrm{out}}}{\pi} - \beta D_{\mathrm{KL}}(\pi,\piref)
\]
Therefore, we have that
\begin{align*}
\innerprod{\pi^\star - \piout}{r^\star} & = \innerprod{\pi^\star - \piout}{r_{\mathrm{out}}} + \innerprod{\pi^\star - \piout}{r^\star - r_{\mathrm{out}}} \\
&\leq \beta D_{\mathrm{KL}}(\pi^\star,\piref) - \beta D_{\mathrm{KL}} (\piout,\piref) + \innerprod{\pi^\star - \piout}{r^\star - r_{\mathrm{out}}}
\end{align*}
Then, rearranging the above inequality we have that
\[
J_\beta(\pi^\star;r^\star) - J_\beta(\piout;r^\star) \leq \innerprod{\pi^\star - \piout}{r^\star - r_{\mathrm{out}}}
\]
Then, prooceding as in \Cref{thm:ratings_dpo}, we have that
\[
\innerprod{\pi^\star - \piout}{r^\star - r_{\mathrm{out}}} \leq \sqrt{C^{\max} \mathrm{Err}(r_{\mathrm{out}})},
\]
where 
\[
\mathrm{Err}(r_{\mathrm{out}}) := \mathbb{E}_{\!\!\!\!\!\substack{x \sim \initial\\a,a'\sim\pidata(\cdot|X)}}\big[\big(\Delta_{r^\star}(x,a,a') - \Delta_{r_{\mathrm{out}}}(x,a,a') \big)^2\big].
\]
Then, since $\piout$ minimizes $\mathcal{L}_{\mathrm{ML-RDPO}}$ and by the definition of $r_{\mathrm{out}}$, we can notice that $r_{\mathrm{out}}$ coincides with $\tilde{r}$, i.e. $$r_{\mathrm{out}} = \tilde{r},$$ where we recall that $\tilde{r}$ is the maximum likelihood estimator estimator for the reward function used in the derivation of ML-RDPO (see \Cref{lemma:ml-rdpo_derivation}). That is,
\[
 \tilde{r} = \argmax_{r\in\mathcal{R}_{\tilde{\Pi}}} \log P^{\mathrm{rat,rank}}_r( \mathcal{D}_{\mathrm{rat}},  \mathcal{D}_{\mathrm{rank}}).
\] Moreover, notice that $r_{\mathrm{out}} \in [-R_{\max},R_{\max}]$ since $\piout \in \tilde{\Pi}$.
Then, the proof is concluded invoking \Cref{lemma:routconc} which establishes a high probability concentration result for the reward estimate, i.e. \[
\mathrm{Err}(r_{\mathrm{out}}) \leq \min \bc{2(R_{\max} - R_{\min})^2 + 4 \mathbb{V},  16 e^{4 R_{\max}}R^2_{\max}}\frac{2 \log (\tilde{\Pi}/\delta)}{N} ,
\]
and noticing that $\tilde{\Pi} \subset \Pi$.
\end{proof}

\begin{lemma}
\label{lemma:routconc}
Let 
$r_{\mathrm{out}} = \argmax_{r\in\mathcal{R}_{\tilde{\Pi}}} \log P^{\mathrm{rat,rank}}_r( \mathcal{D}_{\mathrm{rat}},  \mathcal{D}_{\mathrm{rank}})$,
then it holds that with probability at least $1-\delta$
\[
\mathrm{Err}(r_{\mathrm{out}}) \leq \min \bc{2(R_{\max} - R_{\min})^2 + 4 \mathbb{V},  16 e^{4 R_{\max}}R^2_{\max}}\frac{2 \log (\tilde{\Pi}/\delta)}{N} .
\]
\end{lemma}
\begin{proof}

By the general bound for the conditional MLE risk \Cref{lemma:conditionalMLE}, we have that for the above MLE problem, we have that
it holds that with probability $1 - \delta$
\begin{align*}
\sum_{x\in \X}\initial(x) \sum_{a,b \in \A\times\A} \pidata(a|x) &\pidata(b|x) \int_{z\in\mathcal{Z}} \br{\sqrt{P^{\mathrm{rat,rank}}_{r_{\mathrm{out}}}(z|x,a,b)} - \sqrt{P^{\mathrm{rat,rank}}_{r^\star}(z|x,a,b)} }^2 \\&\leq \frac{2 \log (\abs{\Pi'}/\delta) }{N}.
\end{align*}
where $N$ is the dataset size and $\mathcal{Z}$ is the set of all possible joint feedbacks (i.e. all possible rating and ranking pairs) , that is $\mathcal{Z} = \bc{0,1} \times \mathbb{R}$.
Under the conditional independence assumption, we have that by the data processing inequality for any $\alpha \in [0,1]$
\begin{align*}
(1-\alpha)&\sum_{x\in \X}\initial(x) \sum_{a,b \in \A\times\A} \pidata(a|x) \pidata(b|x) \sum_{z\in\bc{0,1}} \br{\sqrt{P^{\mathrm{rank}}_{r_{\mathrm{out}}}(z|x,a,b)} - \sqrt{P^{\mathrm{rank}}_{r^\star}(z|x,a,b)} }^2 \\
&+\alpha \sum_{x\in \X}\initial(x) \sum_{a,b \in \A\times\A} \pidata(a|x) \pidata(b|x) \int^{\infty}_{-\infty} \br{\sqrt{P^{\mathrm{rat}}_{r_{\mathrm{out}}}(z|x,a,b)} - \sqrt{P^{\mathrm{rat}}_{r^\star}(z|x,a,b)} }^2 dz \\ &\leq \frac{2 \log (\abs{\tilde{\Pi}}/\delta) }{N}.
\end{align*}
By \Cref{lemma:hellinger_upper_bound}, it holds that
\[
\sum_{z\in\bc{0,1}} \br{\sqrt{P^{\mathrm{rank}}_{r_{\mathrm{out}}}(z|x,a,b)} - \sqrt{P^{\mathrm{rank}}_{r^\star}(z|x,a,b)} }^2 \geq \frac{(r_{\mathrm{out}}(x,a) - r_{\mathrm{out}}(x,b) -
r^\star(x,a) - r^\star(x,b)
)^2}{16 e^{4 R_{\max}} R^2_{\max}}.
\]
Finally, since $P^{\mathrm{rat}}_{r}$ is Gaussian for any $r \in \mathcal{R}$ with same variance $\mathbb{V}$, we have that
\begin{align*}
\int^{\infty}_{-\infty} &\br{\sqrt{P^{\mathrm{rat}}_{r_{\mathrm{out}}}(z|x,a,b)} - \sqrt{P^{\mathrm{rat}}_{r^\star}(z|x,a,b)} }^2 dz \\&= 2 - 2\exp\br{\frac{-(r_{\mathrm{out}}(x,a) - r_{\mathrm{out}}(x,b) - r^\star(x,a) - r^\star(x,b) )^2}{8 \mathbb{V}}} \\
&\geq \frac{(r_{\mathrm{out}}(x,a) - r_{\mathrm{out}}(x,b) - r^\star(x,a) - r^\star(x,b) )^2}{4 \mathbb{V} ((r_{\mathrm{out}}(x,a) - r_{\mathrm{out}}(x,b) - r^\star(x,a) - r^\star(x,b) )^2/(8\mathbb{V}) + 1)} \\
&\geq \frac{(r_{\mathrm{out}}(x,a) - r_{\mathrm{out}}(x,b) - r^\star(x,a) - r^\star(x,b) )^2}{4 \mathbb{V} ((R_{\max} - R_{\min})^2/(2\mathbb{V}) + 1)} \\
&\geq \frac{(r_{\mathrm{out}}(x,a) - r_{\mathrm{out}}(x,b) - r^\star(x,a) - r^\star(x,b) )^2}{2(R_{\max} - R_{\min})^2 + 4 \mathbb{V}},
\end{align*}
where the first equality holds because of the closed form solution for the squared Hellinger distance between Gaussian distributions with the same variance and different means \footnote{\url{https://en.wikipedia.org/wiki/Hellinger_distance}} and the first inequality holds because $1-\exp(-x) \geq \nicefrac{x}{x+1}$ for all $x\geq 0$.
All in all, recalling that we denote $$\mathrm{Err}(r_{\mathrm{out}}) := \sum_{x\in \X}\initial(x) \sum_{a,b \in \A\times\A} \pidata(a|x) \pidata(b|x) (r_{\mathrm{out}}(x,a) - r_{\mathrm{out}}(x,b) - r^\star(x,a) - r^\star(x,b) )^2.$$
Then, we get that
\[
\br{ \frac{1-\alpha}{16e^{4 R_{\max}}R^2_{\max}} + \frac{\alpha}{2(R_{\max} - R_{\min})^2 + 4 \mathbb{V}}} \mathrm{Err}(r_{\mathrm{out}}) \leq \frac{2 \log (\tilde{\Pi}/\delta)}{N} 
\]
Therefore, we conclude that with probability $1-\delta$,
\[
\mathrm{Err}(r_{\mathrm{out}}) \leq \br{\frac{16 e^{4 R_{\max}}R^2_{\max}(2(R_{\max} - R_{\min})^2 + 4 \mathbb{V})}{(2(R_{\max} - R_{\min})^2 + 4 \mathbb{V})(1-\alpha) + 16 e^{4 R_{\max}}R^2_{\max}\alpha}}\frac{2 \log (\tilde{\Pi}/\delta)}{N}.
\]
At this point, choosing $\alpha = \mathds{1}\bc{2(R_{\max} - R_{\min})^2 + 4 \mathbb{V} \leq 16 e^{4 R_{\max}}R^2_{\max}}$, we have that
\begin{align*}
\mathrm{Err}(r_{\mathrm{out}}) &\leq \br{\frac{16 e^{4 R_{\max}}R^2_{\max}(2(R_{\max} - R_{\min})^2 + 4 \mathbb{V})}{\max \bc{2(R_{\max} - R_{\min})^2 + 4 \mathbb{V},  16 e^{4 R_{\max}}R^2_{\max}}}}\frac{2 \log (\tilde{\Pi}/\delta)}{N} \\
&= \min \bc{2(R_{\max} - R_{\min})^2 + 4 \mathbb{V},  16 e^{4 R_{\max}}R^2_{\max}}\frac{2 \log (\tilde{\Pi}/\delta)}{N} .
\end{align*}
Plugging back into Eq.~\ref{eq:circle} concludes the proof.
\end{proof}

\subsection{ML-RDPO with $\chi^2$ regularization}
\label{app:CstarMLRDPO}
In this section, we prove the guarantees for ML-RDPO run with an aditional $\chi^2$ regularization term which allows to attain single policy concentrability guarantees under the setting $\pidata=\piref$.
\begin{theorem} \label{thm:DPO+distilled_app}
Let us assume that $\piref=\pidata$ and that we have conditional independence between $z_i$ and $\Delta^i_{\hat{r}}$ given $x_i, a_i, a'_i$, and rating gaps are Gaussian $\Delta^i_{\hat{r}} \sim \mathcal{N}( \Delta^i_{r}, \mathbb{V})$, for all $i \in [N]$. Under these condition, let $\piout$ as the minimizer of \Cref{eq:ML_RDPO_phi} constrained over a policy class $\tilde{\Pi}$ such that $\pi \in \tilde{\Pi}$ implies that $\Delta_{\phi \pi/\piref} \in [-R_{\max}, R_{\max}]$. Moreover, let us set  $\phi(x)= \gamma x+\log(x)$ for $\gamma = \sqrt{\frac{\mathrm{Err}_{\max}}{\beta^2 C^\star}}$ where 
\[
\mathrm{Err}_{\max} := \min \bc{2(R_{\max} - R_{\min})^2 + 4 \mathbb{V},  16 e^{4 R_{\max}}R^2_{\max}}\frac{2 \log (\tilde{\Pi}/\delta)}{N}.
\]
Then, we obtain that with probability $1-\delta$ for any $\beta \in (0, \infty)$, $$J_\beta(\pi^\star;r^\star) - J_\beta(\piout;r^\star) \leq \tilde{\mathcal{O}} \br{\sqrt{\frac{ C^\star \min\bc{ e^{R_{\max}}R^2_{\max},R_{\max}^2 + \mathbb{V}} \log (\abs{\Pi}/\delta)}{N} }}.$$
\end{theorem}
\begin{proof}
Using the same steps followed in \Cref{lemma:first_bound} changing the reward estimator from $\hat{r}$ to $r_{\mathrm{out}} = \beta \phi \br{\frac{\piout}{\piref}}$ and using the assumption that $\piref=\pidata$, we can prove that 
\begin{equation}
\innerprod{\pi^\star}{r^\star} - \innerprod{\pi_{\mathrm{out}}}{r^\star} 
 \leq  \mathcal{O}\br{ \sqrt{ \mathrm{Err} \cdot C^\star }}.\label{eq:circle}
\end{equation}
for any $\mathrm{Err} \geq \mathrm{Err}_{\piref}(r_{\mathrm{out}})$.
At this point, it remains to prove that $\mathrm{Err}_{\max}$, defined in the theorem statement, is indeed a high probability upper bound on $\mathrm{Err}_{\piref}(r_{\mathrm{out}})$.
To this end, let us introduce the reward class $\mathcal{R}_{\tilde{\Pi}}$ as follows
\[
\mathcal{R}_{\tilde{\Pi}} = \bc{ r ~|~ r = \beta \phi\br{\frac{\pi}{\piref}} ~~~\text{s.t.}~~~ \pi \in \tilde{\Pi}} .
\]

 Then, by definition of $\tilde{\Pi}$, we have that for any $r \in \mathcal{R}_{\tilde{\Pi}}$ it holds that $\Delta_r \in [-R_{\max},R_{\max}]$. Moreover, the definitions of $r_{\mathrm{out}}$ and $\pi_{\mathrm{out}}$ implies that $r_{\mathrm{out}}$ coincides with the ML-RDPO maximum likelihood estimator denoted as $\tilde{r}$. That is,
  
    \[
    r_{\mathrm{out}} = \tilde{r} = \argmax_{r\in\mathcal{R}_{\tilde{\Pi}}} \log P^{\mathrm{rat,rank}}_r( \mathcal{D}_{\mathrm{rat}},  \mathcal{D}_{\mathrm{rank}}).
    \]
Therefore, the proof is then concluded invoking again \Cref{lemma:routconc} which ensures that $\mathrm{Err}_{\piref}(r_{\mathrm{out}})$ is upper bounded in high probability by $\mathrm{Err}_{\max}$ .
\end{proof}

\section{Convergence guarantees for the heterogeneous case}
\label{app:guarantees_hetero}
In this section, we report the convergence guarantees for the case where we observe either completely disjoint or only partly overlapping rating and ranking datasets. The variants of ML-RDPO and RDPO that apply to this setting are derived in \Cref{app:hetereogenous_data}.
\subsection{Convergence guarantees for RDPO in the heterogeneous case (\Cref{alg:RDPO_heterogenous})}
We start by stating the convergence guarantees for RDPO in the heterogeneous setting. The following theorem treats the standard RDPO algorithm and the generalized version with extra $\chi^2$ regularization in a unified manner.
\begin{theorem}\label{thm:chi2}
Let us assume that $\piref=\pidata$, and run \Cref{alg:RDPO_heterogenous} for $\phi(x) = \gamma x + \log(x)$ with $\gamma =\sqrt{\frac{\mathrm{Err}_{\max}}{\beta^2(1-\alpha)^2} C^\star}$ with $\mathrm{Err}_{\max} = \frac{2\mathrm{Err_{DPO}}(N,\delta) \mathrm{Err}^{\max}_{\piref}(\hat{r}_{\mathrm{rat}}) }{\mathrm{Err_{DPO}}(N,\delta)  + \mathrm{Err}^{\max}_{\piref}(\hat{r}_{\mathrm{rat}}) }$, $\alpha = \br{1 + \mathrm{Err}^{\max}_{\piref}(\hat{r}_{\mathrm{rat}}) / \mathrm{Err_{DPO}}(N,\delta)}^{-1}$ and $\beta_1 = \frac{1-\alpha}{\alpha}\beta$.
If the ratings gap $\bc{\Delta^i_{\hat{r}_{\mathrm{rat}}}}^{N_{\mathrm{rat}}}_{i=1}$ are unbiased and $\mathbb{V}$ subgaussian, it holds that with probability at least $1-\delta$
\[
\mathrm{Err}_{\piref}(\hat{r}_{\mathrm{rat}})  \leq \mathrm{Err}^{\max}_{\piref}(\hat{r}_{\mathrm{rat}}) := \frac{c(R_{\max} + \sqrt{\mathbb{V}\log(2N_{\mathrm{rat}}/\delta)} )^2\log (\abs{\mathcal{R}}/\delta)}{N_{\mathrm{rat}}}
\]
for some $c\geq 0$ and that the suboptimality of the policy output by \Cref{alg:RDPO_heterogenous} is upper bounded, with probability at least $1-\delta$, as follows
\begin{align*}
J_{\frac{\beta\beta_1}{\beta+\beta_1}}(\pi^\star;r^\star) - J_{\frac{\beta\beta_1}{\beta+\beta_1}}(\piout;r^\star)&\leq \mathcal{O}\br{\sqrt{C^\star \frac{\mathrm{Err_{DPO}}(N,\delta) \mathrm{Err}^{\max}_{\piref}(\hat{r}_{\mathrm{rat}}) }{\mathrm{Err_{DPO}}(N,\delta)  + \mathrm{Err}^{\max}_{\piref}(\hat{r}_{\mathrm{rat}}) }}}\\
&\leq \mathcal{O} \br{\sqrt{ \frac{ R^2_{\max} e^{4 R_{\max}} (R_{\max} + \sqrt{\mathbb{V}\log(2N_{\mathrm{rat}}/\delta)})^2 \log \br{\abs{\mathcal{R}}(\abs{\mathcal{R}} + \abs{\Pi'})/\delta}}{R^2_{\max} e^{4 R_{\max}} N_{\mathrm{rat}} + (R_{\max} + \sqrt{\mathbb{V}\log(2N_{\mathrm{rat}}/\delta)})^2N_{\mathrm{rank}}} C^\star  }}.
\end{align*}
Moreover, even if $\pidata\neq\piref$, running \Cref{alg:RDPO_heterogenous} with $\gamma=0$ and all other hyperparameters unchanged ensures that with probability at least $1-\delta$
\begin{align*}
J_{\frac{\beta\beta_1}{\beta+\beta_1}}(\pi^\star;r^\star) - J_{\frac{\beta\beta_1}{\beta+\beta_1}}(\piout;r^\star)
&\leq \mathcal{O} \br{\sqrt{ \frac{ R^2_{\max} e^{4 R_{\max}} (R_{\max} + \sqrt{\mathbb{V}\log(2N_{\mathrm{rat}}/\delta)})^2 \log \br{\abs{\mathcal{R}}(\abs{\mathcal{R}} + \abs{\Pi'})/\delta}}{R^2_{\max} e^{4 R_{\max}} N_{\mathrm{rat}} + (R_{\max} + \sqrt{\mathbb{V}\log(2N_{\mathrm{rat}}/\delta)})^2N_{\mathrm{rank}}} C^{\max}  }}.
\end{align*}
\end{theorem}
Next, we provide the proofs of both statements in a unified manner.
\begin{proof}\emph{of \Cref{thm:chi2}}
Let us start by proving the suboptimality bound assuming that $\mathrm{Err}_{\piref}(\hat{r}_{\mathrm{rat}}) \leq \mathrm{Err}^{\max}_{\piref}(\hat{r}_{\mathrm{rat}})$. We will later show that this fact indeed holds with high probability.

First notice that in the heterogeneous case $\Delta^i_{\hat{r}}$ is replaced by $\Delta^i_{\hat{r}_{\mathrm{rat}}}$. Therefore,  assuming $\piref=\pidata$, and using \Cref{thm:ratings_dpo_app} for $\phi(x) = \gamma x + \log(x)$ with $\gamma =\sqrt{\frac{\mathrm{Err}_{\max}}{\beta^2(1-\alpha)^2} C^\star}$  where $\alpha = \br{1 + \mathrm{Err}^{\max}_{\piref}(\hat{r}_{\mathrm{rat}}) / \mathrm{Err_{DPO}}(N,\delta)}^{-1}$, we have that 
with probability at least $1-\delta$,
 \begin{align*}
J_{\frac{\beta\beta_1}{\beta+\beta_1}}(\pi^\star;r^\star) - J_{\frac{\beta\beta_1}{\beta+\beta_1}}(\piout;r^\star) 
 &\leq  \mathcal{O}\br{\sqrt{C^\star \br{ \mathrm{Err_{DPO}}^{-1}(N,\delta)   + {(\mathrm{Err}^{\max}_{\piref}(\hat{r}_{\mathrm{rat}}))^{-1}}}^{-1}}} \\
 &= \mathcal{O}\br{\sqrt{C^\star \frac{\mathrm{Err_{DPO}}(N,\delta) \mathrm{Err}^{\max}_{\piref}(\hat{r}_{\mathrm{rat}}) }{\mathrm{Err_{DPO}}(N,\delta)  + \mathrm{Err}^{\max}_{\piref}(\hat{r}_{\mathrm{rat}}) }}}.
 \end{align*}
 Then replacing, the expression for $\mathrm{Err_{DPO}}(N,\delta)$, we obtain
 \[
J_{\frac{\beta\beta_1}{\beta+\beta_1}}(\pi^\star;r^\star) - J_{\frac{\beta\beta_1}{\beta+\beta_1}}(\piout;r^\star) \leq \mathcal{O}\br{\sqrt{C^\star \frac{\frac{32 R^2_{\max}e^{4 R_{\max}} \log(\abs{\Pi}\abs{\mathcal{R}}/\delta)}{N} \mathrm{Err}^{\max}_{\piref}(\hat{r}_{\mathrm{rat}}) }{\frac{32 R^2_{\max}e^{4 R_{\max}} \log(\abs{\Pi}\abs{\mathcal{R}}/\delta) }{N} + \mathrm{Err}^{\max}_{\piref}(\hat{r}_{\mathrm{rat}}) }}}.
 \]
 Then, it remains to show that $\mathrm{Err}_{\piref}(\hat{r}_{\mathrm{rat}}) \leq \mathrm{Err}^{\max}_{\piref}(\hat{r}_{\mathrm{rat}})$ with high probability. To achieve this goal, we invoke a standard result concerning the generalization error of the empirical risk minimization for the square loss with subgaussian noise ( see \Cref{lemma:rating_concentration}) to obtain that with probability at least $1-2\delta$
\begin{align*}
\mathrm{Err}_{\piref}(\hat{r}_{\mathrm{rat}}) &= \sum_{x\in \X} \initial(x) \sum_{a\in \A} \pi_{\mathrm{ref}}(a|x) \pi_{\mathrm{ref}}(b|x)\br{r^\star(x,a) - r^\star(x,b) - \hat{r}_{\mathrm{rat}}(x,a)+ \hat{r}_{\mathrm{rat}}(x,b)}^2 \\
&\leq \frac{c(R_{\max}-R_{\min} + \sqrt{\mathbb{V}\log(2N_{\mathrm{rat}}/\delta)})^2\log (\abs{\mathcal{R}}/\delta)}{N_{\mathrm{rat}}}\\
&\leq \frac{c(R_{\max} + \sqrt{\mathbb{V}\log(2N_{\mathrm{rat}}/\delta)} )^2\log (\abs{\mathcal{R}}/\delta)}{N_{\mathrm{rat}}} = \mathrm{Err}^{\max}_{\piref}(\hat{r}_{\mathrm{rat}}) ,
\end{align*}
for some $c \in \mathbb{R}$. Now, plugging in this bound and rearranging yields the conclusions for the statement with for $\gamma = \sqrt{\frac{\mathrm{Err}_{\max}}{\beta^2(1-\alpha)^2}C^\star}$.
The proof for $\gamma = 0$ is analogous invoking up to the difference of invoking \Cref{thm:ratings_dpo} instead of \Cref{thm:ratings_dpo_app}.
\end{proof}

\subsection{Convergence guarantees for ML-RDPO in the heterogeneous case}

For the guarantees of ML-RDPO, let us consider that the datasets $\mathcal{D}_{\mathrm{rat}}$ and $\mathcal{D}_{\mathrm{rank}}$ are generated as follows. First, $N$ state action pairs are collected offline from $\pidata$. For each of these pairs, we observe a rating with probability $p_{\mathrm{rat,obs}}$ and a ranking with probability $p_{\mathrm{rank,obs}}$.
    If a ranking is observed, the state action pair is added to $\mathcal{D}_{\mathrm{rank}}$. Similarly, if a rating is observed, the state action pair is appended to  $\mathcal{D}_{\mathrm{rat}}$.
    Therefore, state action pairs might appear in both datasets or only in one of the two.
\begin{theorem} \label{thm:DPO+distilled_unobserved}
    For any $\beta \in (0,\infty)$. Let us assume that $\pidata = \piref$ and consider $\piout$ computed as in Eq.~\ref{eq:ML_RDPO_het} with $\gamma = \sqrt{\frac{\mathrm{Err}_{\max}}{\beta^2 C^\star}}$, where the maximum error is defined as 
    \[
    \mathrm{Err}_{\max} =  \min \bc{\frac{(2(R_{\max} - R_{\min})^2 + 4 \mathbb{V})}{\mathbb{E}[N_{\mathrm{rat}}]}, \frac{16 e^{4 R_{\max}}R^2_{\max}}{\mathbb{E}[N_{\mathrm{rank}}]}} 2 \log (\abs{\Pi}/\delta) 
    \]
    Then, it holds that with probability $1-\delta$, 
    \begin{align*}
J_\beta(\pi^\star;r^\star) - J_\beta(\piout;r^\star)
 &\leq  
 \mathcal{O} \br{\sqrt{ C^\star\min \bc{\frac{(2(R_{\max} - R_{\min})^2 + 4 \mathbb{V})}{p_{\mathrm{obs,rat}}}, \frac{16 e^{4 R_{\max}}R^2_{\max}}{p_{\mathrm{obs,rank}}}} \frac{2 \log (\abs{\Pi}/\delta)}{N} }}\\
 &=\mathcal{O} \br{\sqrt{ C^\star \mathrm{Err}_{\max} }} .
\end{align*}

Moreover, if $\gamma = 0$, even if $\pidata \neq \piref$ we have that with probability at least $1-\delta$ it holds that
\[
J_\beta(\pi^\star;r^\star) - J_\beta(\piout;r^\star)
 \leq  
\mathcal{O} \br{\sqrt{ C^{\max} \mathrm{Err}_{\max} }}.
\]
\end{theorem}
\begin{proof}
The proof differes from the homogeneous case presented in \Cref{thm:DPO+distilled_app} only because we need to bound the error of the implicit reward estimator, i.e. $\mathrm{Err}(r_{\mathrm{out}})$ differently as presented next.

For any $r \in \mathcal{R}_{\tilde{\Pi}}$, let us define the ratings and ranking probability law to take into account the possibility of not observing a rating or a ranking, respectively.
In particular, let us denote by $\emptyset$ the event where the rating or the ranking is not observed. Then, we define
\[
\tilde{P}^{\mathrm{rat}}_{r}(z|x,a,b) = 
\begin{cases}
P^{\mathrm{rat}}_{r}(z|x,a,b) p_{\mathrm{obs,rat}} &~~\text{if}~~ z \neq \emptyset\\
1 -  p_{\mathrm{obs,rat}} &~~\text{otherwise}
\end{cases}
\]
and
\[
\tilde{P}^{\mathrm{rank}}_{r}(z|x,a,b) = 
\begin{cases}
P^{\mathrm{rank}}_{r}(z|x,a,b) p_{\mathrm{obs,rank}} &~~\text{if}~~ z \neq \emptyset\\
1 -  p_{\mathrm{obs,rank}} &~~\text{otherwise}
\end{cases}
\]
ML-RDPO can therefore be seen as the algorithm seeking implicitly the maximizer of the following loglikelihood problem
\begin{align*}
r_{\mathrm{out}} &= \argmax_{r\in\mathcal{R}_{\tilde{\Pi}}} \log \tilde{P}^{\mathrm{rat,rank}}_r( \mathcal{D}_{\mathrm{rat}},  \mathcal{D}_{\mathrm{rank}}) \\&= \argmax_{r\in\mathcal{R}_{\tilde{\Pi}}} \sum^N_{i=1} \log (\tilde{P}^{\mathrm{rat}}_r(\Delta^i_{\hat{r}} | x'_i, \tilde{a}_i, \bar{a}_i)\tilde{P}^{\mathrm{rank}}_r(z_i | x_i, a_i, a'_i))
\end{align*}
where $\Delta^i_{\hat{r}}$'s and $z_i$'s are possibly equal to $\emptyset$.
Therefore, let us consider the squared Hellinger divergence
\[
D^2_H(p,q) = \int_{z\in\mathcal{Z}} \br{\sqrt{p(z|x,a,b)} - \sqrt{q(z|x,a,b)} }^2
\]
via the general bound for the conditional MLE risk, we have that with probability $1-\delta$,
\begin{align*}
\sum_{x\in \X}\initial(x) &\sum_{a,b \in \A\times\A} \pidata(a|x) \pidata(b|x) \int_{z\in\mathcal{Z}} \br{\sqrt{\tilde{P}^{\mathrm{rat,rank}}_{r_{\mathrm{out}}}(z|x,a,b)} - \sqrt{\tilde{P}^{\mathrm{rat,rank}}_{r^\star}(z|x,a,b)} }^2 \\ & \leq \frac{2 \log (\abs{\Pi}/\delta) }{N},
\end{align*}
where we also upper bounded the size of the hypothesis class using the fact that $\tilde{\Pi} \subset \Pi$.
At this point, due to the data processing inequality for the squared Hellinger divergence, we have that for all $x,a,b \in \X\times\A\times\A$ it holds that
\[
D_H^2(\tilde{P}^{\mathrm{rat,rank}}_{r_{\mathrm{out}}}(\cdot|x,a,b), \tilde{P}^{\mathrm{rat,rank}}_{r^\star}(\cdot|x,a,b)) \geq D_H^2(\tilde{P}^{\mathrm{rat}}_{r_{\mathrm{out}}}(\cdot|x,a,b), \tilde{P}^{\mathrm{rat}}_{r^\star}(\cdot|x,a,b))
\]
and
\[
D_H^2(\tilde{P}^{\mathrm{rat,rank}}_{r_{\mathrm{out}}}(\cdot|x,a,b), \tilde{P}^{\mathrm{rat,rank}}_{r^\star}(\cdot|x,a,b)) \geq D_H^2(\tilde{P}^{\mathrm{rank}}_{r_{\mathrm{out}}}(\cdot|x,a,b), \tilde{P}^{\mathrm{rank}}_{r^\star}(\cdot|x,a,b)).
\]
Therefore, for any $\alpha\in[0,1]$ we have that
\[
(1-\alpha) D_H^2(\tilde{P}^{\mathrm{rank}}_{r_{\mathrm{out}}}(\cdot|x,a,b), \tilde{P}^{\mathrm{rank}}_{r^\star}(\cdot|x,a,b)) + \alpha D_H^2(\tilde{P}^{\mathrm{rat}}_{r_{\mathrm{out}}}(\cdot|x,a,b), \tilde{P}^{\mathrm{rat}}_{r^\star}(\cdot|x,a,b)) \leq \frac{2 \log(\abs{\Pi}/\delta)}{N}
\]
At this point, let us connect the squared Hellinger divergence $D_H^2(\tilde{P}^{\mathrm{rat}}_{r_{\mathrm{out}}}(\cdot|x,a,b), \tilde{P}^{\mathrm{rat}}_{r^\star}(\cdot|x,a,b)) $ with $D_H^2(P^{\mathrm{rat}}_{r_{\mathrm{out}}}(\cdot|x,a,b), P^{\mathrm{rat}}_{r^\star}(\cdot|x,a,b)).$
\begin{align*}
D_H^2(\tilde{P}^{\mathrm{rat}}_{r_{\mathrm{out}}}(\cdot|x,a,b),& \tilde{P}^{\mathrm{rat}}_{r^\star}(\cdot|x,a,b))  = \int_{z\in\mathcal{Z} \setminus\bc{\emptyset}} \br{\sqrt{\tilde{P}^{\mathrm{rat}}_{r_{\mathrm{out}}}(z|x,a,b)} - \sqrt{\tilde{P}^{\mathrm{rat}}_{r^\star}(z|x,a,b)} }^2  \\
&\phantom{=}+ \br{\sqrt{\tilde{P}^{\mathrm{rat}}_{r_{\mathrm{out}}}(\emptyset|x,a,b)} - \sqrt{\tilde{P}^{\mathrm{rat}}_{r^\star}(\emptyset|x,a,b)} }^2 \\
&= \int_{z\in\mathcal{Z} \setminus\bc{\emptyset}} \br{\sqrt{\tilde{P}^{\mathrm{rat}}_{r_{\mathrm{out}}}(z|x,a,b)} - \sqrt{\tilde{P}^{\mathrm{rat}}_{r^\star}(z|x,a,b)} }^2 \\
&= \int_{z\in\mathcal{Z} \setminus\bc{\emptyset}} \br{\sqrt{p_{\mathrm{obs,rat}}P^{\mathrm{rat}}_{r_{\mathrm{out}}}(z|x,a,b)} - \sqrt{p_{\mathrm{obs,rat}}P^{\mathrm{rat}}_{r^\star}(z|x,a,b)} }^2 \\
&= p_{\mathrm{obs,rat}} \int_{z\in\mathcal{Z} \setminus\bc{\emptyset}} \br{\sqrt{P^{\mathrm{rat}}_{r_{\mathrm{out}}}(z|x,a,b)} - \sqrt{P^{\mathrm{rat}}_{r^\star}(z|x,a,b)} }^2 \\
&= p_{\mathrm{obs,rat}} D_H^2(P^{\mathrm{rat}}_{r_{\mathrm{out}}}(\cdot|x,a,b), P^{\mathrm{rat}}_{r^\star}(\cdot|x,a,b)).
\end{align*}
Analogously, one can get
\[
D_H^2(\tilde{P}^{\mathrm{rank}}_{r_{\mathrm{out}}}(\cdot|x,a,b), \tilde{P}^{\mathrm{rank}}_{r^\star}(\cdot|x,a,b))  =  p_{\mathrm{obs,rank}} D_H^2(P^{\mathrm{rank}}_{r_{\mathrm{out}}}(\cdot|x,a,b), P^{\mathrm{rank}}_{r^\star}(\cdot|x,a,b)).
\]
At this point, with the same steps followed for the proof of \Cref{thm:DPO+distilled_main}, we can get
\[
D_H^2(P^{\mathrm{rank}}_{r_{\mathrm{out}}}(\cdot|x,a,b), P^{\mathrm{rank}}_{r^\star}(\cdot|x,a,b)) \geq \frac{(r_{\mathrm{out}}(x,a) - r_{\mathrm{out}}(x,b) -
r^\star(x,a) - r^\star(x,b)
)^2}{16 e^{4 R_{\max}} R^2_{\max}}
\]
and 
\[
D_H^2(P^{\mathrm{rat}}_{r_{\mathrm{out}}}(\cdot|x,a,b), P^{\mathrm{rat}}_{r^\star}(\cdot|x,a,b)) \geq \frac{(r_{\mathrm{out}}(x,a) - r_{\mathrm{out}}(x,b) - r^\star(x,a) - r^\star(x,b) )^2}{2(R_{\max} - R_{\min})^2 + 4 \mathbb{V}}.
\]
All in all, we get, 
\begin{align*}
(1-\alpha) &D_H^2(\tilde{P}^{\mathrm{rank}}_{r_{\mathrm{out}}}(\cdot|x,a,b), \tilde{P}^{\mathrm{rank}}_{r^\star}(\cdot|x,a,b)) + \alpha D_H^2(\tilde{P}^{\mathrm{rat}}_{r_{\mathrm{out}}}(\cdot|x,a,b), \tilde{P}^{\mathrm{rat}}_{r^\star}(\cdot|x,a,b)) \\ &\geq 
(1-\alpha) \frac{p_{\mathrm{obs,rank}}(r_{\mathrm{out}}(x,a) - r_{\mathrm{out}}(x,b) -
r^\star(x,a) - r^\star(x,b)
)^2}{16 e^{4 R_{\max}} R^2_{\max}} \\ &\phantom{=}+ \alpha \frac{p_{\mathrm{obs,rat}}(r_{\mathrm{out}}(x,a) - r_{\mathrm{out}}(x,b) - r^\star(x,a) - r^\star(x,b) )^2}{2(R_{\max} - R_{\min})^2 + 4 \mathbb{V}}.
\end{align*}
Therefore,
\[
\br{ \frac{(1-\alpha) p_{\mathrm{obs,rank}}}{16e^{4 R_{\max}}R^2_{\max}} + \frac{\alpha p_{\mathrm{obs,rat}}}{2(R_{\max} - R_{\min})^2 + 4 \mathbb{V}}} \mathrm{Err}_{\piref}(r_{\mathrm{out}})
\leq \frac{2 \log (\abs{\Pi}/\delta)}{N},
\]
which leads to
the following bound with probability $1-\delta$,
\[
\mathrm{Err}_{\piref}(r_{\mathrm{out}}) \leq \br{\frac{16 e^{4 R_{\max}}R^2_{\max}(2(R_{\max} - R_{\min})^2 + 4 \mathbb{V})/ p_{\mathrm{obs,rank}} p_{\mathrm{obs,rat}} }{\nicefrac{(2(R_{\max} - R_{\min})^2 + 4 \mathbb{V})}{p_{\mathrm{obs,rat}}}(1-\alpha) + \nicefrac{16 e^{4 R_{\max}}R^2_{\max}}{p_{\mathrm{obs,rank}}}\alpha }}\frac{2 \log (\abs{\Pi}/\delta)}{N}.
\]
Therefore, for $\alpha = \mathds{1}\bc{ \nicefrac{16 e^{4 R_{\max}}R^2_{\max}}{p_{\mathrm{obs,rank}}} \geq \nicefrac{(2(R_{\max} - R_{\min})^2 + 4 \mathbb{V})}{p_{\mathrm{obs,rat}}}}$, we obtain that with probability at least $1-\delta$
\[
\mathrm{Err}_{\piref}(r_{\mathrm{out}}) \leq \mathrm{Err}_{\max} := \min \bc{\frac{(2(R_{\max} - R_{\min})^2 + 4 \mathbb{V})}{p_{\mathrm{obs,rat}}}, \frac{16 e^{4 R_{\max}}R^2_{\max}}{p_{\mathrm{obs,rank}}}} \frac{2 \log (\abs{\Pi}/\delta)}{N} .
\]
which allows to conclude the proof. The equality in the theorem statement holds because $p_{\mathrm{obs,rat}} N = \mathbb{E}[N_{\mathrm{rat}}]$ and $p_{\mathrm{obs,rank}} N = \mathbb{E}[N_{\mathrm{rank}}]$.

The proof for the case $\gamma = 0$, follows from the fact that
\[
J_\beta(\pi^\star;r^\star) - J_\beta(\piout;r^\star) \leq \sqrt{C^{\max} \mathrm{Err}(r_{\mathrm{out}})}
\]
established in the proof of \Cref{thm:DPO+distilled_main} and the bound $\mathrm{Err}(r_{\mathrm{out}}) \leq \mathrm{Err}_{\max}$ which holds with probability  at least $1-\delta$ as shown above.
\end{proof}
\section{Technical Lemmas}
This section contains technical results that are used in proving the guarantees for RDPO and ML-RDPO.
We start with an important duality result between RLHF and DPO.
Before delving in the proof, let us state \Cref{ass:policy_realizability_main} in the following equivalent form.
\begin{assumption}
Let us consider any reward $r:\X\times\A\rightarrow [R_{\min}, R_{\max}]$. Then, we assume that $\Pi$ is such that for any $\gamma\in[0, \infty)$, it holds that
\begin{align*}
\argmax_{\pi\in \Pi}& \innerprod{\pi}{r} - \beta \gamma D_{\chi^2}(\pi , \pi_{\mathrm{ref}}) - \beta D_{\mathrm{KL}}(\pi , \pi_{\mathrm{ref}})
\\ &= \argmax_{\pi\in \Pi_{\mathrm{all}}} \innerprod{\pi}{r} - \beta \gamma D_{\chi^2}(\pi , \pi_{\mathrm{ref}}) - \beta D_{\mathrm{KL}}(\pi , \pi_{\mathrm{ref}})
\end{align*}
\label{ass:policy_realizability}
\end{assumption}
The assumption says that the policy class $\Pi$ should be expressive enough to realize the solution of all RLHF problems associated with all possible bounded reward functions. 
The next lemma leverages the rewriting of Assumption~\ref{ass:policy_realizability_main} to show that any policy $\pi$ is the maximizer of the regularized RL problem having as reward the function $\beta \phi(\pi/\piref)$.
\begin{lemma}
\label{lemma:dpo_rlhf_duality}
 Let \Cref{ass:policy_realizability} hold. For any policy $\pi \in \Pi$, let us consider the reward function $r_{\pi}(x,a) = \beta  \phi\br{\frac{\pi(a|x)}{\piref(a|x)}}$ where $\phi(x) = \gamma x + \log(x)$.
Then, it holds true that
\[
\pi \in \argmax_{p\in \Pi} \innerprod{p}{r_{\pi}} - \beta \gamma D_{\chi^2}(p , \pi_{\mathrm{ref}}) - \beta D_{\mathrm{KL}}(p, \pi_{\mathrm{ref}}) .
\]
\end{lemma}
\begin{proof}
By \Cref{ass:policy_realizability}, we have that $\pi$ is the solution when the domain is enlarged from $\Pi$ to $\Pi_{\mathrm{all}}$, that is
\[
\pi \in \argmax_{p\in \Pi_{\mathrm{all}}} \innerprod{p}{r_{\pi}} - \beta \gamma D_{\chi^2}(p , \pi_{\mathrm{ref}}) - \beta D_{\mathrm{KL}}(p, \pi_{\mathrm{ref}}) 
\]

At this point, 
noticing that $\gamma D_{\chi^2}(\pi,\pi') + D_{\mathrm{KL}}(\pi,\pi')$ can be interpreted as $f$-divergence generated by $f(x) = \gamma \frac{x^2}{2} + x\log(x)$. Notice that $f'(x) = \gamma x + \log(x) + 1$ and therefore
since $0 \notin \mathrm{dom}(f')$. Therefore, the conditions of \citet[Lemma F.2]{huang2024correcting} are satisfied and we can invoke their result to conclude the proof.
\end{proof}

Next, we present two results that bound the generalization error for the estimator learned via ranking information. The first result upper bounds the absolute value error of $r_{\mathrm{out}}$ in terms of the Hellinger divergences between the Bradley-Terry model with ground truth reward $r$ and the one with reward $r_{\mathrm{out}}$.
\begin{lemma}(adapted from \citet[Lemma F.5]{huang2024correcting})
\label{lemma:hellinger_upper_bound}
Recall that $r^\star(x,a) - r^\star(x,b) \in [R_{\min} - R_{\max}, R_{\max} - R_{\min}]$ and $r_{\mathrm{out}}(x,a) - r_{\mathrm{out}}(x,b) \in [R_{\min} - R_{\max}, R_{\max} - R_{\min}]$ , then we define \[
P_r(\mathds{1}\bc{a ~\text{is preferred to~} b} |x,a,b) = \sigma  (r(x,a) - r(x,b))
\]
for any $r \in \mathcal{R}_{\Pi'}$. Then, it holds that
\begin{align*}
|r^\star(x,a) - r^\star(x,b) &- r_{\mathrm{out}}(x,a) + r_{\mathrm{out}}(x,b)| \\&\leq 4 e^{2 R_{\max}} R_{\max}
\sqrt{\sum_{z\in\bc{0,1}} (\sqrt{P_{r^\star}(z|x,a,b)} - \sqrt{P_{r_{\mathrm{out}}}(z|x,a,b)})^2 }
\end{align*}
\end{lemma}
Second, we invoke a classic result on the generalization error of the maximum likelihood estimator to bound the Hellinger divergence.
\begin{lemma}(Conditional MLE Risk)\label{lemma:conditionalMLE} 
For any policy class $\Pi$, let us consider the following policy-induced reward class
\[
\mathcal{R}_{\Pi} = \bigg\{ r ~~|~~\exists \pi \in \Pi ~~\text{s.t.}~~r(x,a) = \beta \phi\br{\frac{\pi(a|x)}{\piref(a|x)}} ~~\forall x,a \in \X\times\A\bigg\}.
\]
Moreover, consider the conditional density $P_r:  \X \times \A \times \A \rightarrow \bs{0,1}$, the dataset $\mathcal{D}_{\mathrm{rank}} = \bc{x_i, a^+_i, a^-_i}$ with $\abs{\mathcal{D}_{\mathrm{rank}}} = N_{\mathrm{rank}}$ and the estimator
\[
r_{\mathrm{out}} = \argmax_{r \in \mathcal{R}_{\Pi}} \sum^{N_{\mathrm{rank}}}_{i=1} \log P_r(a_i^+ ~\text{is preferred to} ~a_i^{-} |x_i, a_i^-, a_i^+).
\]
Then, it holds that with probability $1 - \delta$
\begin{align*}
\sum_{x\in \X}\initial(x) & \sum_{a,b \in \A\times\A} \pidata(a|x) \pidata(b|x) \sum_{z\in\bc{0,1}} \br{\sqrt{P_{r_{\mathrm{out}}}(z|x,a,b)} - \sqrt{P_{r^\star}(z|x,a,b)} }^2 \\ &\leq \frac{2 \log (\abs{\Pi}/\delta) }{N_{\mathrm{rank}}}.
\end{align*}
\end{lemma}
Finally, we provide the bound for the least square estimate using as targets the observed rating gaps, which we assume to be unbiased and Gaussian distributed.
\begin{lemma}(Concentration for the rating reward)\label{lemma:rating_concentration}
Let us consider the generalization error for the rating reward estimator
\[
\mathrm{Err}(\hat{r}_{\mathrm{rat}}) =\sum_{x\in \X} \initial(x) \sum_{a\in \A} \pidata(a|x) \pidata(b|x)\br{r^\star(x,a) - r^\star(x,b) - \hat{r}_{\mathrm{rat}}(x,a)+ \hat{r}_{\mathrm{rat}}(x,b)}^2
\]
where 
\[
\hat{r}_{\mathrm{rat}} = \argmin_{r \in \mathcal{R}} \sum^{N_{\mathrm{rat}}}_{i=1} (r(x'_i, \bar{a}_i) - r(x'_i, \tilde{a}_i) - \Delta^i_{\hat{r}} )^2,
\]
and $\Delta^i_{\hat{r}} | x'_i, \bar{a}_i, \tilde{a}_i$ is a $\mathbb{V}$-sub Gaussian random variable with mean $r^\star(x'_i, \bar{a}_i) - r^\star(x'_i, \tilde{a}_i)$. Then we have that
with probability at least $1-2\delta$,
\[
\mathrm{Err}(\hat{r}_{\mathrm{rat}}) \leq \mathcal{O}\br{\frac{(R_{\max} - R_{\min} + \sqrt{\mathbb{V}\log (2N_{\mathrm{rat}}/\delta)})^2 \log(\frac{\abs{\mathcal{R}}}{\delta})}{N_{\mathrm{rat}}}}
\]
\end{lemma}
\begin{proof}
By the definition of a subgaussian random variable, we have that
\[
\mathbb{P}\br{ \abs{\Delta^i_{\hat{r}} - r^\star(x'_i, \bar{a}_i) + r^\star(x'_i, \tilde{a}_i)} \geq t ~~~\bigg|~~~ x'_i, \bar{a}_i, \tilde{a}_i} \leq 2 e^{-t^2/\mathbb{V}}
\]
for any scalar $t \in \mathbb{R}$. Therefore, for $t = \sqrt{\mathbb{V} \log \br{2 N_{\mathrm{rat}}/\delta}}$, and a union bound we have that
\[
\mathbb{P}\br{ \abs{\Delta^i_{\hat{r}} - r^\star(x'_i, \bar{a}_i) + r^\star(x'_i, \tilde{a}_i)} \geq  \sqrt{ \mathbb{V} \log \br{2 N_{\mathrm{rat}} /\delta}} ~~\forall i \in [N_{\mathrm{rat}}] ~~~\bigg|~~~ x'_i, \bar{a}_i, \tilde{a}_i} \leq \delta
\]
\[ \implies
\mathbb{P}\br{ \abs{\Delta^i_{\hat{r}}} \geq R_{\max} - R_{\min} + \sqrt{ \mathbb{V} \log \br{2 N_{\mathrm{rat}} /\delta}} ~~\forall i \in [N_{\mathrm{rat}}] ~~~\bigg|~~~ x'_i, \bar{a}_i, \tilde{a}_i} \leq \delta
\]
Therefore, with probability $1-\delta$ the regression targets $\bc{\Delta^i_{\hat{r}}}^{N_{\mathrm{rat}}}_{i=1}$ are in the interval $[- R_{\max} + R_{\min} -  \sqrt{\mathbb{V} \log \br{2  N_{\mathrm{rat}}/\delta}}, R_{\max} - R_{\min} + \sqrt{ \mathbb{V} \log \br{2N_{\mathrm{rat}}/\delta}}]$.
Invoking \cite{foster2023foundations}[Chapter 1, Exercise 1 ], we have that under the event for $\mathcal{E}_{\mathrm{bounded}} = \bc{\Delta^i_{\hat{r}} \in [-B, B], B=R_{\max} - R_{\min} +  \sqrt{\mathbb{V} \log \br{2N_{\mathrm{rat}}/\delta}}}$, it holds that
\[
\mathbb{P}\br{\mathrm{Err}(\hat{r}_{\mathrm{rat}}) \leq \mathcal{O}\br{\frac{B^2 \log(\frac{\abs{\mathcal{R}}}{\delta})}{N_{\mathrm{rat}}}} \bigg | \mathcal{E}_{\mathrm{bounded}} } \geq 1-\delta.
\]
Finally, to remove the conditioning $\mathcal{E}_{\mathrm{bounded}}$, by the law of total probability we have that
\begin{align*}
\mathbb{P}&\br{\mathrm{Err}(\hat{r}_{\mathrm{rat}}) \geq \mathcal{O}\br{\frac{B^2 \log(\frac{\abs{\mathcal{R}}}{\delta})}{N_{\mathrm{rat}}}} } \\&= \mathbb{P}\br{\mathrm{Err}(\hat{r}_{\mathrm{rat}}) \geq \mathcal{O}\br{\frac{B^2 \log(\frac{\abs{\mathcal{R}}}{\delta})}{N_{\mathrm{rat}}}} \bigg | \mathcal{E}_{\mathrm{bounded}} } \mathbb{P}\br { \mathcal{E}_{\mathrm{bounded}} } \\ &\phantom{=} + 
\mathbb{P}\br{\mathrm{Err}(\hat{r}_{\mathrm{rat}}) \geq \mathcal{O}\br{\frac{B^2 \log(\frac{\abs{\mathcal{R}}}{\delta})}{N_{\mathrm{rat}}}} \bigg | \mathcal{E}^c_{\mathrm{bounded}} } \mathbb{P}\br { \mathcal{E}^c_{\mathrm{bounded}} } \\
&\leq \mathbb{P}\br{\mathrm{Err}(\hat{r}_{\mathrm{rat}}) \geq \mathcal{O}\br{\frac{B^2 \log(\frac{\abs{\mathcal{R}}}{\delta})}{N_{\mathrm{rat}}}} \bigg | \mathcal{E}_{\mathrm{bounded}} }  + \mathbb{P}\br { \mathcal{E}^c_{\mathrm{bounded}} } \\
&\leq 2\delta.
\end{align*}
\end{proof}
\newpage
\section{Additional Experiments}
\label{app:experiments}
In this section, we provide a clarification concerning how we compute the error bars in our plots and some additional experiments.

\paragraph{Error Bars} The error bars reported in all plots showing the win rate against GPT-4 are measured as standard deviations over the evaluations sets: i.e. AlpacaEval and ArenaHard treating each prompt as an independent sample. This is the standard estimation of the win rate uncertainty in the  AlpacaEval and ArenaHard libraries. Differently in Figure~\ref{fig:intro_exp} we compute the standard deviation of the win rate across the different choices of $\piref$.

\paragraph{Additional experiments} Next, we show in \Cref{sec:battles} the win rates between each model pairs trained starting from a base model with 7 or 8 B parameters . Then, we report additional experiments with perturbed rating information in \Cref{sec:noisy}.  Experiments on the smaller models SmolLM2-135M\footnote{\url{https://huggingface.co/HuggingFaceTB/SmolLM2-135M}}, SmoLM2-360M \footnote{\url{https://huggingface.co/HuggingFaceTB/SmolLM2-360M}} and SmolLM2-1.7B \footnote{\url{https://huggingface.co/HuggingFaceTB/SmolLM2-1.7B}} are presented in \Cref{sec:SmolLM2}. 
Finally, we present in \Cref{app:zephyr_ablation} the ablations we performed on Zephyr-7B to choose the values of $\beta_1$ and $\mathbb{V}$ which we used also for the experiments on Llama-3.1-8B and Mistral-7B. We conclude in \Cref{app:exp_constrain} showing the practical benefit of restricting the policy class to $\Pi'$ as done for the proof of \Cref{thm:ratings_dpo}.
\subsection{Relative win rates at 7/8B scale}
\label{sec:battles}
\begin{figure*}[!h] 
\centering
\begin{tabular}{ccc}
\subfloat[Llama-3.1-8B]{%
    \includegraphics[width=0.31\linewidth]{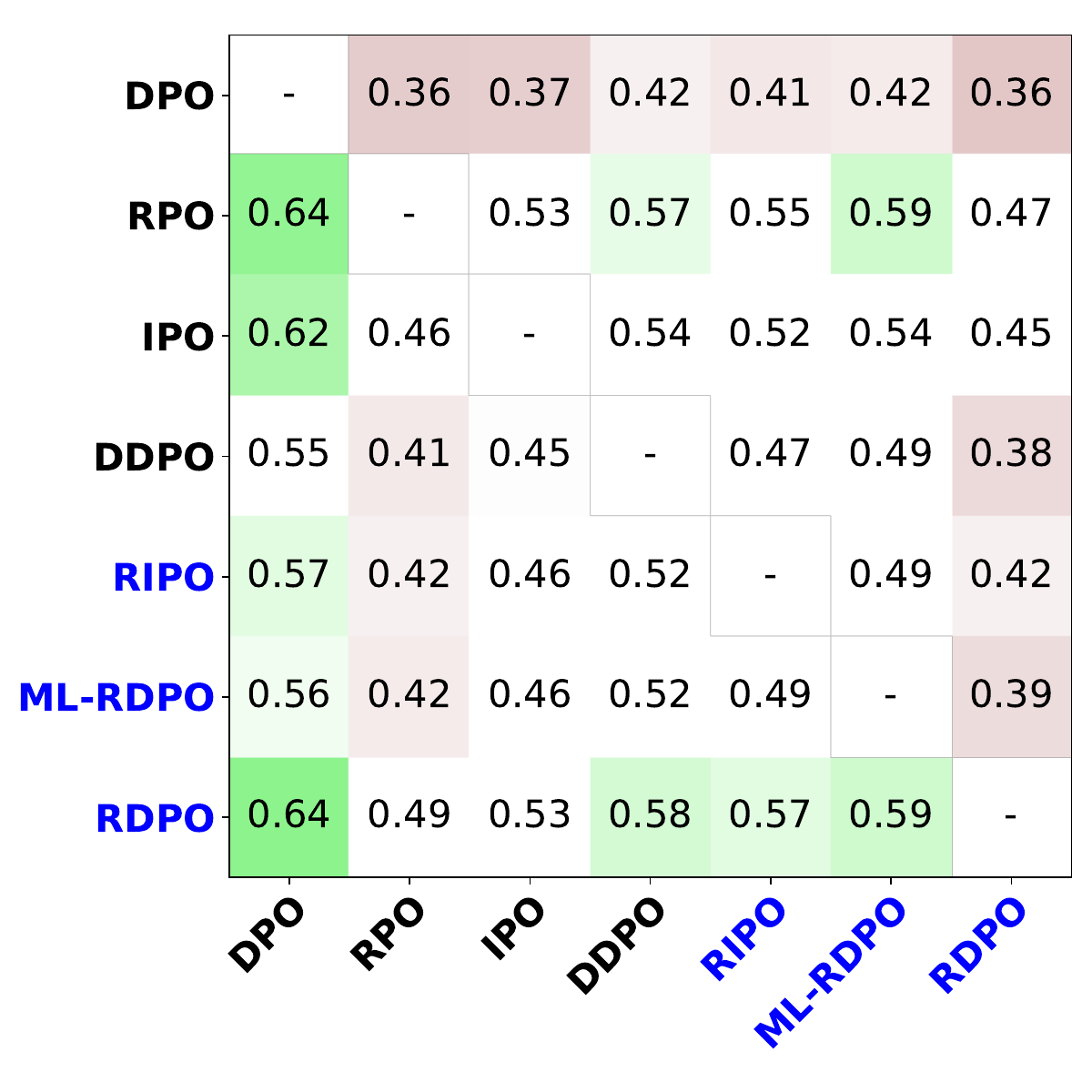}
     } &
\subfloat[Mistral-7B]{%
\includegraphics[width=0.31\linewidth]{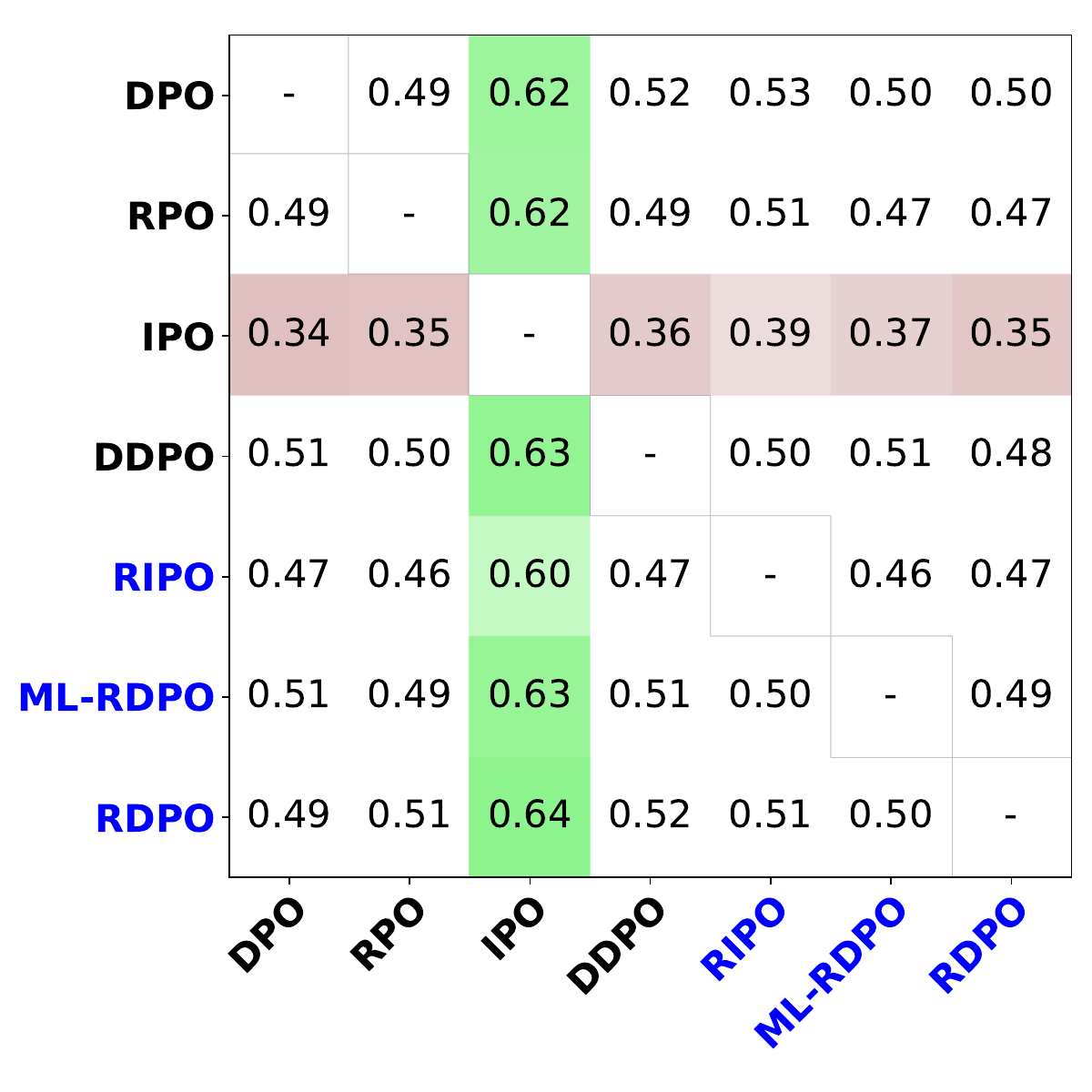}
     } &
\subfloat[Zephyr-7B]{%
\includegraphics[width=0.31\linewidth]{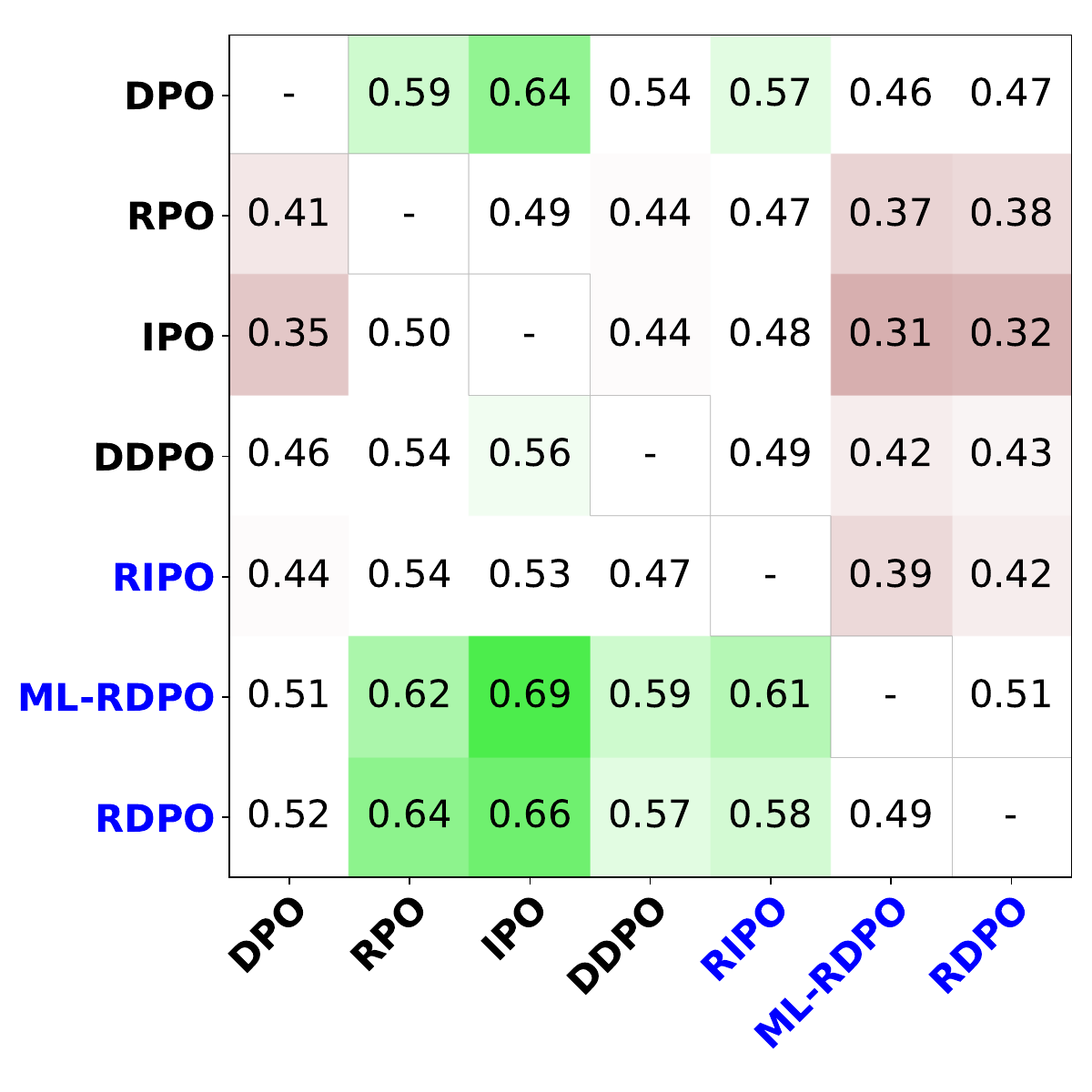}
     } 
\end{tabular}
\caption{ Model by model comparison using Llama, Mistral and Zephyr as base models. Each entry of the table reports the win rate of the row player against the column player.}
\label{fig:7Bbattles}
\end{figure*}
We see that RDPO achieves a high win rate against any other alignment methods for all three choices for $\piref$. ML-RDPO generally achieves an even higher win rate for Zephyr and Mistral while it underperforms slightly in Llama. IPO and RPO are strong baselines for the experiment with Llama as a base model. The attentive reader might notice that the antidiagonal elements do not sum to one. This is because we run two independent judgments for each pair of models. 

\subsection{Additional experiments with noisy rating information}
\label{sec:noisy}
\begin{figure*}[!h] 
\centering
\begin{tabular}{cc}
\subfloat[Zephyr-7B]{%
    \includegraphics[width=0.5\linewidth]{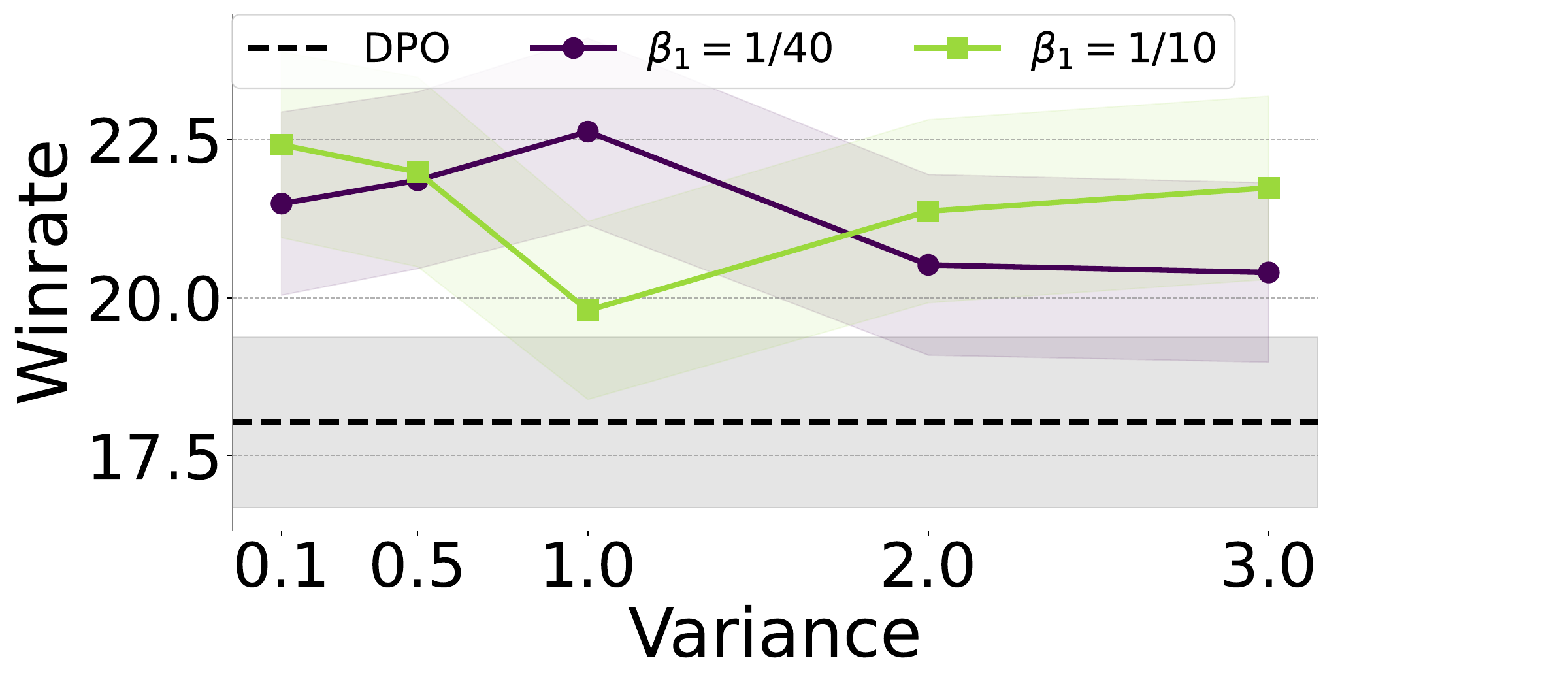}
     } &
\subfloat[Mistral-7B]{%
\includegraphics[width=0.5\linewidth]{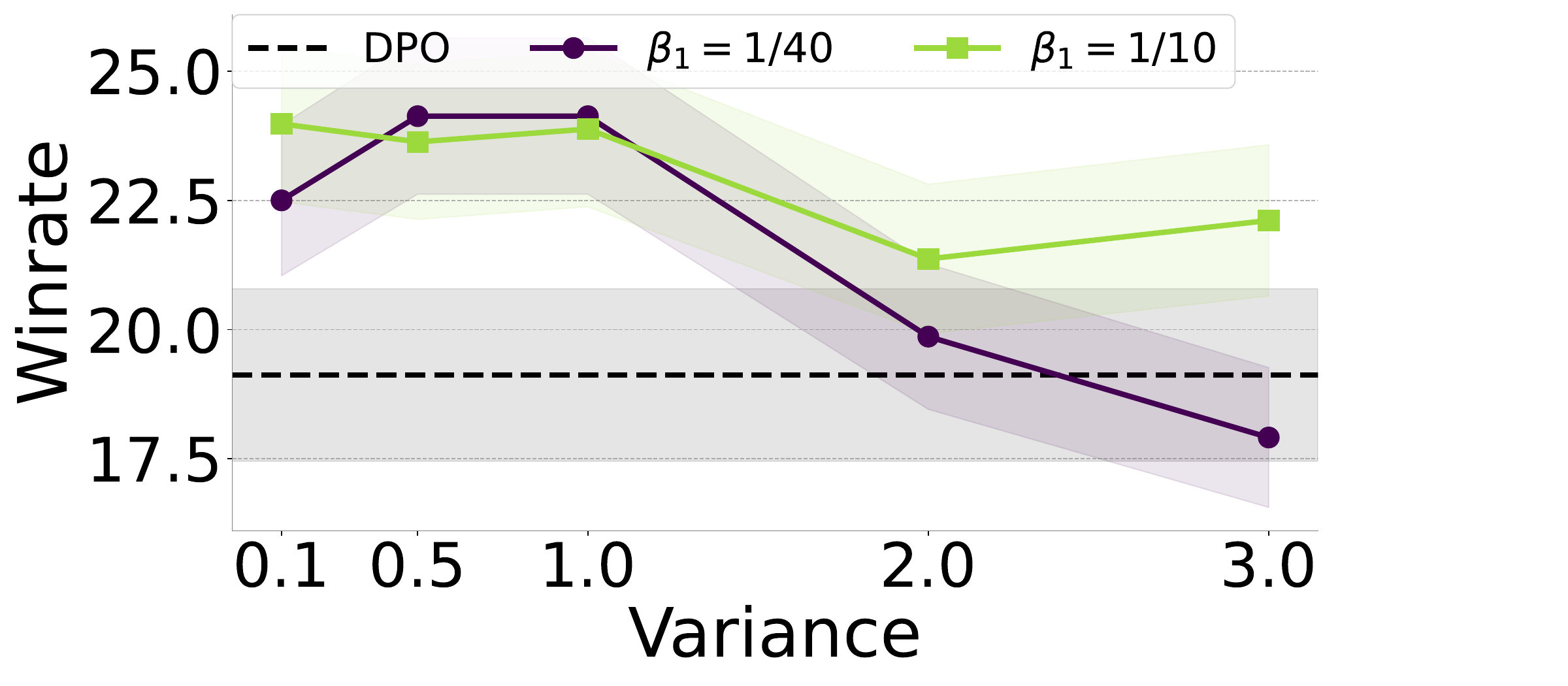}
     }
\end{tabular}
\caption{Variance ablation in Mistral-7B and Zephyr-7B.}
\label{fig:variance_zephyr_mistral}
\end{figure*}
In \Cref{fig:variance_zephyr_mistral} we report the win rate against GPT-4 as judged by Claude-Sonnet-3.5v2 for the base models Zephyr-7B and Mistral-7B. The experiment aims at monitoring the win rate as a function of the variance of the Gaussian noise injected in the \texttt{ultrafeedback} rating.

As we found using Llama-8B as base model ( see \Cref{fig:corruption_via_noise}), both choices for $\beta_1$, i.e. $\beta_1=1/10$
and $\beta_1=1/40$, perform similarly at low variance level while $\beta_1=1/10$ performs better when high variance values are considered.
This behaviour is consistent with the theoretical prediction that a higher $\beta_1$  should be considered when the rating information is inaccurate.
\subsection{Experiments with Smol-LM2 Family}
\label{sec:SmolLM2}
We experiment with the SmolLM2 family as an initial test; however, small language models can be important in some situations, such as on-device applications \cite{xu2024device}. Therefore, we include the results in the following. In particular, we report the results of the relative win rates between any pair of considered alignment methods.
\begin{figure*}[!h] 
\centering
\begin{tabular}{ccc}
\subfloat[SmolLM2-135M]{%
    \includegraphics[width=0.3\linewidth]{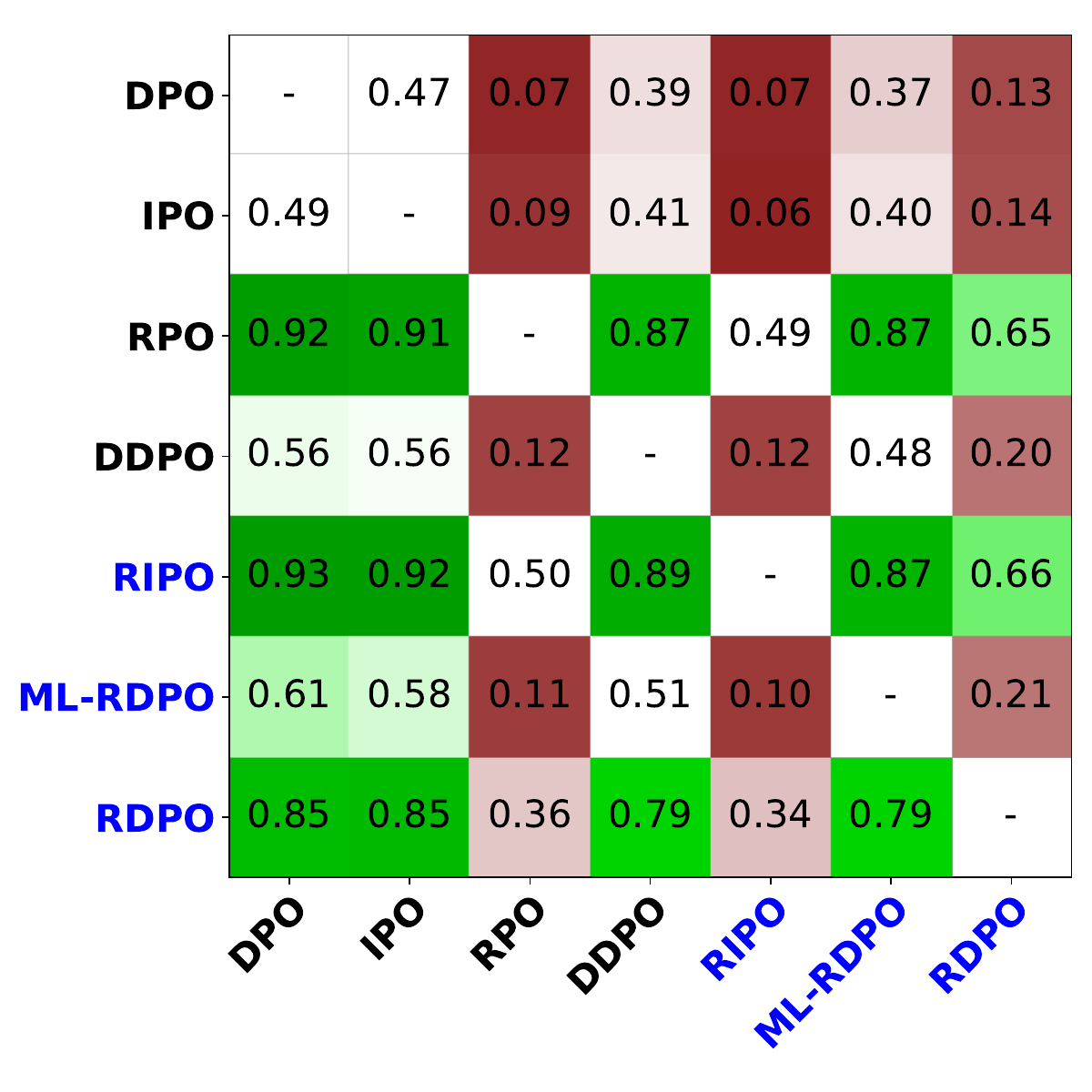}
     } &
\subfloat[SmolLM2-360M]{%
\includegraphics[width=0.3\linewidth]{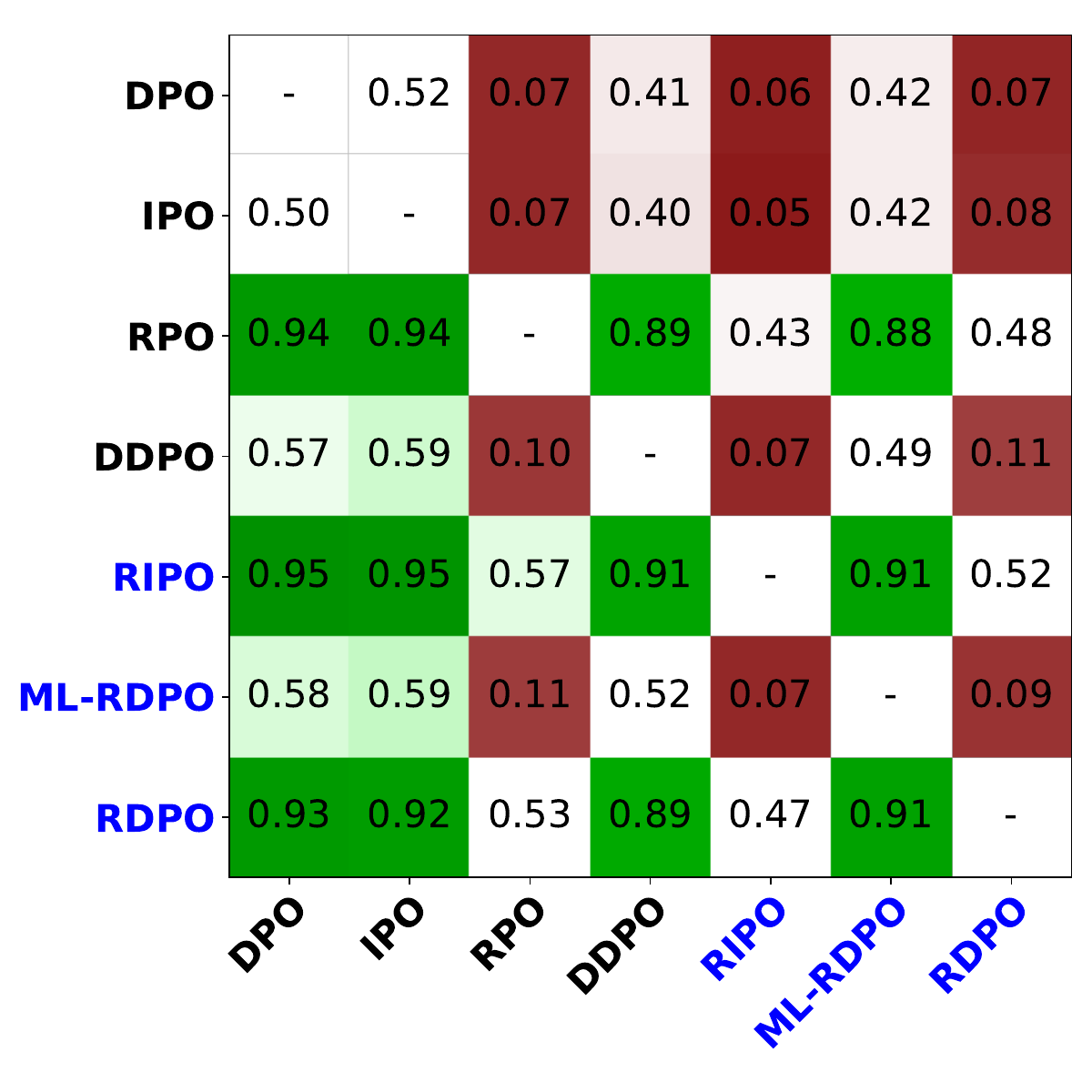}
     } &
\subfloat[SmolLM2-1.7B]{%
\includegraphics[width=0.3\linewidth]{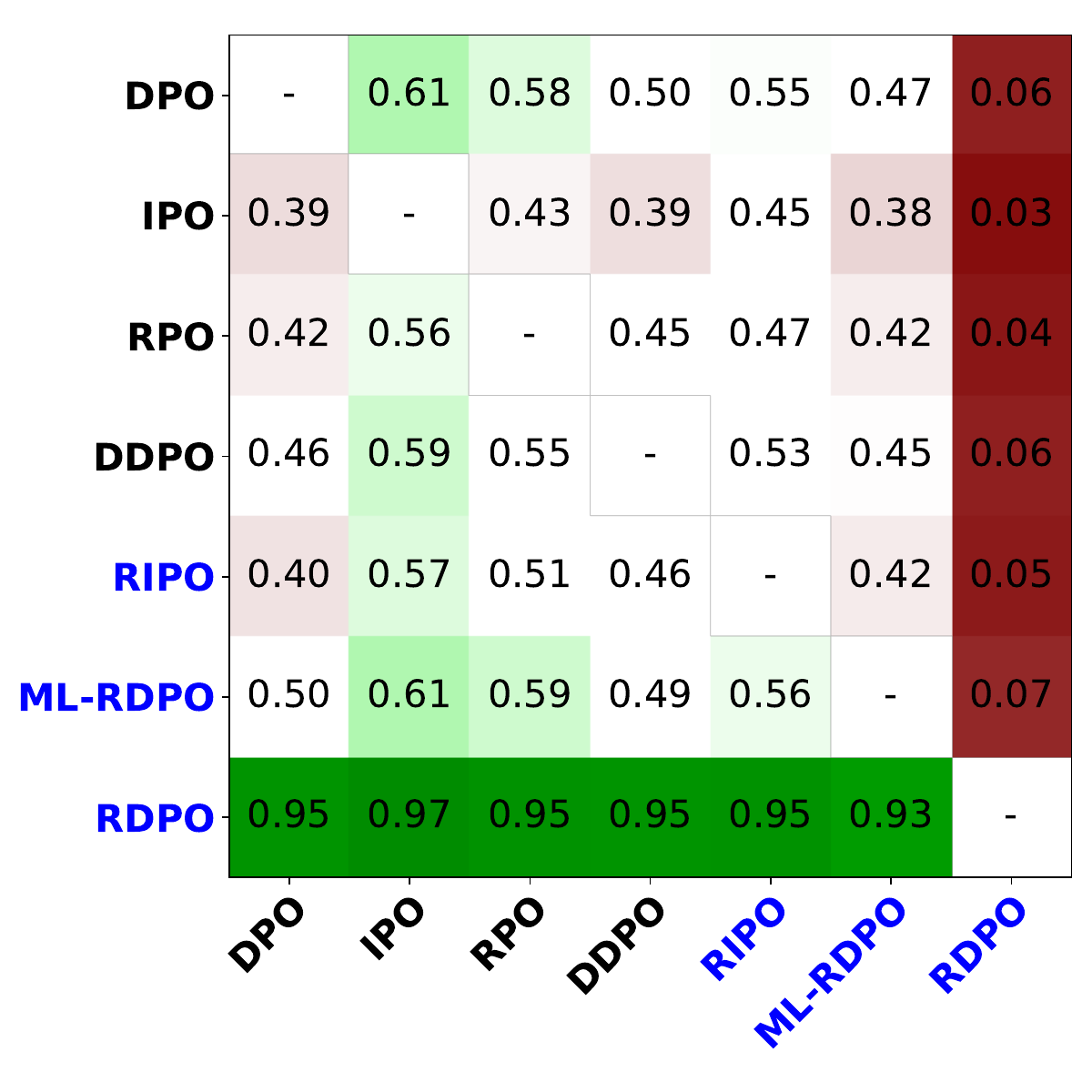}
     } 
\end{tabular}
\caption{ Model by model comparison using Llama, Mistral and Zephyr as base model.}
\label{fig:Smolbattles}
\end{figure*}
We observe that for SmolLM2 models, it is possible to achieve larger win rates, and in this case, the benefit of rating is even more evident than at 7B/8B scale. At 135M, the best methods are RPO \cite{adler2024nemotron} and our RIPO, while RDPO is the third best performing in this setting.  These three algorithms are comparable at 360M scale, while RDPO outperforms all the other methods neatly at 1.7B scale.

ML-RDPO does not perform as well in this setting, but it is always improving over both DPO and Distilled DPO as predicted by \Cref{thm:DPO+distilled_main}.

In these tables, we compare the "best" epochs for each algorithm. Best epoch here means the epoch that achieved the highest win rate against $\piref$, which is reported in \Cref{fig:Smolwin rates}.

Moreover, we highlight that we re-tuned the baselines in this setting because we observed that, for any alignment algorithms,  the hyperparameters working the best at 7/8B scale can be very different from the best ones at smaller scale.
\begin{figure}[h!] 
    \centering 
\begin{tabular}{ccc}
    \subfloat[135M]{
        \includegraphics[width=0.3\textwidth]{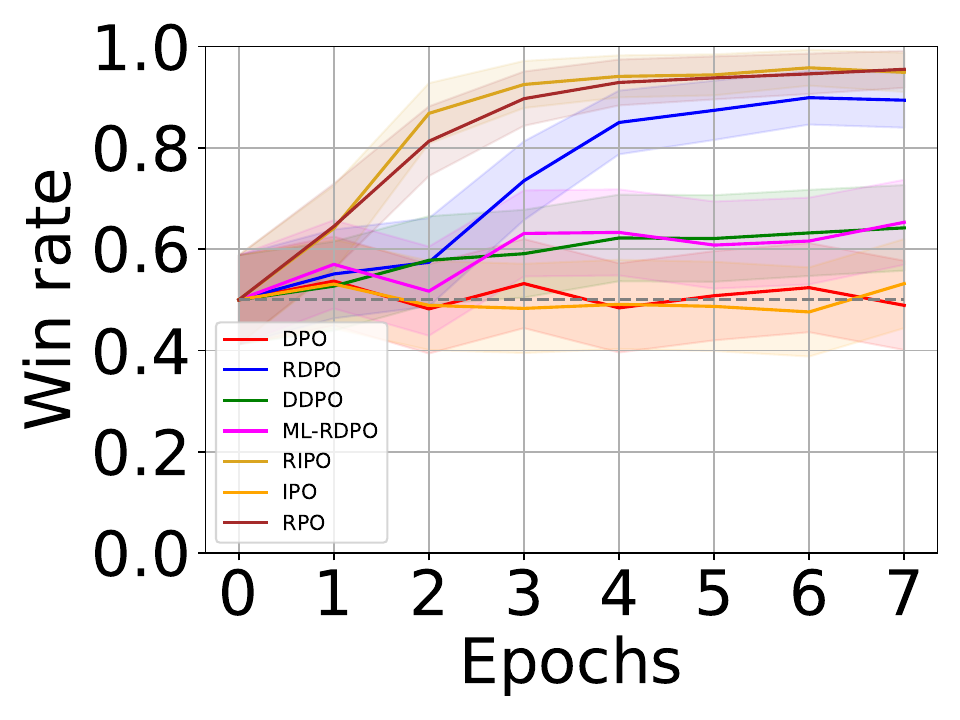}

    }
    &
    \subfloat[360M]{
        \includegraphics[width=0.3\textwidth]{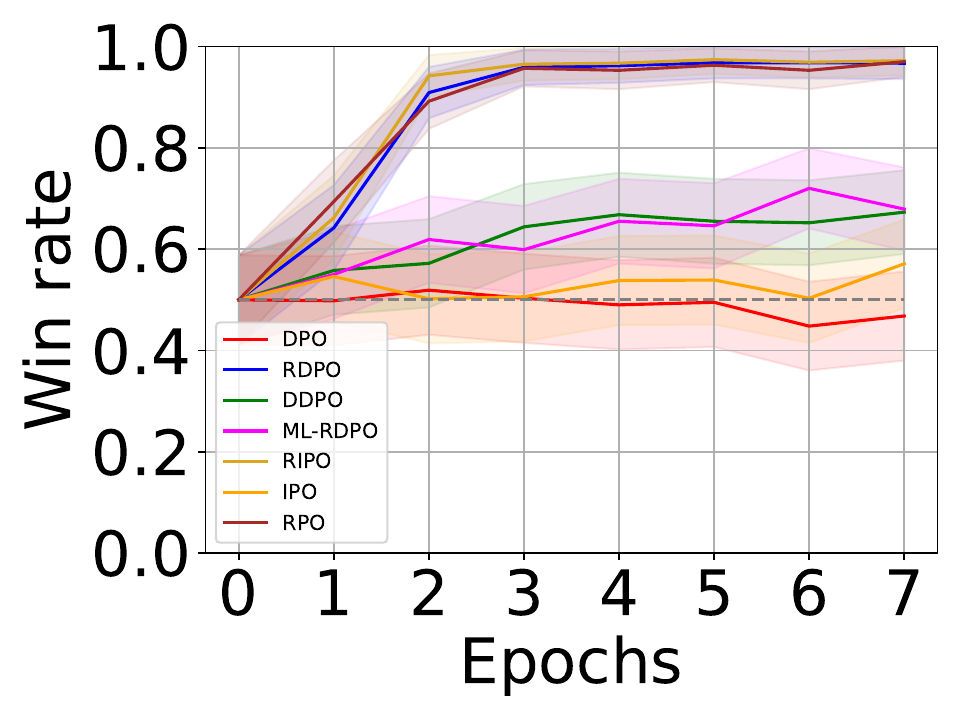}
    }
    &
    \subfloat[1.7B]{
        \includegraphics[width=0.3\textwidth]{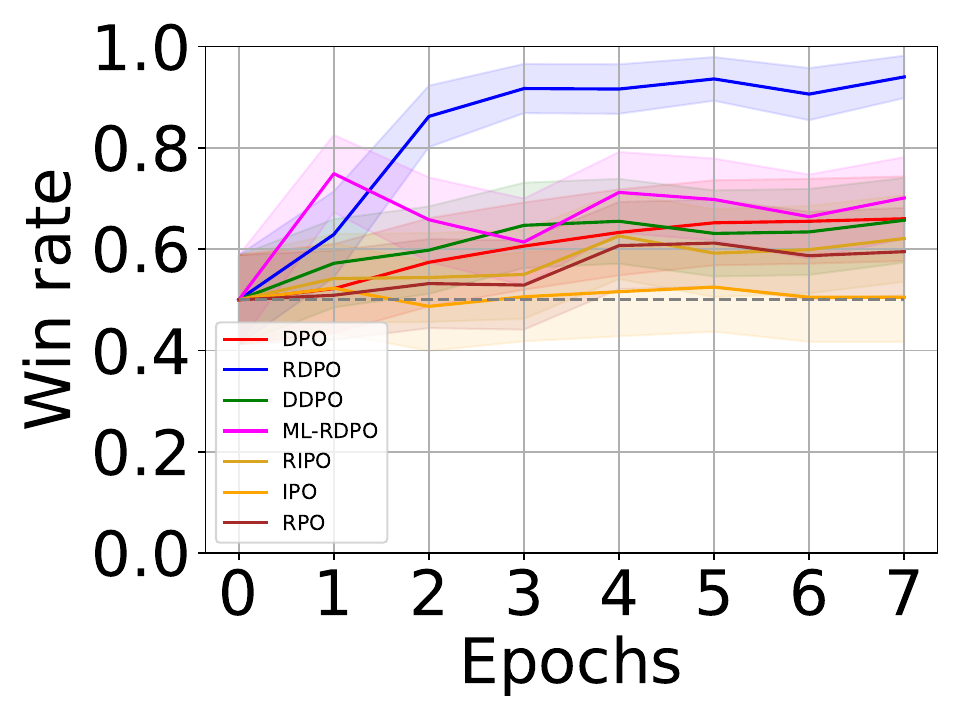}
    }
\end{tabular}
    \caption{Win rate against $\piref$ for the SmolLM2 family experiments. \label{fig:Smolwin rates}}
\end{figure}
\subsection{Ablation for $\beta_1$ and $\mathbb{V}$}
\label{app:zephyr_ablation}
For deciding the hyperparameters to use in the experiments shown in the main text, we run ML-RDPO with $5$ possible choices for $\mathbb{V}$ and RDPO with $5$ possible choices for $\beta_1$ using Zephyr-7B as base model. We select the configuration with the best-performing win rate against $\piref$ for the other experiments. The best hyperparameters found in this way are $\beta_1=1/10$ for RDPO and $\mathbb{V}=1/100$ for ML-RDPO.
We report the win rates achieved by the configuration we tried in \Cref{tab:ablation7B}
 
    \begin{table}[h!]
\caption{Zephyr-7B, Ablation for win rate against $\piref$ on Alpaca Eval \label{tab:ablation7B}}
  \centering
  \rowcolors{2}{gray!25}{white}
\begin{tabular}{lrr}
\toprule
Model & win rate & Standard Error \\
\midrule
ML-RDPO $\mathbb{V} =1$ & 2.43 & 0.539 \\
ML-RDPO $\mathbb{V} =1/10$ & 17.68 & 1.398 \\
ML-RDPO $\mathbb{V} =1/50$ & 19.07 & 1.384 \\
ML-RDPO $\mathbb{V} =1/100$ & \textbf{23.45} & 1.494 \\
ML-RDPO $\mathbb{V} =1/200$ & 19.59 & 1.399 \\
RDPO $\beta_1 = 1$ & 11.68 & 1.133 \\
RDPO $\beta_1 = 1/10$ & \textbf{23.38} & 1.491 \\
RDPO $\beta_1 = 1/40$ & 17.83 & 1.352 \\
RDPO $\beta_1 = 1/100$ & 16.04 & 1.295 \\
RDPO $\beta_1 = 1/125$ & 15.17 & 1.266 \\
Zephyr-7B-SFT-full & 15.11 & 1.262 \\
\bottomrule
\end{tabular}
\end{table}
For both algorithms, we used otherwise standard DPO choices, that is $\beta=0.1$ and learning rate $1e-6$.

Surprisingly, the best hyperparameters for the SmolLM2 family are different. In particular, in \Cref{fig:win rate_base_plot1.7B_ablation_ratings_dpo} we show that for the setting $\piref = \texttt{SmolLM2-1.7B}$ the best value for $\beta_1$ is $0.005$, which is much smaller than the values found at 7B/8B level.
Recall that a smaller $\beta_1$ translates to higher trust for the ratings. Interestingly, it seems that these experiments suggest that the model size plays a role in choosing the right level of trust. This fact is not captured by our current theory, which neglects the optimization aspects of the problem. We believe it is an interesting open question to develop a deeper understanding of this phenomenon, either empirically or theoretically.

\begin{figure}
    \centering
    \includegraphics[width=0.45\linewidth]{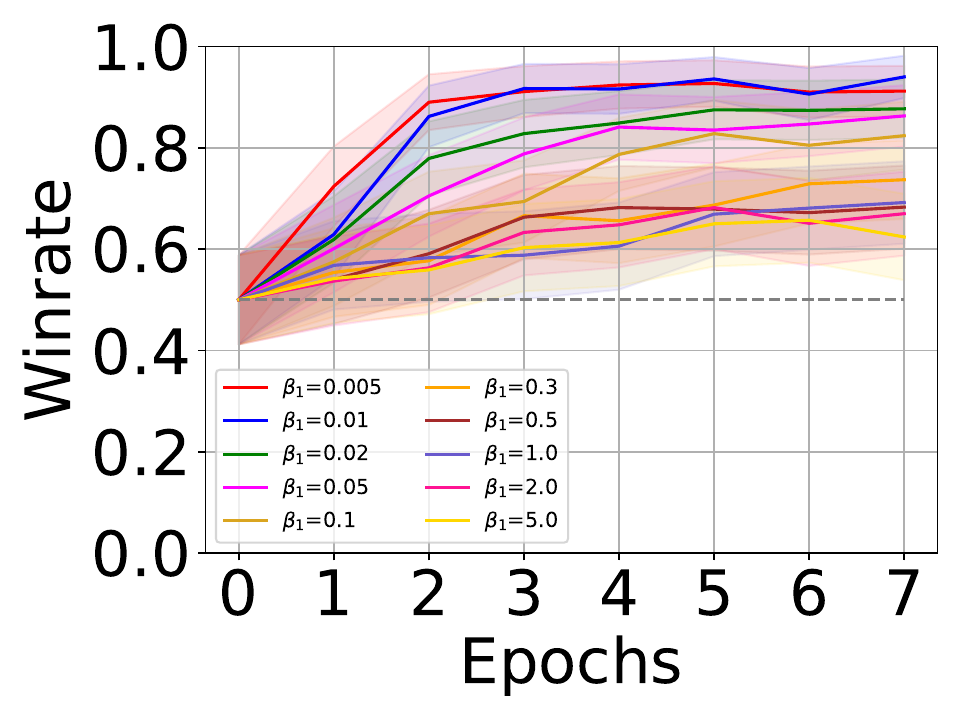}
    \caption{Win rate vs base model SmolLM2-1.7B in Alpaca-Eval for different values of $\beta_1$ used in \Cref{alg:ratingsDPO}}
    \label{fig:win rate_base_plot1.7B_ablation_ratings_dpo}
\end{figure}
\subsection{On the restriction to $\Pi'$ in \Cref{thm:ratings_dpo}}
\label{app:exp_constrain}
In this section, we investigate a practical implementation of RDPO which embodies some penalty to ensure that the output policy is in the restricted class $\Pi'$ which we recall being the subset of $\Pi$ such that $\pi \in \Pi'$ implies that $\abs{\beta \Delta_{\log \pi/\piref} - \frac{\beta}{\beta_1}\Delta_{\hat{r}}} \leq \Delta_{\max} $ where $ \Delta_{\max} := R_{\max}-R_{\min}$. To this end, we consider the following minimization problem
\begin{align*}
    \argmin_{\pi\in\Pi} &\sum^N_{i=1} - \log \sigma \br{\beta \Delta^i_{\log \pi/\piref} - \frac{\beta}{\beta_1}\Delta^i_{\hat{r}}} + \\
    &+ \lambda_1 \br{\Delta^i_{\log \pi/\piref} - \frac{\Delta^i_{\hat{r}}}{\beta_1} - \Delta_{\max}}\mathds{1}\bc{\Delta^i_{\log \pi/\piref} - \frac{\Delta^i_{\hat{r}}}{\beta_1} \geq \Delta_{\max}} \\
    &+ \lambda_2 \br{ - \Delta^i_{\log \pi/\piref} + \frac{\Delta^i_{\hat{r}}}{\beta_1} - \Delta_{\max}}\mathds{1}\bc{\Delta^i_{\log \pi/\piref} - \frac{\Delta^i_{\hat{r}}}{\beta_1} \leq - \Delta_{\max}}
\end{align*}
where $\lambda_1,\lambda_2\in\mathbb{R}^+$ are penalties hyperparameters. The larger $\lambda_1,\lambda_2$ are, the smallest the constraint violation of $\piout$ is.

As shown in \Cref{fig:RDPOpenalty}, we observe an improvement for Llama and Mistral as base model while the performance without penalty is better for $\piref$ chosen to be Zephyr-7B. However, in all cases the difference between RDPO with or without penalties is not remarkable, therefore, while the restriction to $\Pi'$ is needed for the theoretical result in \Cref{thm:ratings_dpo}, we consider it as optional in practice. In particular, it might be avoided when handling a simpler objective is desirable. 
\begin{figure}[t!] 
    \centering 
\begin{tabular}{ccc}
    \subfloat[Llama-3.1-8B]{
        \includegraphics[width=0.3\textwidth]{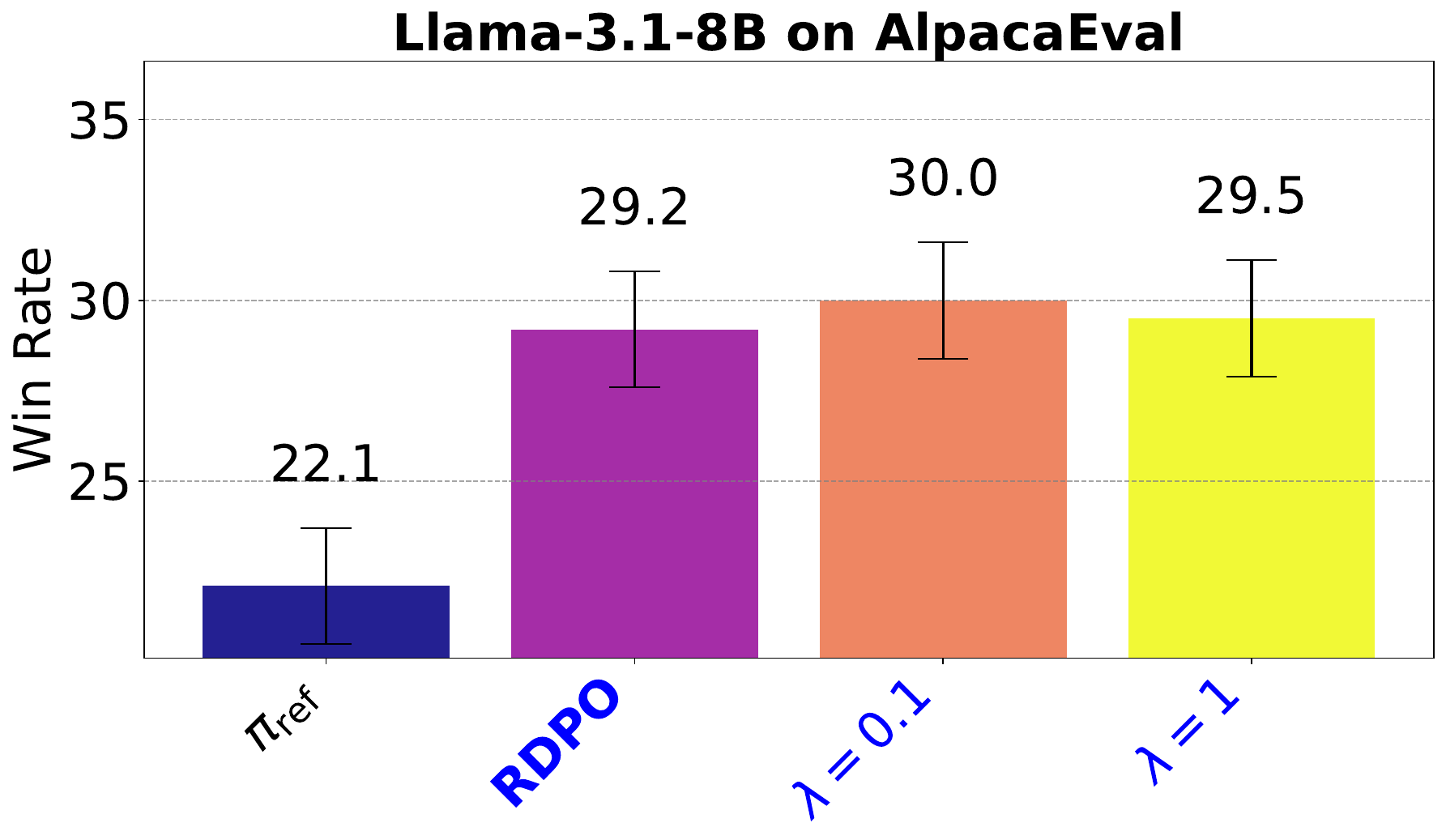}

    }
    &
    \subfloat[Mistral-7B]{
        \includegraphics[width=0.3\textwidth]{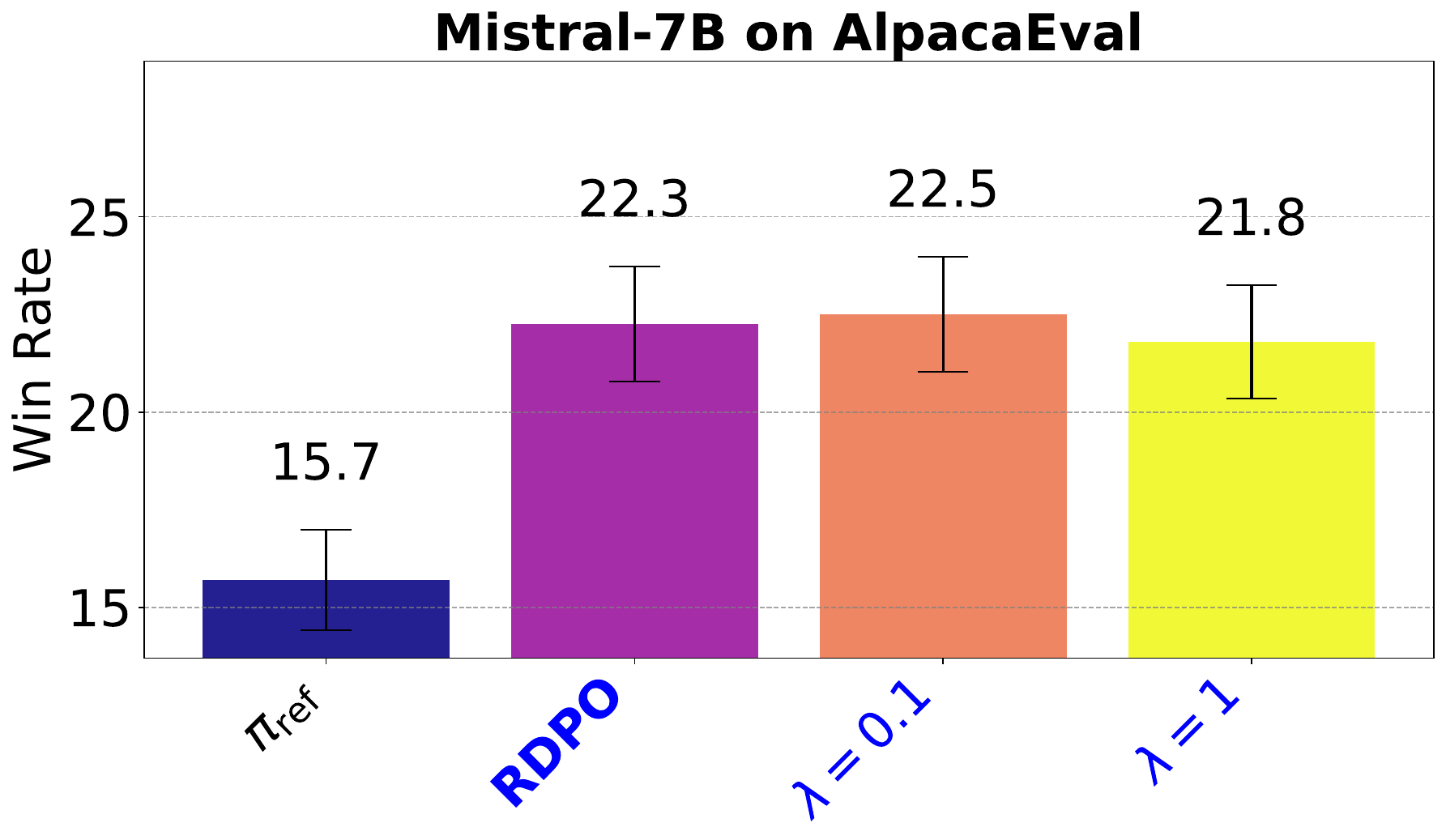}
    }
    &
    \subfloat[Zephyr-7B]{
        \includegraphics[width=0.3\textwidth]{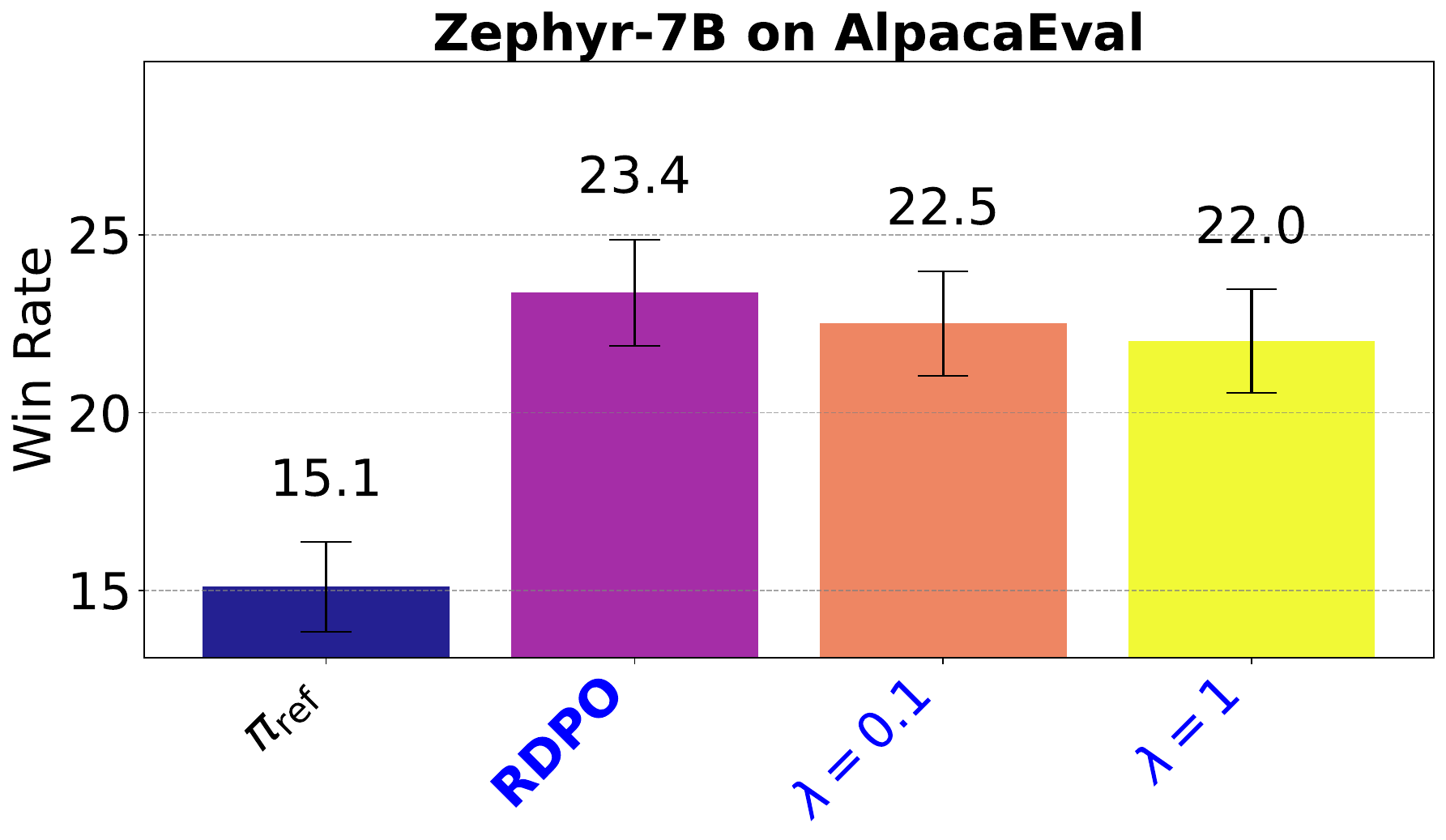}
    }
\end{tabular}
    \caption{Comparison between RDPO and RDPO with penalty parameters $\lambda_1,\lambda_2$ equals to the value of $\lambda$ reported in the plots. \label{fig:RDPOpenalty}}
\end{figure}
\end{document}